\documentclass[twoside,11pt]{article}

\usepackage{jmlr2e}
\usepackage{times,subfig}
\usepackage{amsmath,amssymb,hassan,algorithm,algorithmic,graphicx,subfig}%,datetime}
\usepackage{float}

\newcommand{\accpt}{\mathsf{Acc}}
\newcommand{\baccpt}{\bar{\mathsf{Acc}}}

\newcommand{\rta}{\rightarrow}

\newcommand{\Rd}{\mathcal{R}}

\newcommand{\mdp}{\mathcal{M}}

\newcommand{\psmh}{\bar{P}_{\mathrm{MH}}}
\newcommand{\msmh}{\bar{M}_{\mathrm{MH}}}
\newcommand{\bbfPi}{\mathbf{\bar{\Pi}}}

\newcommand{\qed}{}

\newcommand{\lta}{\leftarrow}

\newcommand{\eqdef}{\triangleq}

\newcommand{\bfA}{\mathbf{A}}

\newcommand{\expc}{\mathrm{I\! E}}

\newcommand{\lonorm}[1]{|\!| #1 |\!|_1}
\newcommand{\ltvnorm}[1]{|\!| #1 |\!|_{\mathrm{TV}}}
\newcommand{\argmin}{\mathop{\arg\min}}

\newcommand{\barphi}{\bar{\phi}}

\newcommand{\mdps}{\mathsf{M}}

\newcommand{\clstngs}{\mathcal{C}}

\newcommand{\ovst}{\overset}

\newcommand{\tldW}{\tilde{W}}

\newcommand{\bfy}{\mathbf{y}}

\ShortHeadings{Clustering MDPs For Continual Transfer}{Mahmud, Hawasly, Rosman and Ramamoorthy}
\firstpageno{1}

\begin{document}

%\title{Clustering Markov Decision Processes for Continual Transfer%\thanks{Grants or other notes


%\author{M. M. Hassan Mahmud \and Majd Hawasly \and Subramanian Ramamoorthy \and Benjamin Rosman}
%\author{RAD}

%\institute{M.M.H. Mahmud \at
%              University of Edinburgh \\
%              School of Informatics\\
%              10 Crichton Street\\
%              Edinburgh, EH8 9AB\\
%              United Kingdom.
%              Tel.: +123-45-678910\\
%              Fax: +123-45-678910\\
%              \email{hmahmud42@example.com}        }

\title{Clustering Markov Decision Processes For Continual Transfer}

%\author{\name M. M. Hassan Mahmud \email hmahmud42@gmail.com \\
%\name Majd Hawasly \email majd.hawasly@gmail.com \\ 
%\name Benjamin Rosman \email benjros@gmail.com \\
%\name Subramanian Ramamoorthy \email sramamoo@staffmail.ed.ac.uk \\
%\addr School of Informatics\\
% University of Edinburgh\\
%Edinburgh, EH8 9AB, UK.
%}

\author{\name M. M. Hassan Mahmud$^\dag$$^\ddag$ \email hmahmud42@gmail.com \\
\name Majd Hawasly$^\dag$ \email m.hawasly@ed.ac.uk \\ 
\name Benjamin Rosman$^\dag$$^\star$$^+$ \email brosman@csir.co.za \\
\name Subramanian Ramamoorthy$^\dag$ \email s.ramamoorthy@ed.ac.uk \\
\addr $^\dag$School of Informatics\\
 University of Edinburgh\\
Edinburgh, EH8 9AB, UK.\\
\addr $\ddag$Now at: Blue Yonder UK Ltd.\\
London, UB11 1FW, UK.\\
\addr $^\star$ Now at: Mobile Intelligent Autonomous Systems Group\\
Council for Scientific and Industrial Research\\
Pretoria, South Africa.\\
\addr $^+$ Now at: School of Computer Science and Applied Mathematics\\
University of the Witwatersrand, South Africa
}

\editor{TBD}

\maketitle

\begin{abstract}
We present algorithms to effectively represent a set of Markov decision processes (MDPs), whose optimal policies have already been learned, by a smaller {\em source} subset for lifelong, policy-reuse-based transfer learning in reinforcement learning. This is necessary when the number of {\em previous} tasks is large and the cost of measuring similarity counteracts the benefit of transfer. The source subset forms an `$\eps$-net' over the original set of MDPs, in the sense that for each previous MDP $\mdp_p$, there is a source $\mdp^s$ whose optimal policy has  $<\eps$ regret in $\mdp_p$. Our contributions are as follows. We present EXP-3-Transfer, a principled policy-reuse algorithm that {\em optimally} reuses a given source policy set when learning for a new MDP. We present a framework to cluster the previous MDPs to extract a source subset. The framework consists of (i) a distance $d_V$ over MDPs to measure policy-based similarity between MDPs; (ii) a cost function $g(\cdot)$ that uses $d_V$ to measure how good a particular clustering is for generating useful source tasks for EXP-3-Transfer and (iii) a provably convergent  algorithm, MHAV, for finding the optimal clustering. We validate our algorithms through experiments in a surveillance domain.
\end{abstract}


\section{Introduction}\label{sec_intro}

Reinforcement learning (RL) in Markov decision processes (MDPs) is a well known framework in machine learning for modelling artificial agents \citep{puterman:1994, sutton_barto:1998}, where the agent's task is one of sequential decision making. While problems such as policy learning in these MDPs are well posed in terms of an objective such as maximising the expected reward, they are often based on a specific instance of an underlying Markov Decision Process model which may or may not be explicitly known to the learning agent. 

Realistic agents must cope with environments with variability, wherein a process generates many instances of MDPs within which the agent must now solve the optimisation problem. Finding optimal policies with respect to an unrestricted family of possible MDPs is not only intractable, but often leads to highly conservative and not practicable solutions. However, many realistic environments actually occupy a middle ground. 

In many reinforcement learning problems, that are of great interest in practice, the main source of complexity is partial variability in task description rather than a complete change in the domain. Stated in the language of MDPs, these domains involve a family of possible reward and transition functions over a shared state-action space. This is typically the result of non-stationary behavior of some component that describes the problem. One class of problems is where the human-machine interaction setting remains the same, while the participants change. A recent, highly successful example of such a problem is that of website morphing \citep{hauser_urban_liberali_braun:2009}. In this problem, the goal is to present to the user of a website a view/interface adapted to the needs and skill level (savvy, newbie etc.)  of that particular user. The interface to present to a user is determined adaptively based on the sequence of her choice of links. So, even though the setting (i.e. the website) remains the same , the dynamics defining the problem (the links the user might choose) changes with each user -- i.e. each user corresponds to a new task and corresponding change in problem dynamics. Given the change in dynamics, the algorithm has to determine the best policy/sequence of actions (sequence of interfaces to present). Note that since the algorithm knows that a new user has arrived at the homepage, it always knows if the dynamics may have changed -- but not how it has changed. In the MDP formalism, each task corresponds to a particular reward and transition function,  while the state and action space remain the same across tasks. Additionally, the reinforcement learning agent is always told when the dynamics/task have changed.

In this paper, we present an approach to dealing with such problems that is based on the notion of transfer learning. We view the new task instance as being similar to previously experienced ones, although we have no explicit measurement of how similar the new task instance is to any previously seen one. Our objective then is to devise an efficient transfer learning method for reinforcement learning in MDPs \citep{taylor_stone:2009} , that enables learning agents to learn efficiently enough to be useful in domains as the one mentioned above.

As a motivating example that we develop further in our experiments, consider a surveillance agent that is monitoring a large geographical region (this is a variant of the kinds of problems that are considered, for instance, in \citep{an_kempe_kiekintveld_shieh_singh_tambe_vorobeychik:2012}). The agent faces a sequence of monitoring problems where each problem corresponds to the pattern in which infiltrators appear in different locations. If two tasks have similar infiltration patterns, then the same surveillance policy may be good for both of them. During each task, the goal of the agent is to learn the regions where infiltrators appear and choose the appropriate surveillance policy. We do not expect the patterns to be completely different every time, but at the same time we cannot completely rule out a new pattern emerging. In the former case, we should recognize this repetition and take advantage of this fact by reusing old surveillance policies. In the latter case, we should also determine that the new scenario is novel and learn an appropriate policy for that scenario. Furthermore, if the number of previous patterns becomes too large, we also need to compactly re-represent them so that the procedure for determining the correct way to act is more sample efficient.

Our setting of MDPs with changing transition and reward functions are typically referred to as {\em non-stationary} MDPs and has been considered before \citep{nilim_el_ghaoui:2005}, \citep {yu_mannor_shimkin:2009} \citep{yu_mannor:2009}.  In these problems, the goal of the learner is to take actions in such a way so as to perform as well as the stationary policy which is optimal in hindsight with respect to all the transition and reward function pair the MDP attained. Note that for a specific pair, this policy may be substantially worse than the optimal policy for that pair. Our approach in this paper represents an orthogonal setting, wherein we consider the case where the learner tries to {\em track} the optimal policy -- i.e., always perform as well as the best policy for the current set of transition and reward function -- with the proviso that the learner is told when the dynamics have changed. We also require the learner to, if possible, perform better than `pure' reinforcement learning algorithms applied to specific transition and reward functions by transferring information from the transition and reward functions for which the optimal policy was learned previously. 

More formally, in this paper we consider TLRL for the case of a `lifelong' learning agent that learns a (possibly never-ending) sequence of MDPs which are defined on the same discrete state and action space but differ in terms of the transition and reward distributions. We assume the distribution generating the sequence is unknown and unlearnable (for instance, in the motivating surveillance problem described above, the infiltration pattern may depend on the internal variables of the infiltrators that are not known to us). In this setting, the goal of the agent is to, if possible, {\em reuse} the optimal policies in the previous MDPs in order to learn the new MDP more {\em efficiently}. In this continual setting, we assume that the agent operates in the new MDP for a fixed number of episodes, and hence we measure efficiency by the total reward accumulated while learning the new task during these fixed number of episodes. Reusing a policy means that we should try the optimal policies of the previous MDPs in the new MDP and if one results in efficient behavior we should keep using it. However, as we described above, a problem in this setting is that, when the number of previous tasks become too large, transfer becomes ineffective as the agent spends too much time testing the old policies. In this instance, one possible solution to this problem is to find a subset of the $N$ {\em previous} policies, which we call {\em source} policies, that are, in a well-defined and useful sense, representative of all the $N$ previous policies (see Section \ref{sec_relWork} for alternative encodings). In other words, the source policies should form the analogue of an $\eps$-net in a metric space \citep{kolmogorov_fomin:1970} over the space of previous MDPs with respect to an appropriate distance over MDPs.  In this paper we present a clustering based approach to finding this smaller subset of source policies. 

Our overall approach is illustrated in figure \ref{fig_overall}. The main idea is to cluster the $N$ previous MDPs into $c$ clusters, where the number $c$ and the clusters themselves are to be determined via discrete optimization,  and then choose the representative element of each cluster to obtain the source MDPs. The optimal policies of the source MDPs then become the source policies. In our approach to choosing the clustering  and the corresponding source policies, we attempt to ensure a-priori that the chosen source policies are a good representative of the previous tasks for the purposes of transferring to the unknown target task. 

\begin{figure}
\begin{center}
\includegraphics[scale=0.5]{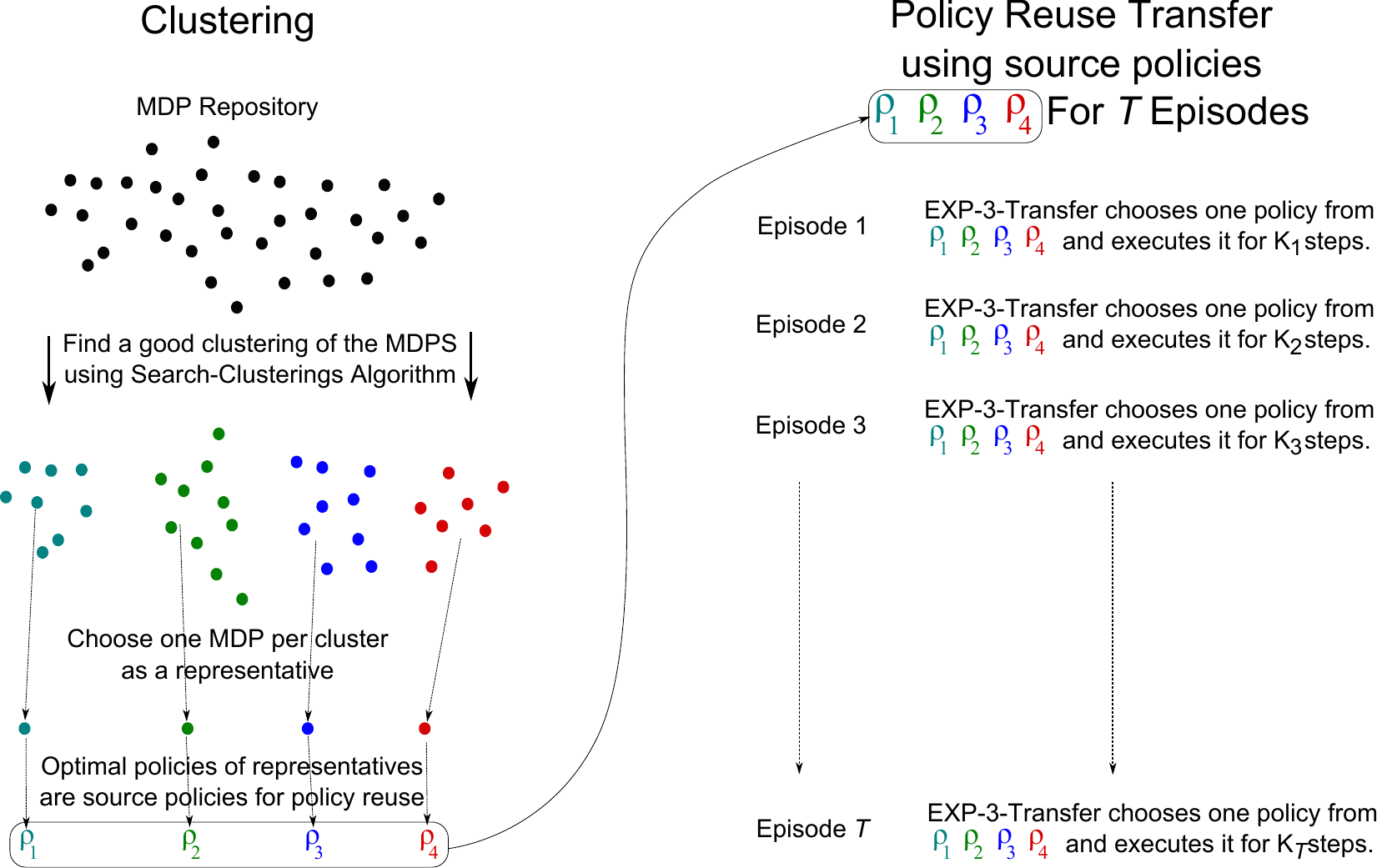}
\end{center}
\caption{This figure illustrates our overall algorithmic approach. Our method consists of two parts. In the first part (left pane) we use the Search-Clusterings algorithm to find a good clustering of the MDPs in our repository, which are then used to generate a source set of policies. In the second part (right pane) these policies are used by EXP-3-Transfer to perform transfer using policy reuse in a new MDP. Reducing the number of policies input to EXP-3-Transfer helps it choose the best policy quickly. On the other hand, reducing policies risks leaving out a policy with good performance in the new MDP. The bulk of the paper is dedicated to deriving a principled way to trade off these two contradictory goals .}\label{fig_overall}
\end{figure}

In particular, we define a transfer learning algorithm, EXP-3-Transfer, with performance bound $g(c)$ that depends on the number $c$ of source policies.  Hence this explicitly measures how good the size of the clustering $c$ is. We are now left with the task of choosing the clustering and the corresponding source policies. To that end, we define a distance function $d_V$ between two MDPs that takes the distance between two MDPs to be how well the optimal policy of one MDP performs in the other (this in turn is measured as the difference between the value of the optimal policy of the former and the value of the optimal policy of the latter at the start states of the latter). Hence, given that our goal is to reuse optimal policies of one MDP in another, we choose our clustering so that within each cluster the pairwise $d_V$ distances between the elements of the cluster are low. Similarly, we choose the source policy for each cluster to be the optimal policy of the MDP in that cluster which has low $d_V$ distance with respect to all the other elements. Hence, the cost of a clustering with $c$ clusters is, roughly speaking, $g(c) + \eps$ where $\eps$ is a measure of the inter-element $d_V$ distances in the clusters.

Given the cost function, we show that it is NP-hard to find the optimal clustering. So, we introduce a Markov chain Monte Carlo based discrete optimization algorithm to find it. The algorithm is an extension of the Metropolis-Hastings algorithm, which we call Metropolis-Hastings with Auxiliary Variables (MHAV in short), and can also be thought of as an extension to simulated annealing \citep{kirkpatrick_gelatt_vecchi:1983} with stochastic temperature changes. Simulated annealing is a well known algorithm for discrete optimization, but requires carefully setting up an infinite sequence of parameters known as the temperature schedule. Determining this schedule in practice to ensure convergence is acknowledged to be very difficult, and practically an art form. In our version of the algorithm, we search over both the temperature and the optimal point simultaneously, thereby handling the problem of setting the schedule automatically.

To summarize, our overall continual transfer algorithm is as follows. The agent continually learns optimal policies for MDPs presented in sequence. When learning the optimal policy for a particular MDP, the learning agent uses the optimal policies of the  previously solved MDPs in a policy reuse transfer learning algorithm. To make transfer more effective, at fixed intervals, the agent clusters the previous MDPs to derive a small subset as the set of source MDPs, whose optimal policies are then used as input to the policy reuse algorithm. The clustering is chosen so as to optimize the regret of the transfer algorithm, and is found by using a convergent discrete optimization algorithm. 

We conclude this brief introduction to our method by noting that our transfer algorithm EXP-3-Transfer is in fact an extension of the well known EXP-3 algorithm \citep{auer_cesa-bianchi_freund_schapire:2002} for non-stochastic multi-armed bandits, and our performance bound $g(c)$ is in fact a regret bound of the type well known in bandit algorithm literature. Our strategy is to cast the policy reuse transfer learning problem as a bandit problem, with `pure reinforcement learning algorithm' as one arm, and the $c$ source policies as the remaining arms. The regret bound for EXP-3 ensures that we minimize negative transfer by never performing much worse than pure reinforcement learning.  We will now discuss related work.

%[Point out that the clustering approach can be used as a recipe with other methods of transferring - may be in conclusion ?]

%Due to lack of space and in the interest of presenting a complete method, most of the proofs are in sketch form. These are all available in the longer version.

\subsection{Related Work}\label{sec_relWork}

As evidenced by the survey paper \citep{taylor_stone:2009}, a significant amount of work has been done on transfer learning in reinforcement learning. As mentioned previously, lifelong learning in reinforcement learning was first explicitly considered in  \citep{mitchell_thrun:1993,thrun_mitchell:1995,thrun:1996a} in the context of learning in robots. In these works, the main aim was to learn the dynamics of robot motion in one circumstance using a function approximator (such as neural networks) and then use these learnt dynamics as an initial bias in a new situation using an explanation based learning framework. 

In terms of recent work on TLRL, two different strands are related to our work. The first is work on policy reuse and the second on task encoding. The first and, to the best of our knowledge, the only authors to have considered policy reuse are \citep{fernandez_veloso:2006, fernandez_garcia_veloso:2010}. The algorithms presented there, at the beginning of every episode, choose between different source policies by using a softmax criteria on accumulated reward and then use the chosen policy as an initial exploration policy before switching to Q-learning exclusively. In contrast, we extend the EXP-3 algorithm for multi-armed bandits to choose between source policies and Q-learning, and as a result inherit its theoretical guarantees. Additionally, they do not consider the problem of source task selection, whereas in our work this is a major focus. 

Two closely related works are \citep{talvitie_singh:2007} and \citep{azer_lazaric_brunskill:2013} where the authors consider the best way to choose between a set of stationary policies. The main distinction between policy reuse and these is the requirement in the former to not perform much worse than a base RL algorithm (which is also our requirement).

We now look at previous work that uses a smaller set of source tasks to represent the complete set of previous tasks. The problem of source task selection through clustering seems to have been considered only by \cite{carroll_seppi:2005}. They introduce several measures for task similarity and then consider clustering tasks according to those measures. The distance functions introduced were heuristic, and the clustering algorithm used was a simple greedy approach. The evaluation of their method was on several toy problems. In contrast, we derive a cost function for clustering in a principled way to optimize the regret of our EXP-3-Transfer policy reuse algorithm. Additionally, instead of constructing the cluster in a greedy fashion, we search through clustering space using a convergent discrete optimization algorithm.

Recent work that also chooses selectively from previous tasks is \citep{lazaric_restlli_bonarni:2008}, \citep{lazaric_restilli:2011}. The setting for this paper is a collection of tasks defined on the same state-action space with the tasks and the state-action-state triples for the different tasks generated i.i.d. (similar to the multi-task transfer in classification setting considered in \citep{baxter:2000}) rather than sequentially as is typical in RL. Under this setting the authors are able to bound the error when samples from one task are used to learn the new task. This is quite a different setting from us as it is `batch' RL rather than the typical online and sequential RL and measures similarity in terms of the actual transition and reward functions rather than policies or values. Additionally, the analysis and algorithms are derived under the assumption of a fixed set of prior tasks rather than the continual lifelong learning setting we consider.

%\hspace{10pt}
%\noindent {\bf *********}
%
%REviewer 1
%
%\begin{quote}
%{\em ``It might help to also contrast this work with AtEase (Talvite and Singh, 2007), as this work also considers transferring from multiple policies (experts) and choosing the best.''}
%\end{quote}
%
%
%
%
%\noindent {\bf *********}
%\hspace{10pt}

% Lifelong learning stuff.

%Source task selection was also considered in \cite{lazaric:2008}. In this paper, the author assumed that the MDPs are generated according to an i.i.d. distribution and then the goal is to reuse the policies in there. In our work we make no assumption about the way the tasks are generated. 

Source task selection is not the only possible way to represent previous tasks, and the overall goal of finding abstractions for exploiting commonality has received considerable attention in the transfer learning community. Most of the work done in deriving abstractions for the purposes of transfer has been for MDP homomorphisms \citep{ravindran_barto:2003,ferns_panangaden_precup:2004,ravindran:2013,konidaris_barto:2007,sorg_singh:2009,castro_precup:2010}. In these works, similarity between MDPs is defined in terms of bisimulation between states of different MDPs. Bisimulation is a concept borrowed from process algebra. In the context of transfer learning in MDPs, at its most general formulation, a bisimulation is an isomorphism $f,g$ between the state and action spaces that is preserved under the transition distribution -- that is for every state-action-state triple $s,a,s'$, $T_1(s'|s,a) = T_2(f(s')|f(s),g(a))$ where $T_i$ are the transition distribution of the two MDPs. Unfortunately, in this pure form, bisimulation is absolute (two MDPs are either bisimilar or not)  and does not take into account the reward function. And so, in the papers mentioned above, this basic notion was extended in various ways to address both these issues. However, one of the main issues with bisimulation is computational cost, and this remains so in the extensions as well. Another issue with these approaches is that, as observed by \cite{castro_precup:2010}, bisimulation is a worst case metric (two states may have identical optimal actions but still be completely different according to the metric) and as a result requires heuristic modifications. 

Technically, the main difference between our approach and bisimulation based methods is that the similarity between different MDPs are ultimately determined by distance between value functions. In our case, however, we are interested in distance in terms of {\em policy}. As a result, even though two tasks might be quite different in terms of the value function they might be identical in terms of the optimal policy, and our approach will capture this (as noted earlier, failing to do this was one of the issues with bisimulation based approaches). 

%[should i describe the bisimulation papers individually ?]

Another interesting line of work that uses a different approach to abstracting MDPs is the proto-value function based approach of \citep{ferrante_lazaric_restelli:2008}. Proto-value functions were introduced in \citep{mahadevan:2005} as an efficient way to represent the value function for large state spaces as a linear combination of functions, which are called proto-value functions. The main innovation in this approach is that, in representing the value function as a real function over state space, the state-space is treated as a manifold where the distance between points/states is determined by the reachability graph of the MDP. This idea of a spectral-decomposition of the value function naturally lends itself to transfer learning, as, given a new task, we can imagine using the proto-value functions learned in a previous task to initialize the new value function in the new task. It has been noted that proto-value function based transfer has issues in terms of scalability and tractability. The main difference between this and our work is that, as with the homomorphism based approach, our similarity notion is based on policy similarity, while theirs is based on similarity between value functions.  Identifying policy similarity is more desirable because tasks similar in terms of value function will be similar in terms of policy, but not necessarily the other way round. 

A somewhat different approach to finding MDP abstractions was adopted in  \citep{ammar_tuyls_taylor_driessen_weiss:2012}, where MDPs were related by mapping state action state triples to a lower dimensional space, and consider triples to be equivalent if their representations were found to be close. The authors were able to transfer between tasks such as inverted pendulum, cart pole and mountain car, showing that in these cases their approach was able to discover the fact that the differential equations describing these domains have similar/identical forms. 

%
%\hspace{10pt}
%\noindent {\bf *********}
%
%REviewer 1
%
%
%\begin{quote}
%{\em ``The authors may also wish to cite ``Reinforcement Learning Transfer via Sparse Coding'' by Bou Ammar, 2012, as another example of finding an abstraction over which to transfer.'' }
%\end{quote}

Finally, our work can also be related to the notion of equivalence between probabilistic models, which has been influential in early work on Bayesian network learning. For instance, Chickering and collaborators wrote a series of papers in which the notion of event equivalence and score equivalence is used to render the problem of searching over network structures somewhat tractable. In \citep{heckerman_chickering:1995}, it is shown that the notion of event equivalence, i.e., that two Bayesian networks should be treated as similar if they agree on the independence and dependence relationships between random variables, can be used to define local structure edit operations that enable learning of network structures. Subsequently, in \citep{chickering_boutlier:2002}, \citep{chickering:2003} this idea is developed further to show that by considering this notion of equivalence it is possible to achieve optimal structure identification with an essentially greedy procedure. We take inspiration from such work, but also note that our task of comparing two sequential decision making problems differs from that of making predictions with a probabilistic model, calling for new notions of process similarity and corresponding algorithms for transfer.

%A recent related work is \cite{coviello_chan_lanckriet:2012} which cluster together HMMs that encode similar distributions over state space. Here again, this is not directly applicable to our case, where we need to determine if tasks are similar in terms of policies.

%A lot of work has been done on TLRL, particularly when we try to transfer between MDPs defined on the same state and action space (see \cite{taylor_stone:2009} for a more thorough discussion). However, to the best of our knowledge, none of those papers have to date has handled the question of finding a proper subset of large number of tasks. They have not tried to summarize the information in a way that
% can be reused. The policy-reuse algorithm \cite{fernandez_veloso:2006} is the one paper that is most related mainly in terms of how we do manage different things. 

\subsection{Paper Organization}

In the following we proceed as follows. We present the notation and some fundamental notions in Section \ref{sec_prelim}. Then we define our transfer learning algorithm and framework for measuring distance in Section \ref{sec_clustEnc} respectively. Sections \ref{sec_findclust} and \ref{sec_contTrans} presents our clustering algorithm and the full continual transfer algorithm. We then present our experiments in Section \ref{sec_exp} and then end with a conclusion in Section \ref{sec_conc}. 

\section{Preliminaries}\label{sec_prelim}

We use $\eqdef$ for definitions, $Pr$ to denote probability and $\expc$ for expectation. A finite MDP $\mdp$ is defined by the tuple $(\stts,\acts,\Rd,P,R,\gamma)$ where $\stts$ is a finite set of states, $\acts$ is a finite set of actions and $\Rd = [l,u]\subset \bbbr$ is the set of rewards. $P(s'|s,a)$ is a the state transition distribution for $s,s' \in \stts$ and $a \in \acts$ while $R(s,a)$, the reward function, is a random variable taking values in $\Rd$. Finally, $\gamma \in [0,1)$ is the discount rate. 

A (stationary) policy $\pi$ for $\mdp$ is a map $\pi : \stts \rta \acts$. For a policy $\pi$, the Q function $Q^\pi : \stts \times \acts \rta \bbbr$ is given by $Q^\pi(s,a) = \expc[R(s,a)] + \gamma\sum_{s'} P(s'|s,a)Q(s,\pi(s'))$. The value function for $\pi$ is defined as $V^{\pi}(s) =  Q^\pi(s,\pi(s))$. An optimal policy $\pi^*$ satisfies $V^{\pi^*}(s) \geq V^{\pi}(s)$ for all policy $\pi$ and state $s$ -- it can be shown that every finite MDP has an optimal policy.  $V^{\pi^*}$ is written $V^*$, and the corresponding $Q$ function is given by $Q^*(s,a) =  \expc[R(s,a)] + \gamma\sum_{s'} P(s'|s,a) Q^*(s',\pi^*(s'))$. When the agent acts in the MDP, at each step it takes an action $a$ at a state $s$, and moves to the next state $s'$ and the reward $r$. The goal of the agent is to learn $\pi^*$ from these observations and then choose the action $\pi^*(s)$ at each state. If there are multiple optimal policies, we will designate the first policy in a lexicographical order as the canonical policy. We will assume that $R_{\max}$ is a known upper bound on the reward function. Without loss of generality, in the sequel we assume that there is a single initial state $s^\circ$.  We call a policy $\pi$ {\em $\eps$-optimal} if $V^*(s^\circ) - V^\pi(s^\circ) \leq \eps$

%We define the {\em regret} of a policy $\pi$ to be $V^*(s^\circ) - V^\pi(s^\circ)$.
%[fix the presentation of initial state]

{\bf The transfer learning setting.} In the transfer learning setting, we are given previous MDPs $\mdp_i$, $1 \leq i \leq N$ and we transfer from these tasks to learn the $N+1^{st}$ MDP $\mdp_{N+1}$. The learning of $\mdp_{N+1}$ runs for $T$ episodes. We denote the optimal policy of the $i^{th}$ previous MDP by $\pi^*_i$, and the value of a policy $\pi$ in MDP $i$ as $V_i^\pi$. Similarly, we denote the reward and transition functions of the $i^{th}$ MDP by $P_i$ and $R_i$ respectively. We will assume that the rewards of all MDPs fall within the range $[R_{\min},R_{\max}]$ and we define $\Delta R \eqdef R_{\max}-R_{\min}$.

%\begin{equation}\nonumber
%Q^*(s,a) =  \expc[R(s,a)] + \gamma\sum_{s'} P(s'|s,a) \max_a Q^*(s',a)
%\end{equation}

%\begin{equation}\nonumber
%Q^\pi(s,a) = \expc[R(s,a)] + \gamma\sum_{s'} P(s'|s,a)Q^\pi(s',\pi(s'))
%\end{equation}

\section{The Policy Reuse Problem}\label{sec_polReuse}

In this section we first concretely define the problem of policy-reuse for transfer learning and then define our algorithm for solving this problem. We define the goal of policy reuse to be to design an algorithm that runs for $T$ episodes on the target task $\mdp_{N+1}$ and, given a collection of {\em source policies} $\rho_1,\rho_2,\cdots,\rho_c$ and the Q-learning algorithm/policy, performs nearly as well as the best policy (in hindsight) in this collection over the $T$ episodes. This requirement has two important implications. First, since the set of policies contain Q-learning, and as Q-learning converges to the optimal policy, it means the algorithm is required to perform {\em nearly as well as a policy that converges to the optimal policy} -- in other words, the algorithm should {\em avoid negative transfer}. Second, if there is a source policy that is near-optimal, then the algorithm is also required to perform {\em nearly as well as that near-optimal policy} -- in other words, the algorithm should {\em transfer from the new to the old task} if possible.

To derive our algorithm we show that policy reuse may be cast as an instance of the non-stochastic multi-armed bandits problem, and hence the classic EXP-3 algorithm \citep{auer_cesa-bianchi_freund_schapire:2002} may be extended to solve policy reuse. We call our extension EXP-3-Transfer and we derive regret bounds for this algorithm (which is similar to the bounds for EXP-3). In particular we show that EXP-3-Transfer performs nearly as well as the best policy in the collection described above, in expectation with respect to the randomness in the algorithm and the reward and transition function of $\mdp_{N+1}$. Hence, this algorithm satisfies the requirements we laid out for a policy reuse algorithm. We now expand these ideas.

\subsection{Defining The Policy Reuse Problem}

In the policy reuse transfer problem, we have the target task $\mdp_{N+1}$ and a set of $c$ source policies $\rho_1,\rho_2,\cdots,\rho_c$ (in general, $\rho_i$s are arbitrary -- but in this paper, each $\rho_i$ is the optimal policy $\pi^*_j$ of some $\mdp_j$ our repository, and Section \ref{sec_clustEnc} shows how to choose the $\rho_i$s in a principled way). It seems that policy reuse as a method for transfer was first introduced in  \citep{fernandez_veloso:2006,fernandez_garcia_veloso:2010}. The algorithm introduced was called Probabilistic Policy Reuse (PPR), and the goal of the algorithm was to balance using the $c$ source policies and pure reinforcement learning policy ($\eps$-greedy Q-learning) so that the learning algorithm converges faster than running pure RL by itself. The basic idea in PPR is as follows. At the beginning of each episode, PPR chooses a policy from among the source policies and the $\eps$-greedy Q-learning policy using a softmax criterion on the observed returns of the policies in previous episodes. It then initiates a policy-reuse episode, where at each step it probabilistically chooses between $\eps$-greedy Q-learning and the chosen policy, with probability of choosing $\eps$-greedy Q-learning going to $1$ during the episode.  In essence, the $c$ source policies serve as an initial exploration policy, so that if they to take the agent through paths of the optimal policy, it would result in faster learning of the optimal policy. 

%\footnotetext{We use the term `policy reuse' to refer to the generic problem of transfer using policy reuse as defined in this section, and the upper-case `Policy-Reuse' to denote the specific algorithm of \citep{fernandez_veloso:2006}.}

There are several aspects of the above algorithm that are noteworthy. First, even if the $c$ source policies contain the optimal policy, the algorithm would deterministically switch to Q-learning after the initial phase. Another aspect is that, while there is an intuitive connection between the soft-max criterion and the benefit of using a policy, the actual connection is not made rigorous. Both these issues arise from the fact that the goal of policy reuse was not defined concretely in \citep{fernandez_veloso:2006}. So, taking cue from the definition of online learning algorithms \citep{vovk:1990,littlestone_warmuth:1994,cesa-bianchi_lugosi:2006}, we define the policy reuse problem concretely as  designing an algorithm that chooses policies at every episode in such a way that it does not perform much worse than any of the $c$ policies or Q-learning over the $T$ episodes. As discussed above, this requirement implies both that (i) the algorithm avoids negative transfer and (ii) transfers from/reuses a good policy (if one exists). Formally,
\begin{definition}[The policy reuse problem]\label{def_polreuse}
Let transfer learning for the target task $\mdp_{N+1}$ be run for $T$ episodes, and let $\bar{x}_i(t) \eqdef \sum_{n=1}^{K_t} \gamma^n r_n$ be the discounted return accumulated by running $\rho_i$ (with $\rho_{c+1}$ being the non-stationary Q-learning policy) at episode $t$, with $r_n$ the reward at step $n$ and $K_t$ the length of the episode. Let the total discounted reward for policy $\rho_i$ be $R_i(T) = \sum_{t=1}^T \bar{x}_i(t)$. Let $R_A(T)$ be the total discounted reward accumulated by a policy-reuse algorithm $A$. Then we require that for each $i$ %either $R_i (T) -R_A(T)  = o(T)$, or, failing that,% 
\[\expc_{R_{N+1},P_{N+1}}\left [\expc_A[R_i(T) - R_A(T)] \right]  = o(T)\] where $\expc_{R_{N+1},P_{N+1}}$ is the expectation is with respect the dynamics of $\mdp_{N+1}$ and $\expc_A$ is the expectation with respect to randomization in the algorithm. 
\end{definition}
In the following subsections we present the EXP-3-Transfer algorithm for solving this problem and then analyze its performance.

%The above definition subsumes the main goal of the PPR algorithm of using the source policies as exploration policies. It also imposes the additional constraint that the source policies should be chosen so that negative transfer vanishes asymptotically -- i.e. exploration using source policies does not harm the base Q-learning performance. 

\subsection{EXP-3-Transfer For Policy Reuse}\label{sec_tralg}

%In this section we present our algorithm for transfer learning. In the subsequent sections, we will present an approach to clustering previous tasks to derive the source tasks that are given as input to this algorithm. In our setup we have a set of source tasks $\mdp_1,\mdp_2,\cdots,\mdp_c$ and a target task $\mdp$ that are all defined on the same state and action space, but with possibly different reward and state-transition functions. Our goal is to re-use the optimal policies $\rho_i$, $1 \leq i \leq c$, of the $c$ source tasks as initial guesses for the optimal policy of $\mdp$ along the same lines as \cite{fernandez_veloso:2006} \cite{fernandez_garcia_veloso:2010}. The difference between their algorithm and ours is that ours is principled. 

In this section, we introduce the non-stochastic multi-armed bandits problem \citep{auer_cesa-bianchi_freund_schapire:2002} and show that the policy reuse problem may be cast as in instance of this problem. We then present EXP-3-Transfer to solve this problem, which is an extension/modification of the classic EXP-3 algorithm \citep{auer_cesa-bianchi_freund_schapire:2002} for non-stochastic multi-armed bandits. We discuss the difference between EXP-3 and EXP-3-Transfer in Section \ref{sec_E3TReg}.

%will describe the non-stochastic bandits problem  and the EXP-3 algorithm for solving it. In the next section we will show that the policy reuse problem is a variation on the non-stochastic bandits problem, and we will introduce our modification of EXP-3, EXP-3-Transfer, for solving policy reuse. 

\subsubsection*{Non-Stochastic Multi-armed Bandits}

The non-stochastic multi-armed bandits problem is defined as follows.
\begin{itemize}
\item There are $c+1$ {\em arms} where each arm $i$ has a {\em payoff process} $x_i(t)$  associated with it. No assumptions are made on the payoff process $x_i(t)$ and they may in fact be adversarial (this is the reason the problem is called non-stochastic).

\item A learner operates for $T$ steps and at each step $t$ it pulls/selects one of the arms $i_t$.

\item At step $t$, the learner obtains payoff $x_{i_t}(t)$ and only observes the payoff of the arm $i_t$ it has selected. 
\end{itemize}

The goal of the learner is to minimize its {\em regret} with respect to the best arm, that is choose arms $i_1,i_2,\cdots,i_T$ to minimize the quantity 
\begin{equation}\nonumber
Reg(i_1,i_2,\cdots,i_T) \eqdef \max_{i} \sum_{t=1}^T x_i(t)  - \sum_{t=1}^T x_{i_t}(t)
\end{equation} 
A randomized algorithm, called EXP-3, was developed in \citep{auer_cesa-bianchi_freund_schapire:2002} which minimizes the {\em expected regret}, where the expectation is taken with respect to the randomization in the algorithm. It turns out that the regret of EXP-3 satisfies the requirements for the regret $R_A$ in Definition \ref{def_polreuse} (we expand on this in the next subsection). This implies that the non-stochastic bandits approach, extended to our setting, solves the problem of policy reuse. 

\subsubsection*{Policy Reuse as Non-Stocastic Multi-armed Bandits}

We may cast the policy reuse problem as a non-stochastic bandit problem as follows.
\begin{itemize}

\item Each source policy $\rho_i$ and the the Q-learning policy corresponds to an arm, giving a total of $c+1$ arms.

\item At each step $t$ in the bandit problem, we (the learner) select a policy/arm $\rho_{i_t}$, and execute it for an {\em episode} in the target $\mdp_{N+1}$ (so a step in the multi-armed bandit problem corresponds to an episode in the policy reuse problem). 

\item The payoff we receive for executing policy $\rho_{i_t}$ (i.e. choosing arm $i_t$) is the (normalized) total discounted reward by executing $\rho_{i_t}$ for that episode: $x_{i_t}(t) = [\Delta R(1-\gamma)^{-1}] [\sum_{n=1}^{K_t} \gamma^n r_n - R_{\min}(1-\gamma)^{-1}]$, with $r_n$ the reward at step $n$ and $K_t$ the length of the episode (the normalization is required as EXP-3 expects arm payoffs to lie in $[0,1]$).

\end{itemize}

\subsubsection*{EXP-3-Transfer}

Given the above transformation, our algorithm EXP-3-Transfer for choosing the policy $i_t$ at step $t$ is given in pseudocode form in Algorithm \ref{alg_EXPTran}. The basic idea is straightforward -- at each step it chooses a policy with probability proportional to an {\em adjusted} cumulative observed reward of the policy and a term to encourage exploring the policy if it has not been selected recently. The rewards are adjusted to compensate for the fact that, at each step $t$, the algorithm observes the payoff of only $\rho_{i_t}$, the policy chosen at that step (see Section \ref{sec_E3TReg} for details on the adjustment). In addition, if the algorithm determines that with high probability a particular source policy (i.e. all but the Q-learning policy) is worse than some other source policy, then it eliminates it from further consideration. The idea is that if we remove an arm we know to be sub-optimal, then we save the episodes that would have been wasted trying that policy.

%The payoffs are adjusted because, at each step $t$, the algorithm observes the payoff of only $i_t$, the arm chosen at that step. The adjustment then ensures that the expectation of the adjusted value is equal to the {\em true} payoff of {\em all} the arms $i$ at step $t$. The expectation above is with respect to the randomziation in the algorithms choice of arms and the randmoziation in the transition and reward functions of $\mdp_{N+1}$. 

In detail, the main loop of EXP-3-Transfer runs from line \ref{stp_ml_start} to line \ref{stp_ml_end}. $C_t$ contains all the policies not yet eliminated by EXP-3-Transfer. Line \ref{stp_defpit} computes the probability of choosing a policy, which is proportional to the adjusted observed payoff of each policy, plus the exploration term ($\beta/|C_t|$). In the next two steps, a policy is chosen probabilistically, executed and its payoff observed and normalized (normalized because EXP-3 expects payoffs to lie in $[0,1]$). In line \ref{stp_updatez} we record the payoff of the executed policy for use in the elimination step. In lines \ref{stp_adjust} and \ref{stp_updatewt}, the adjusted payoffs of the policies are computed, and their weights updated, respectively. Finally, steps \ref{line_ex3_delta_start} to \ref{line_ex3_delta_end} looks at each stationary policy in $C_t$, and checks to see if average {\em non-adjusted} payoff so far satisfies the elimination condition. If so, the policy is removed from $C_t$. The elimination condition is justified/obtained from the theorems below which derive the regret bound of the algorithm.

\begin{algorithm}[th]
   \caption{EXP-3-Transfer$(\mdp,\{\rho_1,\rho_2,\cdots,\rho_c\},\beta,T,l,\Delta R,\delta)$}
   \label{alg_EXPTran}
\begin{algorithmic}[1]
	
	\STATE {\bfseries Input:} MDP $\mdp$, policies $1$ to $c$: the source policies $\rho_1,\cdots,\rho_c$ and EXP-3 parameters $\beta$ and $T$; $l$ the interval at which to eliminate policies; $\Delta R$, upper bound on the range of the per step rewards; $\delta$, confidence parameter for eliminating source policies.

	\STATE {\bf Initialize:} 

\begin{itemize}	
	
	\item  Set Q-learning policy as the $c+1^{th}$ policy.
	
	\item Set $w_i(1) = 1$,  let $x_i(t)$ be the payoff of the policy $i$ at step $t$; let $C_1 \eqdef \{\rho_1,\rho_2,\cdots,\rho_c,\mbox{Q-learning policy}\}$.
	
	\item Set $n_i \lta 0$, where $1 \leq i \leq c$, and $n_i$ is  the number of times $\rho_i$ has been used; set $z_i \lta 0$, where $1 \leq i \leq c$, and $z_i$ is the total normalized discounted reward observed for $\rho_i$ when it was executed. 
	
\end{itemize}

\FOR{$t = 1$ to $T$}\label{stp_ml_start}
	
	\STATE {\bf If} $\rho_i \in C_t$ {\bf then} set $p_i(t) = (1-\beta)\frac{w_i(t)}{\sum_{\rho_i \in C_t } w_i(t)} + \frac{\beta}{|C_t|}$; else set $p_i(t) = 0$. \label{stp_defpit}
	
	\STATE Select policy $\rho_{i_t}$ for step $t$ to be $i$ with probability $p_i(t)$, increment $n_{i_t} \lta n_{i_t} + 1$. 
	
	\STATE Execute policy $\rho_{i_t}$ for one-episode, and observe discounted payoff $\bar{x}_{i_t}(t)$;  normalize $x_{i_t}(t) \lta [\bar{x}_{i_t}(t) - R_{\min}(1-\gamma)^{-1}] /[\Delta R(1-\gamma)^{-1}]$. \label{stp_normTerm}
	
	\STATE {\bf if} $\rho_{i_t}$ is not the Q-learning policy {\bf then} set $z_{i_t} \lta z_{i_t} + x_{i_t}(t)$. \label{stp_updatez}
	
	\STATE For each $\rho_i \in C_t$, set
           \begin{equation}
					\hat{x}_i(t) \lta 
					\begin{cases}
					x_i(t)/p_i(t) &\mbox{ if } i = i_t\\
					 0  &\mbox{ otherwise  }
					\end{cases}
          \end{equation}
     \label{stp_adjust}
          	
	\STATE For each $\rho_i \in C_t$, set $w_i(t+1) \lta w_i(t)\exp[\beta \hat{x}_i(t)/(c+1)]$. \label{stp_updatewt}

	\STATE Set $C_{t+1} \lta C_t$.

	\IF{$t \mod l = 0$}\label{line_ex3_delta_start}

   		\FOR{$k=1$ to $c$, $\rho_k \in C_t$}\label{line_elim_start}

				\IF{ $\exists \rho_j \in C_t,$  $j \leq c$, such that, for $\eps = z_j/n_j - z_k/n_k$, we have \\ $\eps/2 > \sqrt{-\ln (\delta/2c) (2n_j)^{-1}}$ and $\eps/2  > \sqrt{-\ln (\delta/2c)(2n_k)^{-1}}$}\label{stp_deltaRemove}
				
					\STATE Set $C_{t+1} \lta C_{t} -  \{\rho_k\}$.
				
				\ENDIF

		\ENDFOR

	\ENDIF\label{line_ex3_delta_end}

	\ENDFOR\label{stp_ml_end}
\end{algorithmic}
\end{algorithm}

\subsection{Analysis of EXP-3-Transfer}\label{sec_E3TReg}

To begin analysis of EXP-3-Transfer, we first note that the payoffs/discounted cumulative reward $\bar{x}_i(t)$ of a source policy $\rho_i$ is i.i.d., while the payoff of the Q-learning arm has an unknown non-stationary distribution (because the choice of actions in Q-learning is non-stationary). Our strategy for analyzing the performance despite the non-stationarity, is to assume that there is an unknown adversary that is generating the payoffs for Q-learning and then bound the expected worst-case regret of EXP-3-Transfer with respect to this adversary (this is {\em exactly} the strategy used to analyze EXP-3). In particular, in our adversarial/worst-case analysis we assume that there are three participants Nature, Adversary and Player, who make the following choices, in that order.

\begin{enumerate}

\item Player chooses the number of episodes $T$ and the source policies $\rho_1,\rho_2,\cdots,\rho_c$, and all the other parameters for EXP-3-Transfer

\item Adversary chooses the payoffs for $T$ episodes of Q-learning

\item Nature chooses the episodic payoffs $\bar{x}_i(t)$ for each of the policies $\rho_i$ for $1 \leq t \leq T$, according the i.i.d. distribution that governs the payoffs (the distribution is determined by $P_{N+1},R_{N+1}$, the transition and reward functions respectively of $\mdp_{N+1}$, and $\rho_i$).

\item Player now runs EXP-3-Transfer for $T$ episodes, and observed payoffs $x_{i_t}$ are the ones chosen by Nature and Adversary as above.
\end{enumerate}

The important thing to note is that in the above framework, there are two sources of randomness/stochasticity -- first in step 3 due to the randomization due to the transition and reward functions of $\mdp_{N+1}$ and second in step 4 due to choices made by EXP-3-Transfer. As such, our regret bound holds in expectation with respect to these two sources, which we denote by $\expc_{R,P}$ and $\expc_{E3T}$, respectively. The former is expectation with respect to the transition and reward function in $\mdp_{N+1}$, while the latter is expectation with respect to EXP-3-Transfer. It is crucial to note that in taking these expectations, we assume that the choices made by the adversary in step 2 are taken to be fixed (this is identical to what was done in the analysis of EXP-3 in \cite{auer_cesa-bianchi_fischer:2002}). Our main result for this section is as follows:
\begin{theorem}\label{thm_EXPBound}
For any source policy $\rho_j \in C_T$,
\begin{align}\label{eq_thmEXPBound}
V_{N+1}^{\rho_j} - \frac{1}{T} \expc_{R_{N+1},P_{N+1}} \left [\expc_{E3T}\left [ \sum_{t=1}^T \bar{x}_{i_t}(t) \bigg |\rho_j \in C_T \right] \right ] 
&\leq 
\frac{\Delta R}{1-\gamma} \left( \frac{|C_T|}{\sqrt{c+1}} + \sqrt{c+1} \right ) \times \\
\nonumber
& \;\;\;\;\;\;\;\; \sqrt{(e-1)\ln(c+1)/T}\\
\nonumber
& \leq \frac{\Delta R}{1-\gamma} 2 \sqrt{e-1}  \sqrt{(c+1)\ln(c+1)/T}\\
& =  \frac{\Delta R}{1-\gamma} 2.63 \sqrt{(c+1)\ln(c+1)/T}
\end{align}
when the algorithm is run with $\beta = \min\{1,\sqrt{(c+1)\ln(c+1)/[T(e-1)]}\}$. Here $\expc_{R_{N+1},P_{N+1}}$ is expectation with respect to the distributions $R_{N+1}$ and $P_{N+1}$ of $\mdp_{N+1}$, while  $\expc_{E3T}[\cdot|\rho_j \in C_T]$ is the expectation with respect to randomization of EXP-3-Transfer, conditional on the event that $\rho_j \in C_T$. 

Additionally, with probability $1-\delta$, with respect to randomization due to the target MDP $\mdp$, for all source policy $\rho_i \not\in C_T$, there is at least one $\rho_{i'} \in C_T$ such that
\begin{equation}\nonumber
V_{N+1}^{\rho_i} < V_{N+1}^{\rho_{i'}}
\end{equation}
\end{theorem}
The proof is given in Appendix \ref{app_proofs}, but in the following we discuss the bound itself. For the sequel, we define the following function:
\begin{equation}\label{eq_gc}
g(c,T) \eqdef \Delta R (1-\gamma)^{-1} 2.63 \sqrt{(c+1)\ln(c+1)/T}
\end{equation}
This is the right side of the bound above (we will also use it in Section \ref{sec_clustEnc} to quantify how good a particular set of source policies is for the purpose of transfer). 

The first issue we need to check is whether the bound satisfies the requirements we set out in Definition \ref{def_polreuse}. The definition requires that $\expc_{R_{N+1},P_{N+1}}\expc_A[R_i(T) - R_A(T)]  = o(T)$ where $R_i(T)$ and $R_A(T)$ are the total discounted reward over $T$ episodes and the expectations are with respect the randomization due to reward and transition function of $\mdp_{N+1}$ and the randomization in the algorithm (i.e. $\expc_A \equiv \expc_{E3T}$). 

We consider the edge case when $\delta = 0$, and so no arms are ever eliminated and all $\rho_j \in C_T$. In this case, for any $\rho_j$,  
\begin{equation}\label{eq_RjExpand}
\expc_{R_{N+1},P_{N+1}}[\expc_{A}[R_j(T)]]= \expc_{R_{N+1},P_{N+1}}[R_j(T)] = T\times V^{\rho_j}_{N+1}, 
\end{equation}
which is $T$ times the first term in the left hand side of the bound. Now, for any $\rho_j$
\begin{align}\label{eq_RAexpand}
\nonumber
\expc_{R_{N+1},P_{N+1}}\expc_{A}[R_A(T)]
& = \expc_{R_{N+1},P_{N+1}} \Bigg ( P_{E3T}(\rho_j\in C_T) 
\expc_{E3T}  \left [ \sum_{t=1}^T \bar{x}_{i_t}(t)   \bigg | \rho_j \in C_T \right] \\
\nonumber
& \;\;\; + P_{E3T}(\rho_j \not\in C_T) 
\expc_{E3T}  \left [ \sum_{t=1}^T \bar{x}_{i_t}(t)   \bigg | \rho_j \not\in C_T \right]
\Bigg )\\
& = \expc_{R_{N+1},P_{N+1}} \expc_{E3T}  \left [ \sum_{t=1}^T \bar{x}_{i_t}(t)   \bigg | \rho_j \in C_T \right]
\end{align}
By the first part of Theorem \ref{thm_EXPBound},
\begin{equation}
\nonumber
V^{\rho_j}_{N+1} - \expc_{R_{N+1},P_{N+1}}\expc_{A}[R_A(T)] \leq g(c,T) = o(T)
\end{equation}
which is what is required. Now recall from discussion of Definition \ref{def_polreuse} above that satisfying the regret requirements imply that the algorithm avoids negative transfer and transfers when possible. Hence, the theorem shows that EXP-3-Transfer does indeed avoid negative transfer and transfers when possible. 

When we move $\delta$ away from $0$, we trade off adherence to the requirement in Definition \ref{def_polreuse} for practical performance. Indeed, while conducting the experiments reported in Section \ref{sec_exp}, we observed that setting $\delta$ to a suitably low but non-zero value greatly improves performance. However theoretical guarantees are not completely lost in this instance, because for the source policies that are eliminated, part 2 of the theorem says that there is at least one source policy which is, with high probability strictly better than the other policy.

%The theorem tells us that if we use EXP-3-Transfer, then we will not do much worse than the best arm that we had not eliminated, be it one of the previous $\rho_i$ policies or pure reinforcement learning policy like Q-learning (Q-learning is never eliminated as we only consider indices $k \leq c$ in step \ref{line_elim_start}). In particular, if one of the previous policies is close to optimal, we will do almost as well as playing with that policy right from the start, and if none of the policies are good enough, we will not do much worse than Q-learning executed on its own. Additionally, the second part of the theorem says that if there was a policy that was eliminated at some point, then there exists another policy that with high probability had higher value in the target task. Hence, this  essentially accomplishes the goal of policy reuse while minimizing negative transfer, and taking advantage of the i.i.d. payoffs. The rest of the paper is devoted to showing how to compute the $c$ policies $\rho_i$ from a set of $N$ previous MDPs. 

We end this section with a brief discussion of the difference between EXP-3 and EXP-3-Transfer. The primary difference is that we eliminate policies/arms to improve practical performance of the algorithm. This is not possible in EXP-3 because, unlike in our case, it is not assumed that some of the arms have i.i.d. payoffs and so their means cannot be estimated from observations. This leads to a slightly different analysis and improved constant $\frac{|C_T|}{\sqrt{c+1}} + \sqrt{c+1}$ for EXP-3-Transfer and $2$ for EXP-3. The practical difference we observe was actually considerable, and without it we were not able to outperform PPR in our experiments (see below). 

\section{The Clustering Approach To Task Encoding}\label{sec_clustEnc}

\newcommand{\repm}{RM}

Recall from Section \ref{sec_intro} that in our problem, we assume that we have a repository of MDPs and their optimal policies and we wish to use the optimal policies as the source policies for EXP-3-Transfer. Now we have the following dilemma. As the number of source policies increase, we spend more time evaluating them and accruing sub-optimal reward in the process. On the other hand, if we choose a subset of policies from the repository, then we risk leaving out policies that may have been very useful in the new task. So, essentially we are faced with the problem trading off the size and diversity of the set of source policies. 

In this section, we concretely define this tradeoff problem as optimizing a cost function, and in Section \ref{sec_findclust} we describe an algorithm to optimize it. The cost function is defined over the set of {\em clusterings}/partitions of the repository, where each clustering is used to choose a particular subset of policies in the repository as the source policy set. We first show how to choose this source set given a clustering, and then define the cost function that measures how well the source set achieves the tradeoff mentioned above. This cost function then helps us choose the optimal clustering and hence the optimal source policies.

%present the basic idea of the clustering approach and derive the high level structure of a cost function for clusterings which helps us choose the best clustering for the purposes of transfer. After that, we derive two different versions of the cost function under worst-case and best-case assumptions on the target task. The worst-case cost function is pleasing theoretically, but empirically leads to poor results (please see Section \ref{sec_exp}), while for the best-case cost function the opposite turns out to be true.

\subsection{Constructing Source Policies Given Clustering}

Given the $N$ MDPs $\mdp_1,\mdp_2\cdots,\mdp_N$ in our repository, let $\{A_1,A_2,\cdots,A_c\} \eqdef \bfA$ be a particular clustering -- that is $\cup_i A_i=$ $\{\mdp_1,\mdp_2\cdots,\mdp_N\}$ and $A_i \cap A_j = \emptyset$. The set of clusterings may vary both by elements of each $A_i$, and the number $c$ of {\em cells} $A_i$.  Given a particular clustering $\bfA$, we will obtain the $c$ source policies by choosing one policy per $A_i$. To choose this policy, we define a distance function that measures how similar two MDPs are in terms of their optimal policies.

Let $\mdp_i$ and $\mdp_j$ be any two MDPs in our repository (the distance definition generalizes to any two MDPs defined on the same state and action space, but with different transition and reward functions). Recall that $V_i^\pi$ and $V_j^\pi$ denotes the value of policy $\pi$ when executed, respectively, in $\mdp_i$ and $\mdp_j$ at the initial state (this can be generalized to different initial states and/or distribution over initial states very easily -- but we consider the same initial state setting to keep the presentation simple). Letting the optimal policies for the two MDPs be $\pi_i^*$, $\pi_j^*$, we define the {\em optimal policy similarity} between two MDPs as follows.
\begin{equation}\label{eq_mdpDDef}
d_V(\mdp_i,\mdp_j) \eqdef \max \{  V_i^* - V_i^{\pi_j^*}, V_j^* - V_j^{\pi^*_i} \}
\end{equation}
So this distance upper bounds how much we lose if we use the optimal policy of one MDP in the other -- in particular we have the following lemma by construction.
\begin{lemma}\label{lemma_dvOpt}
If $d_V(\mdp_i,\mdp_j) \leq \eps$, then the optimal policy of $\mdp_i$ is at least $\eps$-optimal in $\mdp_j$ and vice versa.
\end{lemma}
The definition of $d_V$ is motivated by the fact that the  goal of policy reuse is to use the optimal policy of one MDP in another. We now define the source policies given a clustering $\bfA$:
\begin{definition}\label{def_srcPol}
Given a clustering $\bfA = \{A_1,A_2,\cdots,A_c\}$, define for each $A_i$ the MDP $\mdp^i$ as follows:
\begin{equation}\label{eq_landmark_dV}
\mdp^i \eqdef \argmin_{\mdp \in A_i} \max_{\mdp' \in A_i} d_V(\mdp,\mdp')
\end{equation}
(ties broken in terms of order in the sequence $\mdp_1,\mdp_2,\cdots,\mdp_N$). 
Then the source policies corresponding to $\bfA$ are $\{\rho_1,\rho_2,\cdots,\rho_c\}$ where $\rho_i$ corresponds to $A_i$ and is the optimal policy of $\mdp^i$.
\end{definition}
The definition is illustrated in Figure \ref{fig_clustIll} and it means the following.  $\mdp^i$ is the element of $A_i$ that minimizes the maximum $d_V$ distance to the other elements of the cluster, and hence, by Lemma \ref{lemma_dvOpt}, is in a worst case sense the best representative of the MDPs in cluster $A_i$. By choosing the optimal policy of $\mdp^i$ as a source policy $\rho_i$, we ensure have the best worst-case representation of the MDPs in each cluster $A_i$. We make this final statement exact in the next section, in particular in Lemma \ref{lemma_avgQuant} when we bound the regret of EXP-3-Transfer with respect to optimal policy of {\em any} MDP in the repository (rather than just the source policies as done in Section \ref{sec_polReuse}).

\begin{figure}[h]
\begin{center}
\includegraphics[scale=0.5]{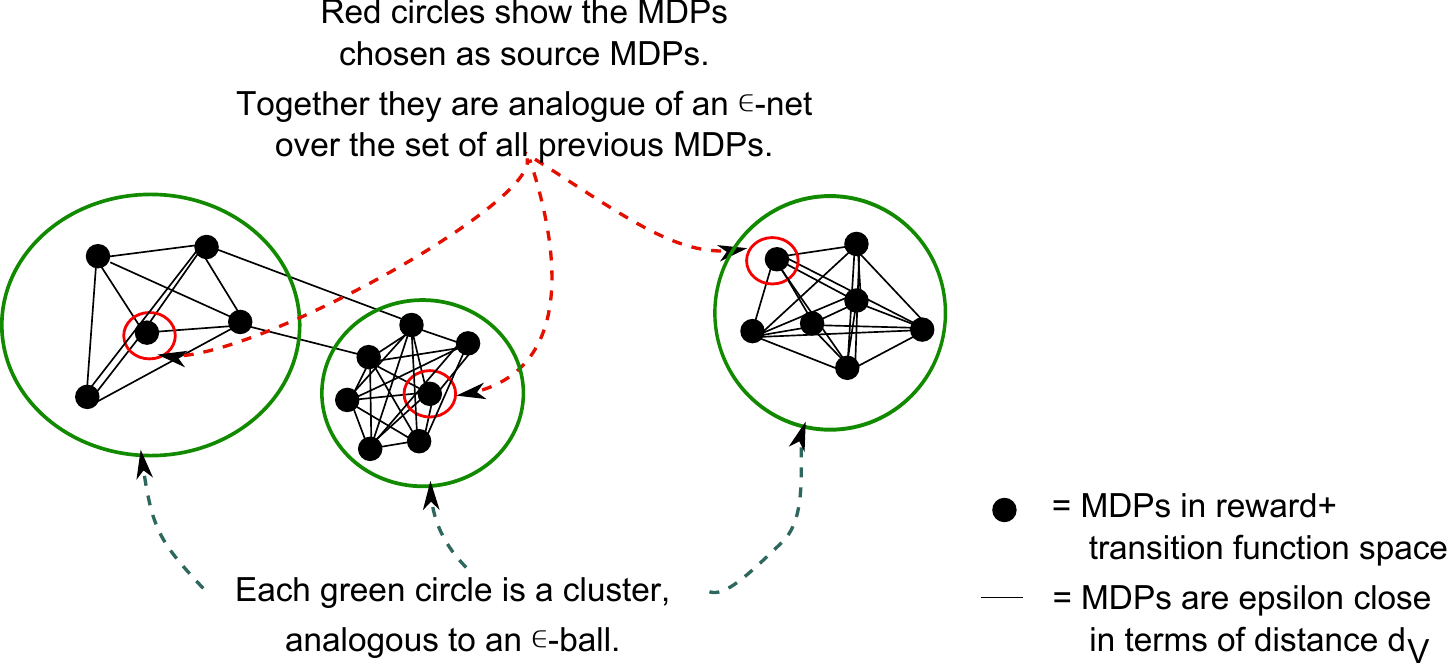}
\end{center}
\caption{This figure sketches our basic approach to deriving the source policies. The black circles represent our repository of $N$ MDPs. The goal is to put them into $c$ clusters and then derive $c$ source policies from the $c$ source tasks. The figure illustrates the idea for $c = 3$. Each cluster is an analogue of an  $\eps$-ball in a metric space according to $d_V$ (\ref{eq_mdpDDef}). The source MDPs form an analogue  of an $\eps$-net  over the set of previous MDPs with respect to $d_V$. The function $d_V$ measures how well the policy of one MDP performs in the other -- and hence the source policies being an $\eps$-net implies that, given any MDP in the repository, there is at least one source policy which has performance that is `$\eps$-close' to the performance of the optimal policy of the previous MDP. }\label{fig_clustIll}
\end{figure}

\subsection{Cost Function of a Clustering}\label{sec_mdpN1prev}

Let $\bfA = \{A_1,A_2,\cdots,A_c\}$ be a clustering with the $c$ source policies $\rho_1,\rho_2,\cdots,\rho_c$ chosen as defined in the previous subsection. Our goal in this section will be to quantify the regret of EXP-3-Transfer with respect to any $\pi^*_k$ when executed on the new task $\mdp_{N+1}$ for $T$ episodes, where $\pi^*_k$ is the optimal policy of $\mdp_k$ in the repository $\mdp_1,\mdp_2,\cdots,\mdp_N$. We will consider the case where EXP-3-Transfer is run with source policies $\rho_1,\rho_2,\cdots,\rho_c$, and hence this regret will quantify how good the clustering $\bfA$ is for transfer -- in particular, the lower the regret, the more preferable the clustering. As a result, the regret will serve as the cost function for choosing a clustering to derive the source policies. Before proceeding, note that in Theorem \ref{thm_EXPBound}, we computed the regret between EXP-3-Transfer and $\rho_i$ -- in this section we extend those results to derive the regret with respect to $\pi^*_k$.

%Our goal in this section will be to quantify the difference in the expected accumulated discounted reward between  EXP-3-Transfer and the best policy hindsight in the set $\{\pi_i^*\}_i \cup \{\mbox{Q-learning policy}\}$. We quantify the difference for the case when the policies are used on the new MDP $\mdp_{N+1}$ for $T$ episodes, where EXP-3-Transfer using the arms $\rho_i$ and Q-learning, to the expected accumulated reward of the policy that obtains best 

To begin, let the diameter of a cluster $A_i$, and the mean diameter of the clustering $\bfA$  be:
\begin{equation}\label{eq_bareps}
\eps_i \eqdef  \max_{\mdp \in A_i} d_V(\mdp^i, \mdp), \;\;\; \bar{\eps} \eqdef \frac{1}{N} \sum_{i=1}^c |A_i|\eps_i
\end{equation}
So $\bar{\eps}$ is the average diameter of the clusters, weighted by the size of the clusters. Therefore, $\bar{\eps}$ gives the average distance from a cluster center $\mdp^i$ to an element of the cluster. As such, $\bar{\eps}$ measures how much of the diversity in the repository is captured by the chosen clusters.  The smaller $\bar{\eps}$, the smaller the average $d_V$ distance between the cluster centers $\rho_i$ and the MDPs in $A_i$, and hence more of the diversity of policies in the repository is captured by the $\rho_i$. 

\footnotetext[2]{A worst case quantification is also possible. However, the assumptions underlying the worst case seems too weak, and the cost function correspondingly not sufficiently discriminating -- i.e. it identifies MDPs that are intuitively dissimilar as being similar. In particular in our experiments, we found this cost function to not give us the intuitive clusters. We discuss this function further in appendix \ref{app_worstCase}.}

%, \mbox{ and } \;\;\; \bar{K} \eqdef \frac{1}{N}\sum_i K(k) 

Using these, we can give an average case quantification\footnotemark of the performance of a clustering $\bfA$ when used to generate the source policies and used in EXP-3-Transfer. Let the new MDP $\mdp_{N+1}$ have transition and reward functions $R_{N+1}$ and $P_{N+1}$. Define $K_k = \max_{s,a}$ $|R_k(s,a) - R_{N+1}(s,a)|$ and $K'_k = \max_{s,a} |P_k(.|s,a) - P_{N+1}(.|s,a)|$ where $R_k$ and $P_k$ are the reward and transition functions for MDP $\mdp_k$. Additionally let
\begin{equation}\label{eq_Kidef}
K(k) \eqdef \frac{K_k + \gamma K'_k R_{max}}{(1-\gamma)^2} 
\end{equation}
We have the following result which derives from Theorem \ref{thm_EXPBound}.
\begin{lemma}\label{lemma_avgQuant}
If EXP-3-Transfer is run with source policies derived from $\bfA$ using definition \ref{def_srcPol} with $\beta$ set as in Theorem \ref{thm_EXPBound},
then for $\pi^*_k \in A_i$, such that EXP-3-Transfer did not eliminate $\rho_i$ by step $T$, the following is true: 
\begin{equation}\label{eq_avgQuant1}
V^{\pi^*_k}_{N+1} - \expc_{R,P}\left [\expc_{E3T}\left [ \sum_{t=1}^T \bar{x}_{i_t}(t) \bigg |\rho_i \in C_T \right] \right ] \leq g(c) + \eps_i + 2K(k)
\end{equation}
Additionally, with probability $1-\delta$, with respect to randomization due to the target MDP $\mdp_{N+1}$, for each $\pi^*_k \in A_i$ such that $\rho_i$ was eliminated, there exist $\rho_{i'}$ such that 
\begin{equation}\nonumber
V_{N+1}^{\pi^*_k} \leq V_{N+1}^{\rho_{i'}} + \eps_i + 2K(k)
\end{equation}
\end{lemma}
The proof is in Appendix \ref{app_proofs}.

We now use the lemma to derive the cost function. First, in the limiting case of $\delta = 1$, none of the policies $\rho_i$ are eliminated and in this case the bound (\ref{eq_avgQuant1}) applies to all the MDPs in the repository. In the lemma we assume that each MDP $\mdp_k$ in the repository is equally likely to be the one with the minimum $K(k)$. Hence, taking the average of (\ref{eq_avgQuant1}) over all the MDPs, we get the upper bound of the average regret with respect to all the optimal policies $\pi^*_k$ of MDPs in the repository:
\begin{equation}\label{eq_avgQuant}
g(c) + \bar{\eps}  + 2 \frac{1}{N}\sum_i K(k)
\end{equation}
$K(k)$ depends on the repository of $\mdp_k$s which we do not control or make any assumptions about. Hence, we associate with each cluster $\bfA$ the parameters $(c,\bar{\eps})$, defined at the beginning of this section, and use (\ref{eq_avgQuant}) to define the following cost:
\begin{definition}\label{def_cost2}
The cost of a clustering $\bfA$ with parameters $(c,\bar{\eps})$ is defined to be:
\begin{equation}
cost(\bfA) \eqdef g(c) + \bar{\eps}
\end{equation}
\qed
\end{definition}

%However, $\bar{K}$ is a property of the $\mdp_i$s, which we do not make any assumptions about. Hence, we can ignore $\bar{K}$ and use $g(c) + \bar{\eps}$ as an average case quantification of the performance of cluster $\bfA$. 

%\begin{equation}\label{eq_avgMaxdiam}
%\bar{\eps}_m = \frac{1}{N}\sum_i |A_i| \bar{\eps}_m^i, \mbox{ where } \eps_m^i %\frac{1}{|A_i|} \sum_{\mdp \in A_i} \max_{\mdp' \in A_i} d_V(\mdp,\mdp')
%\end{equation}

\subsection{Hardness of finding the Optimal Clustering}\label{sec_clustHard}

In this section we introduce the problem of finding the clustering that minimizes $cost()$. We argue that optimizing the cost function is hard, which sets the stage for developing our discrete optimization algorithm in the next section. Specifically we will show that it is hard to optimize an upper bound $cost_{m}()$ of $cost()$ where $cost()$ was defined in Definition \ref{def_cost2}. To that end, define the average max-diameter of a clustering $\bfA$ to be:
\begin{equation}\label{eq_avgMaxdiam}
\bar{\eps}_m = \frac{1}{N}\sum_i \sum_{\mdp \in A_i} \max_{\mdp' \in A_i} d_V(\mdp,\mdp')
\end{equation}
Now define, 
\begin{definition}
Define $cost_{m}(\bfA) \eqdef g(c)+ \bar{\eps}_m$.
\end{definition}
We have the following relationships.
\begin{lemma}
The parameter $\bar{\eps}_m$ of $\bfA$ is an upper bound on the parameter $\bar{\eps}$ of $\bfA$ defined in (\ref{eq_bareps}). Furthermore, $cost_{m}(\bfA)$ $\geq$ $cost(\bfA)$ for all clusterings $\bfA$. 
\end{lemma}
\begin{proof}
This follows directly from the definition of $\bar{\eps}$ -- in particular, 
$\bar{\eps}_m$ upper bounds $\bar{\eps}$  because for each $\mdp_j \in A_i$, $\max_{\mdp' \in A_i}  d_V(\mdp_j,\mdp')$ $\geq \min_\mdp \max_{\mdp'} d_V(\mdp,\mdp')$ $= d_V(\mdp^i,\mdp)$. The second part of the lemma now follows by the definitions of  $cost()$ and $cost_m()$.
\end{proof}

We reduce the minimum clique-cover problem \citep{karp:1972} to the problem of finding the clustering that minimizes  $cost_{m}$ and hence establish that it is NP-complete. We start by describing the clique cover problem. Let $G  = (V,E)$ be a graph where $V$ is the set of vertices and $E$ is the set of edges. A subset $V' \subset V$ is a clique if for any $v,v' \in V$, there is an edge $(v,v') \in E$. The minimum clique cover problem is finding a partition $V_1,V_2,\cdots,V_n$ of $V$ such that each $V_i$ is a clique and $n$ is minimum -- that is there exists no other partition with $V'_1,V'_2,\cdots,V'_m$ of $V$ such that each $V'_i$ is a clique and $m < n$. We have the following theorem for $cost_m$.
\begin{theorem}\label{thm_npcomplete}
Given a graph $G = (V,E)$, in time polynomial in $|V|$ and $|E|$, we can reduce the minimum clique cover problem for $G$ to finding the clustering $\bfA^*$ of some set of MDPs $\mdp_1,\mdp_2$, $\cdots$, $\mdp_{|V|}$, with all $\mdp_i$ defined on the same state and action spaces, where $\bfA^* \eqdef \argmin_{\bfA \in\clstngs}$ $cost_{m}(\bfA)$.
\end{theorem}
The proof is given in Appendix \ref{app_proofs}. Since the clique cover problem is NP-complete, we immediately have the following corollary. 
\begin{corollary}
Finding the clustering minimizing the upper bound $cost_{m}()$ of $cost()$ is NP-complete. 
\end{corollary}
Unfortunately, we do not yet have a proof that minimizing $cost()$ is hard -- but the fact that minimizing the upper bound $cost_m()$ is hard, leads us to conjecture that minimizing $cost()$ is hard as well. For this reason, to optimize the $cost()$ function we need a discrete optimization function, which we will develop in the next section.

We end this section by contrasting our approach to previous approaches to clustering MDPs in TLRL (for instance \citep{wilson_fern_ray_tadepalli:2007}). In prior work, MDPs are typically characterized by a finite number of real valued parameters, and the distance between parameters determine how similar the MDPs are. These MDPs can then be clustered by, for instance, putting the non-parametric Dirichlet process prior over the parameters of all MDPs, and then using Monte Carlo inference methods to find the clustering that maximizes the posterior probability. 

Since our notion of similarity between MDPs is based on policies, to apply this approach to our case, we need (1) a relatively compact parametric representation of optimal policies, and (2) a metric that relates the policy-parameters to values of policies and (3) that optimal policies uniquely characterize MDPs. (1) seems difficult for interesting policies, and (2) seems reasonable only for `linear' domains -- i.e. domains where a small change in policy results in a corresponding small change in its value. (3) runs completely counter to one of our main motivations for using policy based clustering, which is that different MDPs might have identical or  near-identical optimal policies. So given all these, we were motivated to construct a clustering algorithm adapted to our policy based clustering of MDPS. We also compare our method and the method of \cite{wilson_fern_ray_tadepalli:2007} in Section \ref{sec_colGrid}, where in a simplified version of their domain, our approach recovers better clusterings.

\section{Finding the Optimal Clustering}\label{sec_findclust}

In Section \ref{sec_clustEnc} we defined a cost function that measures how well a particular clustering of the repository of MDPs trades off size and diversity of the set of source policies that are obtained from the clustering. In Section \ref{sec_clustHard} we argued that optimizing $cost()$ is hard. This implies that we need a discrete global optimization algorithm for finding the optimal clustering. So in Section \ref{sec_optProb}, we introduce the problem of discrete optimization and our approach to solving it. Then in Sections \ref{sec_discOpt} to \ref {sec_optparam} we derive and analyze our general optimization algorithm. Finally, in Section \ref{sec_srchClust}, we apply the algorithm to the problem of finding the optimal clustering.

\subsection{Global Optimization With Stochastic Search}\label{sec_optProb}

Our goal is to solve the discrete global optimization problem of computing $\argmin_{A \in \clstngs} cost(A)$, where $\clstngs$ is the set of all possible clusterings of $\mdp_i$s.  Algorithms for discrete optimization problems tend to fall into specific classes, appropriate for the problem at hand. We present a {\em stochastic search} algorithm for our problem (see Figure \ref{fig_stochOpt}).

\begin{figure}[h]
\begin{center}
\includegraphics[scale=0.3]{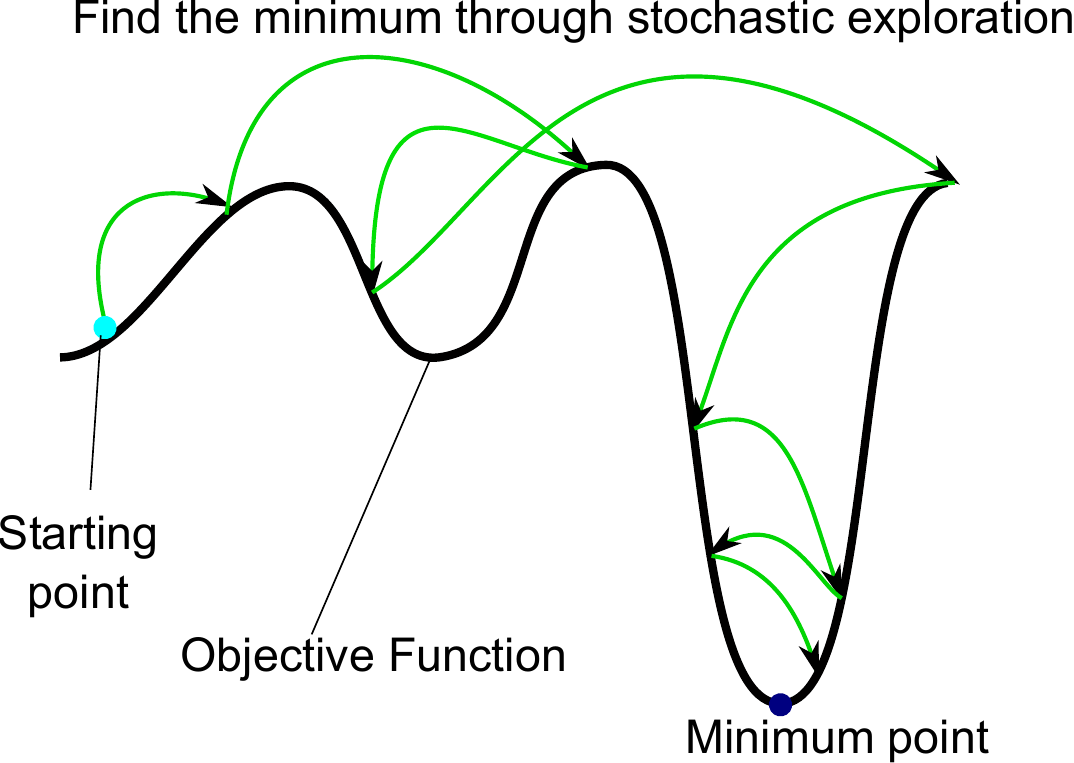}
\end{center}
\caption{This figure illustrates how stochastic search may be used solve function optimization (minimization in this case). The thick black curve is the {\em objective function} to be optimized/minimized and the goal is the find the point $x$ at which the curve the attains the minimum value (the blue circle). A stochastic search algorithm starts at a particular point (the cyan circle in this figure) and each time step the it jumps (shown by thin green arrows) to a new candidate point, chosen according to some stochastic strategy. The arrows shown are one possible run of the stochastic search algorithm, with different runs likely going through different sets of points. The candidate may move towards and away from the minimum point. This kind of optimization is necessary when the objective function does not have nice properties (like convexity) and standard algorithms (like gradient descent) are not applicable.}\label{fig_stochOpt}
\end{figure}

Our basic strategy is to construct a distribution over $\clstngs$, that concentrates around the optimum and around clusterings with low cost. The concentration property implies that if we repeatedly sample from this distribution, we will find the optimum or a good/low cost clustering with high probability. However, in general, exact sampling from such distributions is difficult, and so our algorithm samples approximately from this distribution using a Markov chain Monte Carlo approach -- see  \citep{robert_casella:2005} for a comprehensive introduction to MCMC and Metropolis Hastings Markov chains (MH chain in short) that we use. The use of MCMC turns our algorithm into a stochastic search algorithm.

\begin{figure}[h]
\begin{center}
\includegraphics[scale=0.3]{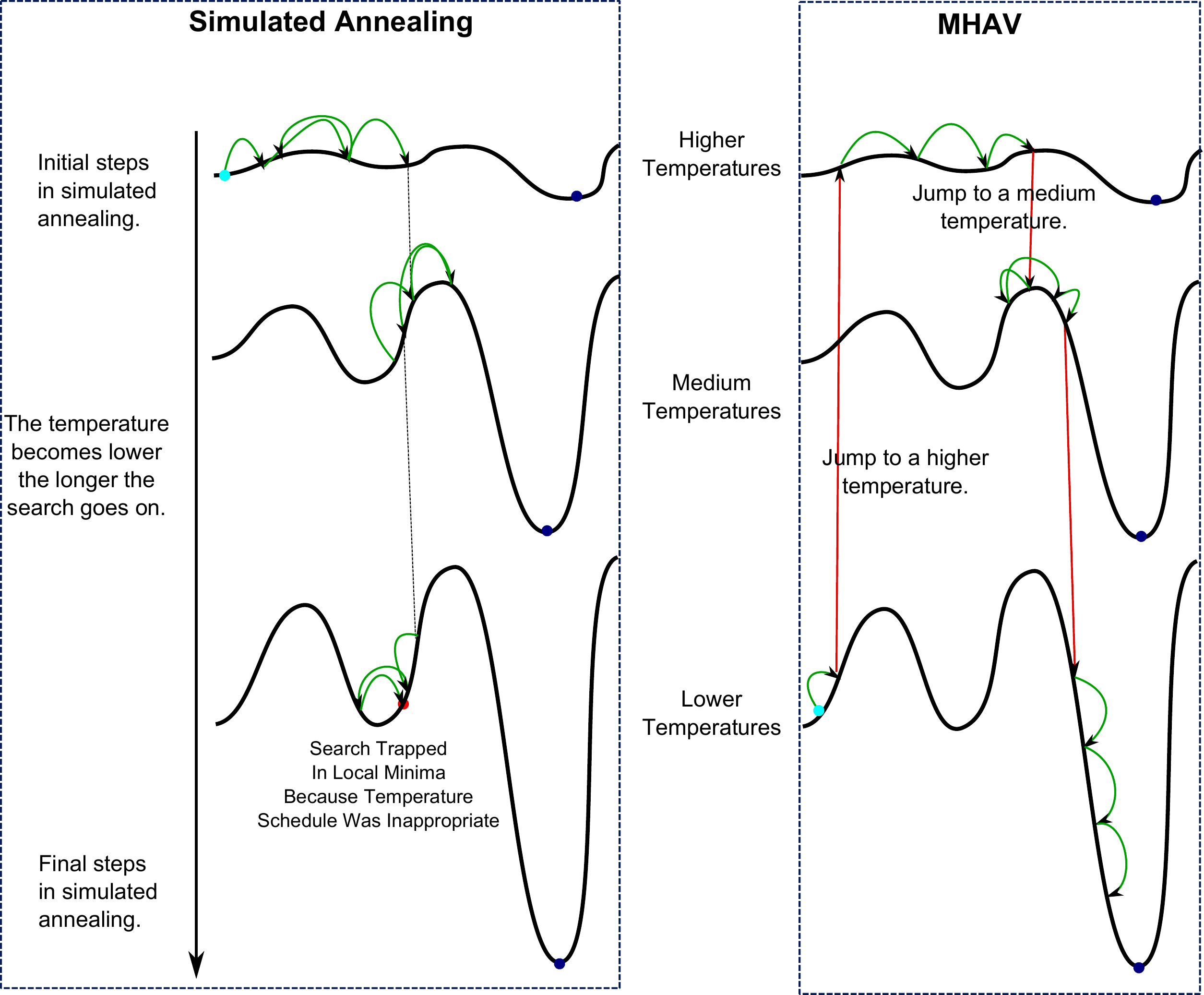}
\end{center}
\caption{This figure compares simulated annealing and our MHAV  algorithm for finding minimum of an objective function $f$. The green arrows show the search steps taken by the algorithms. Simulated annealing decreases a {\em temperature} parameter over time, which effectively results in the function surface becoming effectively less and less flat. The figure illustrates a search that has failed due to incorrect temperature decrease schedule. Unlike simulated annealing, MHAV jumps between temperature values stochastically. The green arrows are search steps for a fixed temperature, while the red arrows are jumps between temperatures. See beginning of Section \ref{sec_optProb} for a full description.}\label{fig_sim_mhav}
\end{figure}

%\begin{figure}[h]
%\begin{center}
%\includegraphics[scale=0.35]{sim_anneal.pdf}
%\end{center}
%\caption{This figure illustrates the simulated  annealing algorithm for stochastic search to find minimum of an objective function $f$. The green arrows show the search steps taken by the algorithm. The algorithm increases the {\em temperature} over time, which results in the function surface becoming less and less flat. See beginning of Section \ref{sec_findclust} for a full description. The figure illustrates a search that has failed due to incorrect temperature schedule.}\label{fig_simAnneal}
%\end{figure}

The resulting algorithm can be understood as simulated annealing \citep{kirkpatrick_gelatt_vecchi:1983} but with {\em stochastic temperature changes}. Figure \ref{fig_sim_mhav} describes simulated annealing qualitatively and also contrasts it with our algorithm, MHAV. Simulated annealing is a stochastic search algorithm for optimizing an objective function $f$, where the search rule at step $t$ jumps from point $y$ to $y'$ with probability $P_t(y,y') = q_{y,y'} \exp[-\frac{1}{z(t)}\max(0,f(y')-f(y)]$. $q_{y,y'}$ is a candidate distribution (probability with which $y'$ is proposed as the next point given current point $x$) and the remaining term depends on the improvement $f(y')-f(y)$ and the {\em temperature} $z(t)$ (see below). Because of the $\max$ in the improvement term, jumps to lower $f$-valued (i.e. better) points $y'$ succeed with probability 1, while jumps to higher $f$-valued (i.e. worse) points $y'$ succeed with probability depending on how much higher (worse) the point is. %Hence for `flatter' objective functions $f$, the search rule is able to cover much more of the function surface. 

The temperature $z(t)$ is a chief feature of simulated annealing, it is user defined and determines practical success of the algorithm. It is a sequence decreasing in $t$ and in essence changes how easy it is to explore the objective function surface -- with higher temperature allowing the search to travel over longer distances and in effect making the objective function flatter, and lower temperature restricting the search to points local to the current point, in effect making the objective function steeper  (Figure \ref{fig_sim_mhav} illustrates three qualitatively different temperatures). The temperature sequence $z(t)$ needs to be set very carefully so that during the higher temperature phase the search travels longer distances to the places where the minimum point is, and then during the lower temperature phases, the search moves in a `local neighborhood' of the minimum point and and finds the point itself or a point close to it. Clearly, this is a fairly difficult problem and requires some understanding of the objective function. 

Figure \ref{fig_sim_mhav} illustrates our algorithm, MHAV. The algorithm is similar to simulated annealing in that the search rule is such that jumps to better points succeed with probability 1, while jumps to worse points succeed with probability depending on how much worse the point is. It also uses temperatures to modify the flatness of the function (in the algorithm description in Section \ref{sec_discOpt} the temperature parameter is represented by $\lambda$, with high temperature corresponding to low $\lambda$). However, unlike simulated annealing, MHAV does not use a fixed temperature schedule but moves between different temperatures stochastically (Figure \ref{fig_sim_mhav} illustrates the algorithm jumping between three qualitatively different temperatures). In essence MHAV searches through the augmented space $(\lambda,y)$ and the search rule has the form $P_t[(\lambda,y),(\lambda',y')]] \approx q_{(\lambda,y),(\lambda',y')} \exp\{\max[0, -(f(y')\ln\lambda' - f(y)\ln\lambda]\}$. That is, MHAV searches over both the temperature and the solution space simultaneously and so avoids the very difficult problem of needing to set the temperature schedule. Convergence is still guaranteed by convergence of the Metropolis-Hastings algorithm (see Section \ref{sec_MHC}). Our proof of the convergence, and the speed of convergence of the algorithm is simple compared to simulated annealing (see, for instance, \citep{locatelli:2000} for contrast).

We present our algorithm in steps. We first cast the discrete optimization problem as the problem of sampling from a specific distribution derived from the objective  function (Section \ref{sec_discOpt}), then we present the general Metropolis-Hastings scheme for approximate sampling from a distribution (Section \ref{sec_MHC}), and after that we present and analyze our adaptation of Metropolis-Hastings for discrete optimization (Sections \ref{sec_mhav}, \ref{sec_MHAVconv} and \ref{sec_optparam}), which we call MHAV (Metropolis-Hastings with a auxiliary variable). The MHAV algorithm is a general optimization algorithm, which we then adapt to the problem of searching for the optimal clustering (Section \ref{sec_srchClust}). %Finally, we analyze the MHAV algorithm in Sections 

\subsection{Optimization as Sampling}\label{sec_discOpt}

In this section, we show how to convert the problem of global optimization to the problem of sampling from a distribution. The method we discuss was inspired by the basic idea behind simulated annealing \citep{kirkpatrick_gelatt_vecchi:1983}. Assume that our goal is to minimize a cost function $f$ defined over some finite set $Y$. In particular assume that there is a subset $\hat{Y} \subset Y$ for which $y\in \hat{Y}$ has acceptable cost $f(y)$. Let $\Lambda = \{\lambda_1,\lambda_2,\cdots,\lambda_n\}$, $\lambda_i < \lambda_{i+1}$ such that $\exists \hat{\Lambda} \subset \Lambda$ which satisfies 
\begin{equation}
\sum_{\lambda \in \hat{\Lambda}, y \in \hat{Y}} \lambda^{-f(y)} \geq \theta > 0 
\end{equation}
We now define the distribution 
\begin{equation}\label{eq_lambdaPi}
\bar{\Pi}(\lambda,y) \eqdef \lambda^{-f(y)} Z^{-1}
\end{equation}	
where $Z \eqdef \sum_{y,\lambda} \lambda^{-f(y)}$ is the normalization term. Given the existence of $\hat{\Lambda}$, if we draw repeatedly from $\bar{\Pi}$, then after $t$ draws, with probability at least
\begin{equation}\label{eq_sampleDist}
1- (1-\theta)^{t}
\end{equation}
we will draw an element $(\lambda,y)$ where $y \in \hat{Y}$. Since $1-\theta$ $< 1$, the probability that we draw an element from $y \in \hat{Y}$ goes to $1$ at rate $(1-\theta)^t$ (hence, the closer $\theta$ is to $1$, the faster the convergence rate will be). So this solves the solves the discrete global optimization problem of optimizing the function $f$. 

The reason we introduce the parameter $\lambda$ rather than just sampling from  $Pr(y) \eqdef \frac{f(y)}{\sum_{y'}f(y')}$ is because this distribution is typically intractable to sample from. Indeed, for our MDP clustering problem, the function $f$ is $cost()$ and the set $Y$ is the set of clusterings $\clstngs$. By the the hardness result for $cost()$ in Appendix \ref{app_proofs}, we conjecture that sampling from $Pr(y)$ is in fact intractable. Hence, we need to use approximate sampling methods. 

Use of approximate sampling turns sampling from $Pr(y)$ into a stochastic search over the objective function $f()$ to find the global minima. As detailed in the previous subsection, one powerful way to augment such search methods is to modify the objective function by introducing temperatures to change the flatness of the function, and the $\lambda$ parameter serves precisely this purpose. In particular, when $\lambda$ is high (corresponding to `low temperature'), the modified objective function $\lambda^{-f(y)}$ is `steep', and when $\lambda$ is low (corresponding to 'high temperature') $\lambda^{-f(y)}$ is `flat'. 

Given this motivation, in Section \ref{sec_MHC}, we present the Metropolis-Hastings (MH) algorithm that may be used to approximately sample from arbitrary distributions, and then in Section \ref{sec_mhav} we adapt the MH algorithm to sample from $\bar{\Pi}$.

\subsection{Sampling Using Metropolis-Hastings Chains}\label{sec_MHC}

\newcommand{\msttspc}{\mathcal{X}}
\newcommand{\bfX}{\mathbf{X}}

In this subsection we describe a standard method to approximately sample from a distribution $\Pi$ over a large finite space $\msttspc$. In the next section we will use this method to sample from $\bar{\Pi}$ and complete our global optimization method. We use upper-case Roman letters for random variables and lower-case letters to refer to their realized values. In the following, we use the theory of Markov chains as found in (for instance) \citep{levin_peres_wilmer:2009}. A {\em Markov chain} over a (finite) state-space $\msttspc$ is stochastic process $X_n$ taking values in $\msttspc$ such that $Pr(X_n=x| x_0,x_1,x_2,\cdots,x_{n-1})$ $=P_n(x|x_{n-1})$. The distribution $P_n(\cdot|\cdot)$ is called the {\em transition kernel} for the chain, and can be represented by a $|\msttspc|\times |\msttspc|$ matrix, also denoted by $P_n$, such that the entry $(x,y)$ is $P_n(y|x)$ (here we have identified each element of $\msttspc$ with an integer in $\{1,2,\cdots,|\msttspc|\}$ in some order). A Markov chain is said to be time-homogeneous if $P_n(x'| x) = P(x'|x)$, i.e. $P_n$ is independent of time $n$. We will only consider time-homogeneous chains. A distribution $\Pi$ over $\msttspc$ is said to be {\em stationary} for the chain with kernel $P$ if it satisfies:
\begin{equation}\label{eq_station}
\Pi(x') = \sum_{x} \Pi(x) P(x'|x) 
\end{equation}
That is if $P$ is stationary for $\Pi$, then, if we draw $x$ according to $\Pi(x)$, then choose $x'$ according to $P(x'|x)$, then the distribution over $x'$ will also be $\Pi(x')$.

Let $P^n(x'|x)$ be the probability that $X_n = x'$ given that $X_0 = x$, that is 
\begin{equation}\nonumber
P^n(x'|x) = \sum_{x_1,x_2,\cdots x_{n-1}} \prod_{i=0}^{n-1} P(x_{i+1}|x_{i}), \qquad \mbox{ where } x_0 = x, x_n = x'
\end{equation}
Then the chain $X_n$ (equivalently, the kernel $P(\cdot|\cdot)$) is said to be {\em irreducible} if for each $x,x' \in \msttspc$, $\exists n$ with  $P^n(x'|x) > 0$. It is called {\em aperiodic} if the set $\{n: P^n(x|x)>0\}$ has greatest common divisor $1$ -- that is there is no period to the set of time steps at which the chain returns to some state $x$, starting from $x$ itself. 
\begin{theorem}\label{thm_stationConv}
The following are true for any aperiodic and irreducible Markov chain with kernel $P$:
\begin{enumerate}
\item $P$ has a stationary distribution $\Pi$ and for any $y \in \msttspc$,
\begin{equation}
\lim_{n\rta \infty} \ltvnorm{P^n(\cdot|y) - \Pi(\cdot)} = 0
\end{equation}
where $\ltvnorm{P-P'} = \sup_{A\subset \msttspc}|P(A)-P'(A)|$ is the total variation distance between any two distributions $P,P'$ over $\msttspc$. 
\item If  $\Pi$ is stationary for $P$ and $\ltvnorm{P^n(\cdot|y) - \Pi(\cdot)} \leq k$ then $\ltvnorm{P^{n'}(\cdot|y) - \Pi(\cdot)} \leq k$ for all $n' > n$.
\end{enumerate}
\end{theorem}
\begin{proof}
For the first part and second part, see (for instance), respectively, Theorem 4.9 and Lemma 4.12 in  \citep{levin_peres_wilmer:2009}). \qed
\end{proof}
This result is important because it can be used to approximately sample from a distribution $\Pi$ that is hard to sample from directly. The idea is to construct a Markov chain $X_n$ with stationary distribution $\Pi$. Theorem \ref{thm_stationConv} implies that if we simulate $X_n$ long enough, then eventually we will start sampling from $\Pi$. To that end, the Metropolis-Hastings chain (MH chain in short) gives a standard way to define such a chain given $\Pi$ as input (see  \cite{robert_casella:2005} for an in-depth introduction).

A MH chain is defined via an irreducible kernel $\phi(x'|x)$ over $\msttspc$ and an acceptance probability $\accpt_x(x')$. $\phi$ is problem dependent while $\accpt$ is defined as follows:
\begin{equation}\label{eq_accpt}
\accpt_x(x') \eqdef \min\left \{ 1, \frac{\phi(x|x')\Pi(x')}{\phi(x'|x)\Pi(x)} \right \}
\end{equation}
Given this, the MH chain has transition
\begin{equation}\label{eq_PMHDef}
P_{MH}(x'|x) = 
\begin{cases}
\phi(x'|x)\accpt_x(x'), &\mbox{ if } x \neq x' \\
1-\sum_{x'\neq x}\phi(x'|x)\accpt_x(x'), &\mbox{ otherwise }, 
\end{cases}
\end{equation}
It can be easily checked that $P_{MH}$ satisfies the {\em detailed balance} equation $\Pi(x)P_{MH}(x'|x) = \Pi(x')P_{MH}(x|x')$ which in turn is equivalent to (\ref{eq_station}) (which can be seen by summing both sides over $x'$). So $P_{MH}$ is a chain which if simulated long enough will sample from the {\em target distribution} $\Pi$. We will now derive a version of this chain to sample approximately from $\bar{\Pi}$

\subsection{Optimization using Metropolis Hastings With Auxiliary Variables (MHAV)}\label{sec_mhav}

In this section, we show how we may use the MH algorithm to sample from the distribution $\bar{\Pi}$ defined in Section \ref{sec_discOpt} and hence perform optimization. We call this algorithm MH with auxiliary variables (MHAV) because we  introduced the auxiliary variable $\lambda$ to enable us to perform global optimization. As we discussed above, MHAV may be thought of as simulated annealing without a temperature schedule. 

To adapt MH for our global optimization problem, we set $\msttspc \eqdef Y \times \Lambda$ and set our target to be $\Pi = \bar{\Pi}$.  We briefly note here that if we plug $\bar{\Pi}$ into (\ref{eq_accpt}), then the normalization term cancels out, and in our algorithms there will never be any need to compute $Z$. Let $\phi_Y$ be any irreducible kernel over $Y$ (this will depend on the nature of $Y$ and will be an input to the optimization algorithm -- we discuss this below). Define the following transition kernel over $\Lambda = \{\lambda_1,\lambda_2,\cdots,\lambda_n\}$, parametrized by $\alpha' \in (0,1)$:
\begin{equation}\nonumber
\phi_\Lambda (\lambda'|\lambda) \eqdef
\begin{cases}
\alpha'   & \mbox{ if } \lambda = \lambda_i, \lambda' = \lambda_{i+1} \mbox{ and } i < n  \\
1-\alpha' & \mbox{ if }  \lambda = \lambda_i, \lambda' = \lambda_{i-1} \mbox{ and } i > 1\\
1 & \mbox{ if } \lambda = \lambda_0, \lambda' = \lambda_{1} \mbox{ or } \lambda = \lambda_n, \lambda' = \lambda_{n-1}\\
0 & \mbox{ otherwise }
\end{cases}
\end{equation}
Since $\alpha' > 0$, it is easy to see that for any $\lambda_i,\lambda_j$, there is  a sequence $\lambda_i,\lambda_{i_1},\cdots,\lambda_{i_n}\lambda_j$ with positive probability under $\phi_\Lambda$. That is, 
\begin{lemma}\label{lemma_plmb_irred}
$\phi_\Lambda$ is irreducible.
\end{lemma}

Given the above, the proposal distribution $\bar{\phi}[ (\lambda',y')|(\lambda,y)]$ for $\psmh$ is defined using the parameters $\alpha,\beta  \in (0,1)$, $\alpha + \beta < 1$,  as follows. 
\begin{equation}\label{eq_barphi}
\barphi[\lambda',y'| \lambda,y]
\eqdef 
\begin{cases} 
\alpha\phi_{\Lambda}(\lambda'|\lambda) &\mbox { if } \lambda \neq \lambda', y = y'\\
\beta \phi_Y(y'|y) &\mbox { if } \lambda = \lambda', y \neq y'\\
(1-\alpha-\beta) + \beta\phi_Y(y'|y)  & \mbox{ otherwise }
\end{cases}
\end{equation}
The transition kernel $\psmh$ is now defined as in (\ref{eq_PMHDef}) using $\bar{\phi}$ as the proposal distribution  and (\ref{eq_lambdaPi}) as the target distribution:
\begin{equation}\label{eq_psmh}
\psmh(x'|x) = 
\begin{cases}
\barphi(x'|x)\baccpt_x(x'), &\mbox{ if } x \neq x' \\
1-\sum_{x'\neq x}\barphi(x'|x)\baccpt_x(x'), &\mbox{ otherwise }, 
\end{cases}
\end{equation}

%\begin{equation}\nonumber
%\phi_\Lambda (\lambda,\lambda') \eqdef
%\begin{cases}
%\alpha'   & \mbox{ if }\lambda = \lambda_i, \lambda' = \lambda_{i+1} \mbox{ or } \lambda = \lambda_N, \lambda' = \lambda_N\\
%1-\alpha' & \mbox{ if } \lambda = \lambda_i, \lambda' = \lambda_{i-1} \mbox{ or } \lambda = \lambda_0, \lambda' = \lambda_0\\
%0 &  \mbox{ otherwise }
%\end{cases}
%\end{equation}

 Given the above, the overall discrete global optimization algorithm MHAV (Metropolis-Hastings with Auxiliary Variable)  is listed in Algorithm \ref{alg_MHAV}. Note that lines 5-6 are sampling from the transition kernel $\psmh$.

\begin{algorithm}[th]
   \caption{MHAV$(\Lambda,Y,\bar{\Pi},\bar{\phi}, T_M)$}
   \label{alg_MHAV}
\begin{algorithmic}[1]
	\STATE {\bfseries Input:} The set of auxiliary variables $\Lambda$, the search space $Y$, the target distribution $\bar{\Pi}$, and proposal distribution $\bar{\phi}$, $T_M$ number of iterations to run algorithm. 
	\STATE {\bf Output:} An element $y \in Y$.
	\STATE {\bfseries Initialize:} Initial, $\lambda(0) = \lambda_0$, $y(0) =$ arbitrary element of $Y$. 
	\FOR{$t= 1$ to $T_M$}
		\STATE Sample $(\lambda',y') \sim \bar{\phi}[\cdot|\lambda(t),y(t)]$
		\STATE With probability $\bar{\accpt}_{\lambda(t),y(t)}[\lambda',y']$, set $\lambda(t+1) = \lambda'$, $y(t+1) = y'$, and with probability $1- \bar{\accpt}_{\lambda(t),y(t)}[\lambda',y']$, set $\lambda(t+1) = \lambda(t)$, $y(t+1) = y(t)$.
	\ENDFOR
	\RETURN $\argmin_{y(t)} f(y(t))$.
\end{algorithmic}
\end{algorithm}

\subsection{Analysis of the MHAV Algorithm}\label{sec_MHAVconv}

The reader may prefer to skip this and the next section and move directly to Section \ref{sec_srchClust}, which details how the MHAV algorithm is adapted to search for the optimal cluster. We begin analysis of our algorithm by showing that the kernel $\psmh$ for the $\bar{\phi}$ defined above is indeed irreducible and aperiodic.
\begin{lemma}\label{lemma_lpIrr}
If $\bar{\Pi}$ and $\phi_Y$ satisfy $\min_{x,x'}\frac{\bar{\Pi}(x')\phi_Y(x|x')}{\bar{\Pi}(x)\phi_Y(x'|x)} > b > 0$, the kernel $\psmh$ is irreducible and aperiodic.
\end{lemma}
The proof is given in Appendix \ref{app_proofs} -- additionally note that  by finiteness of $f$ and $\lambda_i$s, and irreducibility of $\phi_Y$, such a $b$ always exists. The following theorem establishes the probability with which we draw an element from the acceptable set $\hat{Y}$ when using $\psmh$ to sample.
\begin{theorem}\label{thm_drawOpt}
$\psmh$ has $\bar{\Pi}$ as its stationary distribution, and hence for any initial state $x_0$ of the chain $\psmh$, 
\begin{equation}
\lim_{n\rta \infty} \ltvnorm{\psmh^n(\cdot|x_0) - \bar{\Pi}(\cdot)} = 0
\end{equation}
In particular if at step $t$ $\ltvnorm{\psmh^t(\cdot|x_0) - \Pi(\cdot)} \leq k$, then $\psmh^{t'}(x  \in \hat{\lambda} \times  \hat{Y}|x_0) \geq \theta-k$ for all $t' > t$.
\end{theorem}
The proof is given in Appendix \ref{app_proofs}. Combining the above with (\ref{eq_sampleDist}) (and the discussion following the equation) shows that the probability that MHAV samples from the acceptable set goes to $1$ at rate $>1-\theta + k$ from step $t_k$ onwards, where $t_k$ is the step such that $t > t_k$ implies $\ltvnorm{\psmh^{t}(\cdot|x_0) - \Pi(\cdot)} \leq k$. 

We can also derive a convergence rate which establishes that for every $k$ such a $t_k$ exists and the rate of convergence of MHAV goes arbitrarily close to $1-\theta$. Define the {\em diameter} of $\msttspc$ given the Markov chain $\psmh$ to be 
\begin{equation}\label{eq_diameter}
D \eqdef \min \{l | \forall x,x',  \psmh^l(x'|x) > 0 \}
\end{equation}
Now define the ratio 
\begin{equation}\label{eq_delratio}
\delta \eqdef \min_{x,x'} \frac{\psmh^D(x'|x)}{\bar{\Pi}(x')} 
\end{equation}
$D$ is finite and $\delta$ non-zero by the irreducibility of $\psmh$ and finiteness of $f$. We have the following: 
\begin{theorem}\label{thm_SCconv}
The convergence rate of $\psmh$ to $\bar{\Pi}$ satisfies:
\begin{equation}\nonumber
\ltvnorm{\psmh^n(\cdot|x_0) - \bar{\Pi}(\cdot)} \leq (1-\delta)^{n/ D}
\end{equation}
for any initial state $x_0$.
\end{theorem}
This derives directly from the proof of Theorem 4.9  \citep{levin_peres_wilmer:2009} and is given in Appendix \ref{app_proofs}. This implies that for each $k$, we have $t_k$ $=$ $D\ln{k/(1-\delta)}$.

\newcommand{\bfx}{\mathbf{x}}

\subsection{Setting the Optimization Parameters}\label{sec_optparam}

%[discuss simualted annleaing here ?]

We now discuss heuristics to set the parameters $\alpha'$ in $\phi_{\Lambda}$ and $\alpha,\beta$ in $\psmh$ so as to optimize the convergence rate derived above, by minimizing $\delta$ (defined in (\ref{eq_delratio})). In setting these parameters, we are given the proposal distribution $\phi_Y$, which was required to be an irreducible kernel on $Y$, and the target distribution $\bar{\Pi}$ over $\Lambda \times Y$. We start with the following result which simplifies deriving our result
\begin{lemma}\label{lemma_Dindep}
$D$ (defined in (\ref{eq_diameter})) is independent of $\alpha'$, $\alpha,\beta$. 
\end{lemma}
The proof is given in Appendix \ref{app_proofs}.

\begin{corollary}
Given $f,\phi_Y$, the set of paths of positive probability under $\psmh$ is invariant with respect to $\alpha',\alpha,\beta$.
\end{corollary}
\begin{proof}
Follows directly from the proof of Lemma \ref{lemma_Dindep}.
\end{proof}
So, we need to set $\alpha',\alpha,\beta$ to maximize $\delta$. However,  $\delta$ also depends on $f$ and $\phi_Y$, both of which are unknown and so it is difficult to specify optimal values for these a-priori. However, we can give heuristic arguments for setting these parameters in terms of increasing the `flow' of the search process through the search space $Y$. First, $\alpha'$ is used to choose whether we should increase or decrease the $\lambda$ value. We set $\alpha'$ to $1/2$ to ensure a neutral value and that we do not favor either direction and ensure maximum flow through the search space. 

Now note that at each step the chain $\psmh$ moves either in $\Lambda$ space or $Y$ space. $\alpha$ and $\beta$ determine, respectively, how often we move in the $\Lambda$ and how often in $Y$. To make the search more effective (based on analysis of simulated annealing type algorithms), it seems we need to make sure that initially we need to explore the $Y$ quite a bit and only settle down after we have explored sufficiently, by increasing the $\lambda_i$ value. Hence, our recommendation is to set the $\alpha$ to be significantly smaller than $\beta$, ideally the ratio $\alpha/\beta$ should reflect how difficult we expect it to be to get close to the best $y^*$ (with smaller ratio for greater difficulty).  Even though our parameter settings are heuristic, we again stress that this only affects the convergence speed, but not the ultimate convergence. This is in contrast to simulated annealing where convergence itself is guaranteed only if we set the parameters carefully.

\subsection{Searching for the Optimal Cluster}\label{sec_srchClust}

Searching for the optimal cluster can now be be solved using the MHAV algorithm. The algorithm for searching through the space of clusterings is given in Algorithm \ref{alg_srchClust}. In this case, $Y = \clstngs$, and the objective function is $f(\bfA) = cost(\bfA)$. To complete the specification of MHAV for our problem, we define the distribution $\phi^M_Y(\bfA'|\bfA)$ to be the probability with which a randomized procedure converts $\bfA$ to $\bfA'$. The randomized procedure is as follows. 

Given $\bfA =  \{A_1,\cdots, A_n\}$, choose $A_i$ uniformly at random, and $A_j$ uniformly at random from $(\bfA- \{A_i\}) \cup \{B\}$, where $B$ is place-holder/empty set representing a new cluster. Now choose $k_i$ points from $A_i$, uniformly at random (without replacement), and put them in $A_j$. The clustering resulting from this transfer is $\bfA'$.

Note that if $A_j = B$, then $\bfA'$ has one more cluster than $\bfA$. Additionally, if $k_i = |A_i|$, and $A_j \neq B$, then $\bfA'$ has one less cluster than $\bfA$. Otherwise, they both have the same number of clusters, but differing at $A_i$ and $A_j$. The number $k_i$ is chosen using the exponential distribution over $\{1,2,\cdots, |A_i|\}$: 
\begin{equation}\nonumber
PE(k;\theta_1) \eqdef e^{-k\theta_1}(e^{\theta_1} - 1)/(1-e^{|A_i|\theta_1}), \mbox{ where } \theta_1 > 0
\end{equation}
This ensures that $k_i$ is small with higher probability and so we are less likely that $\bfA$ and $\bfA'$ are very different due to moving large number of points from $A_i$ to $A_j$. The parameter $\theta_1 $ is user dependent and in our experiments we set it to $1$.

The following Lemma shows that $\phi^M_Y(\bfA'|\bfA)$ is irreducible and hence satisfies the condition in Lemma \ref{lemma_lpIrr} and hence ensures the convergence results in Section \ref{sec_MHAVconv}.
\begin{lemma}\label{lemma_aperirr}
$\phi^M_Y$ defined above is irreducible. 
\end{lemma}
The proof is in Appendix \ref{app_proofs}.

\begin{algorithm}[th]
   \caption{Search-Clusterings$(\mdps,d,\Lambda,T_M)$}
   \label{alg_srchClust}
\begin{algorithmic}[1]
	\STATE {\bfseries Input:} A set of MDPs $\mdps = \{ \mdp_1,\mdp_2,\cdots, \mdp_N\}$, the set of auxiliary variables $\Lambda$, a cost function $cost$, input condition $term$. 
	\STATE {\bfseries Initialize:} $\phi^M_Y$ defined with respect to $\mdps$; $\bar{\phi}$ defined using $\phi^M_Y$ using (\ref{eq_barphi}); define $\bar{\Pi}(\lambda,\mdp) = \lambda^{-cost(\mdp)}$.
	\RETURN MHAV$(\Lambda,\mdps, \bar{\Pi},\bar{\phi}, T_M)$
\end{algorithmic}
\end{algorithm}

%\begin{equation}\nonumber
%PE(k;\theta_1)\eqdef \frac{1-\exp(-\theta_1)}{\exp(-\theta_1)} \sum_{m=1}^\infty  \exp[-\theta_1((m-1)|A_i| + k)]
%\end{equation}

\begin{algorithm}[th]
   \caption{Continual-Transfer$(d,\Lambda,cost,T_M,l,\Delta R,\beta,T)$}
   \label{alg_contTrans}
\begin{algorithmic}[1]
	\STATE {\bfseries Input:} A metric $d$, which is either a $d_M$ or $d_V$; cost function $cost$, which is either $cost_1$ or $cost_2$; Search-Clustering parameters $T_M,l,\Delta R$,  EXP-3-Transfer parameters $\beta$, $T$.
	\STATE {\bfseries Initialize:}	Initial clustering $\bfA = \emptyset$, collection of previous MDPs $\mdps$.
	\FOR{$h = 1$ to $\infty$}
		\STATE Get unknown MDP $\mdp_h$ from the environment and run EXP-3-Transfer$(\mdp_h,sourcePol(\bfA,d,cost),\beta,T_M,l)$.
		\STATE Set $\mdps \lta \mdps \cup \{\mdp_{h}\}$ 
		\STATE {\bf if} $h \mod J = 0$ {\bf then } $\bfA =$ Search-Clusterings$(\mdps,\Lambda,cost,T_M)$.
	\ENDFOR
\end{algorithmic}
\end{algorithm}

\section{The Continual Transfer Algorithm}\label{sec_contTrans}

In this brief section we combine all the algorithms presented so far into the full continual transfer algorithm, which is listed as  Algorithm \ref{alg_contTrans}. The algorithm runs in phases and in each phase it solves a MDP using the EXP-3-Transfer algorithm and the current set of source policies as input. In line 4, the function $sourcePol(\bfA,d,cost)$ generates the $c$ source policies $\rho_1,\rho_2,\cdots,\rho_c$ from clustering $\bfA$ such that $\rho_j$ is the optimal policy for  $\mdp^j$ where $\mdp^j$ is chosen from $A_j$ according to (\ref{eq_landmark_dV}). If the current phase $h$ satisfies $h \mod J = 0$, then it runs the Search-Clustering algorithm to find a new set of source tasks from the $h$ tasks solved so far.

\section{Experiments}\label{sec_exp}

We performed three sets of experiments to illustrate various aspects and the efficacy of our algorithm\footnotemark. The baseline algorithm for comparison in the first domain was a multi-task hierarchical Bayesian reinforcement learning algorithm proposed by \cite{wilson_fern_ray_tadepalli:2007}, and hence compares our approach to an alternative approach to clustering tasks. In the latter two domains we compare against Probabilistic Policy Reuse (PPR) as introduced by \cite{fernandez_veloso:2006} which, as mentioned in Section \ref{sec_relWork}, is the main prior work on policy reuse algorithms.

In the larger second and third experiments, we report results for our proposed EXP-3-Transfer algorithm, PPR, and standard Q-learning. Additionally, we experiment using different forms of clustering, including no clustering, our Search-Clustering algorithm, clusters chosen manually by hand, and a greedy clustering procedure. The greedy clustering algorithm is to choose a threshold for the $d_V$ distance, then construct clusters by selecting an MDP arbitrarily to seed a cluster, and finally adding all the MDPs with distance less than that threshold to that cluster. In our graphs, we present the results for the best/lowest cost clustering found by using various threshold values.

A summary of the algorithm combinations used are shown in Table \ref{table_expStruct}.

%Given this, we report results of the various combinations of learning algorithms and clustering approaches described in Table \ref{table_expStruct}. 

\footnotetext{The code used in both of the experiments, as well as all of the generated data, can be found here: {\tt http://wcms.inf.ed.ac.uk/ipab/autonomy/code/MDP\_Clustering\_code.zip}}

\begin{table}[!h]
\begin{center}
\begin{tabular}{ | l | c | c | c | c |}
\hline
{\bf Algorithm} & {\sc Full} & {\sc Sans } & {\sc Hand-picked} & {\sc Greedy} \\
\hline
{\sc Exp-3-Transfer } & $\checkmark$ & $\checkmark$ & $\checkmark$ & $\checkmark$ \\
{\sc Probabilistic Policy Reuse } & $\checkmark$ & $\checkmark$ & $\checkmark$ & $\checkmark$ \\
{\sc Q-learning  } & N/A & N/A & N/A & N/A \\
\hline
\end{tabular}
\caption{{\sc Experiment setup matrix}. This table shows the combinations of the algorithms and clustering methods used in the experiments. `Full' refers to the use of Search-Clustering, `sans' means without any kind of clustering, `hand-picked'  means using a set of source policies that we selected believing to be optimal, and `greedy' refers to the heuristic threshold-based method (see text for details).}\label{table_expStruct}
\end{center}
\end{table}

For each graph presented in experiments 2 and 3, the results are averaged over 10 different target tasks with 10 trials per task. The various parameters used for the clustering and transfer algorithms are given in Table \ref{table_clustParam}.

\begin{table}[!h]
\begin{center}
\begin{tabular}{ | c | c | c | c | c | c | c | c |}
\hline
\multicolumn{2}{|c|}{ {\sc Exp-3-Transfer}} & \multicolumn{5}{|c|}{\sc Search-Clusterings} & {\sc RL}  \\
\hline
 $T$ & $\delta$ &  $\alpha$ & $\beta$ & $\alpha'$  & $\theta$ & $T_M$                                    & $\gamma$ \\
\hline
 $1000,5000$ and $10^4$ & $0.1$  & $0.1$      & $0.8$   & $0.5$        & $1$        & $10^5$  (for each of $20$ random restarts) & $0.9$ \\
\hline
\end{tabular}
\caption{{\sc Values of algorithm parameters}. This table gives the values for the different parameters used in our various algorithms. RL refers to common parameters for reinforcement learning algorithms. The $T$ parameter was chosen to illustrate effect of this parameter on algorithm performance. The $\delta$ parameter was chosen to allow for a high degree of confidence in the performance of EXP-3-Transfer. The $\alpha,\alpha'$ and $\beta$ parameters where chosen according to Section \ref{sec_optparam}, and $\theta$ was chosen heuristically. The value of $T_M$ was selected because we found that restarting gave better results. $\gamma$ was chosen as appropriate for these domains.} \label{table_clustParam}
\end{center}
\end{table}

We present results from experiments run on three different domains. In Section \ref{sec_colGrid} we use the coloured grid world domain, commonly used in Bayesian reinforcement learning, to compare the clusters found by our proposed algorithm with those of a hierarchical Bayesian method, in order to motivate our clustering approach. In Section \ref{sec_windCorr} we present results from a simple windy corridor domain to demonstrate the clusters produced by our Search-Clustering algorithm. The results show that the clusters found are intuitive in nature. In Section \ref{sec_survDom}, we present results on the more complex surveillance domain (described briefly in the Section \ref{sec_intro}) which is a variant of the kinds of problems that are considered, for instance, by \cite{an_kempe_kiekintveld_shieh_singh_tambe_vorobeychik:2012}. In this experiment we show the performance for the algorithm combinations given in Table \ref{table_expStruct} under various numbers of previous tasks and task lengths.

%In the following sections, we only use the cost function $cost_2$. While $cost_1$ is better motivated as it is derived from much weaker conditions than $cost_2$, it is nonetheless too weak in the sense that the distance function $d_M$ used therein highlights differences that may be irrelevant. For instance, if there are two MDPs with identical optimal policies and  transition distributions, but a state where the reward difference is $1000$, the $d_M$ distance between the two MDPs is $k_M(1000)$ (defined in Section \ref{sec_mdpN1arb}) while the $d_V$ distance between the two MDPs is $0$. Indeed the $d_M$ distance function is similar to bisimulation based distance functions studied by \cite{castro_precup:2010} which, while theoretically well motivated,  were also found to be inadequate for applications. {\color{red}Can we show something with $cost_1$?}

\subsection{Coloured Grid World Domain}\label{sec_colGrid}

\cite{wilson_fern_ray_tadepalli:2007} introduced a method for modelling the unknown distribution of tasks in a multi-task reinforcement learning problem using a hierarchical infinite mixture model. Their two-layer model is capable of representing a previously unknown number of classes of MDPs, as well as finding the latent distribution of parameters in each class. For any new task, the model acts as a prior over the parameter space, which enables the generation of a faster initial solution, to be subsequently optimised with additional task-specific learning.

This model describes latent structure in the parameter space of the observed tasks. Instead, our proposed clustering algorithm seeks to identify similarities in policies across different tasks.

To compare this infinite mixture model approach with ours, we consider a simplified version of the coloured maze domain of~\cite{wilson_fern_ray_tadepalli:2007}, where each cell of a grid world is coloured with one of two possible colours. Each colour is assigned a weight $w\in[0,1]$, the values of which differ between tasks. The agent navigates the grid world from one corner to the diagonally opposite corner. This is done using four actions, each of which moves the agent to an adjacent cell in the corresponding cardinal direction, deterministically. After every movement, the agent receives a reward equal to the sum of the colour weights of the current cell and the adjacent four cells. The goal of the agent is to maximise the received rewards. Each task is completely described by two parameters (the colour weights) in the unit square. This is intended to simplify visualisation of the resulting clusterings.

We sample 50 tasks uniformly at random from the unit square and cluster them using both the hierarchical infinite mixture model and our framework. Figure~\ref{fig:col_wilson_clus} shows the clustering results of~\cite{wilson_fern_ray_tadepalli:2007}, and Figure~\ref{fig:col_our_clus} shows the clusters obtained by our method. Note that all the tasks that lie on the same line that goes through the origin are equivalent as they have the same colour weight ratio, and thus have the same optimal policy with scaled values. As a result, our method is able to model this fact and summarise the complete set of MDPs into a reduced set of landmark MDPs with equivalent policies. On the other hand, modelling the similarities between tasks based on task parameters will fail to realise this equivalence, and subsequently result in many clusters with similar policies because they share a local neighbourhood in parameter space.

\begin{figure}[h!]
 \centering 
\subfloat[Clustering of~\cite{wilson_fern_ray_tadepalli:2007}]{
    \includegraphics[width=3in]{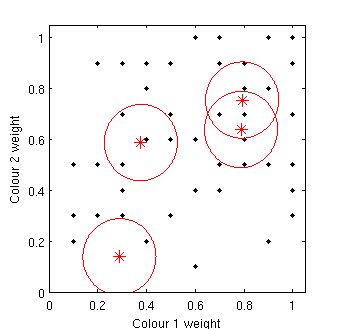}
    \label{fig:col_wilson_clus}
}
\subfloat[Clustering of our proposed framework]{
      \includegraphics[width=3in]{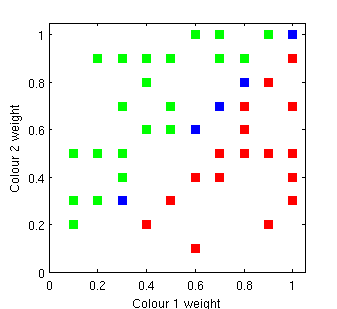}
\label{fig:col_our_clus}
}

\caption{The tasks in parameter space (the weights of the two colours) of the coloured maze domain. Each dot corresponds to a randomly sampled task. (a) Clusters obtained using a hierarchical infinite mixture model. The asterisks are the means of recovered clusters and the circles show three standard deviations. (b) Clusters obtained using our clustering algorithm. Each colour represents a different class. Note how our method captures the similarity in policy much more closely than methods that cluster directly in parameter space.}

\end{figure}

\subsection{Windy Corridor Domain}\label{sec_windCorr}

The windy corridor domain is illustrated in Figure \ref{fig_windDom}. The domain consists of a row of $10$ parallel corridors with a `wind' blowing from the South to the North along two columns of cells near the entrance to the corridors.

\begin{figure}[!h]
\begin{center}
\includegraphics[width = 4.2in]{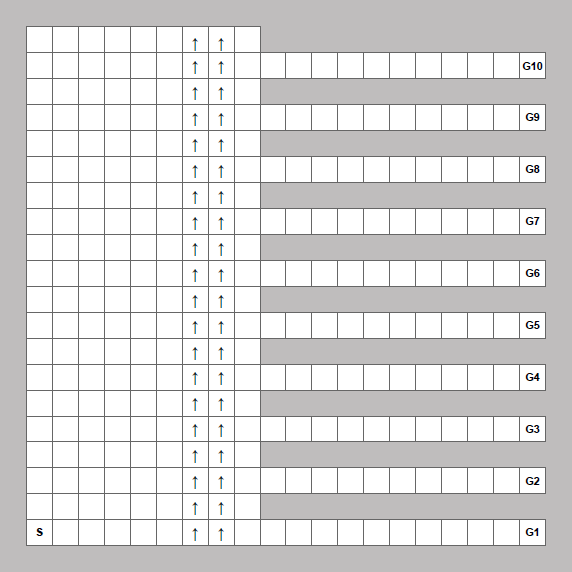}
\end{center}
\caption{{\sc The windy corridor domain.} This shows the $10$ corridors, the location of the goal states (G1-G10) and the direction of the wind (the small arrows). The start state is marked by {\bf s}. }\label{fig_windDom}
\end{figure}

The agent has one action for each possible cardinal direction which moves it in that direction deterministically. In a windy cell, the motion of the agent becomes probabilistic with the probability of moving North being $p$, and moving in the desired direction being $1-p$. $p$ depends on the probability of wind in that cell, which is a task parameter ranging from $0$ to $0.9$.

The MDPs in the domain are distinguished by two values: the location of the goal state and the probability of wind. There are $10$ possible wind probabilities, which together with the $10$ possible goal locations, results in a total of $100$ possible MDPs.

For this domain, we learned the optimal value function for each of the MDPs, and from that computed the distance between every pair of MDPs. This was then used to cluster the MDPs using the Search-Clusterings algorithm. Figure \ref{fig_windClust} presents the final clusters we found for this domain. This figure shows that the best clustering found by the algorithm placed tasks with the same goal state in the same cluster. This follows intuition because, despite the wind probability, the policies required for MDPs with identical goal states will be identical, yet be different for tasks with different goal states. This demonstrates that the Search-Clusterings algorithm is capable of recovering the expected clusters, thereby providing a sanity check for our algorithm, although we note that the results are stochastic and vary between runs.

\begin{figure}[!h]
\begin{center}
\includegraphics[width=5.5in]{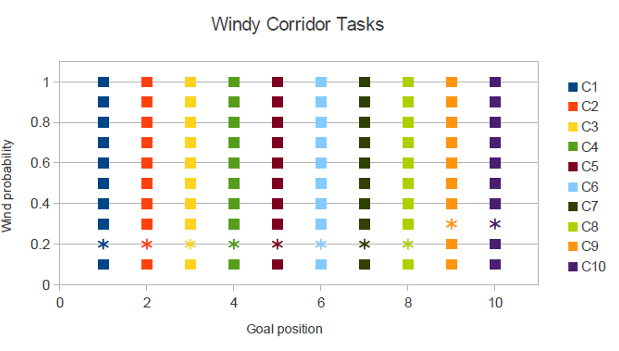}
\end{center}
\caption{{\sc Clustering for the windy corridor domain}. The clustering was obtained by running Search-Clustering on the full set of $100$ MDPs for this domain. Each point in the 2D grid is a MDP with the goal location and wind probability given by the $x$ and $y$ axis respectively. The colours indicate the cluster to which the MDP was found to belong, and shows that all the MDPs with the same goal state are assigned to the same cluster. Additionally, the cluster centres $\mdp^i$ are each marked with a $*$.}\label{fig_windClust}
\end{figure}

In order to illustrate the effects of the clustering, we incrementally built up the full set of MDPs, by presenting them in a random order to the Search-Clustering algorithm, and having it cluster them after the addition of each new MDP. These results are shown in Figure \ref{fig_incrementalclustering}. Note that the allocation of MDPs to clusters remains largely consistent across the presentation of $100$ MDPs. In this case, the algorithm recovers $12$ main clusters.

\begin{figure}[!h]
\begin{center}
\includegraphics[width=6in]{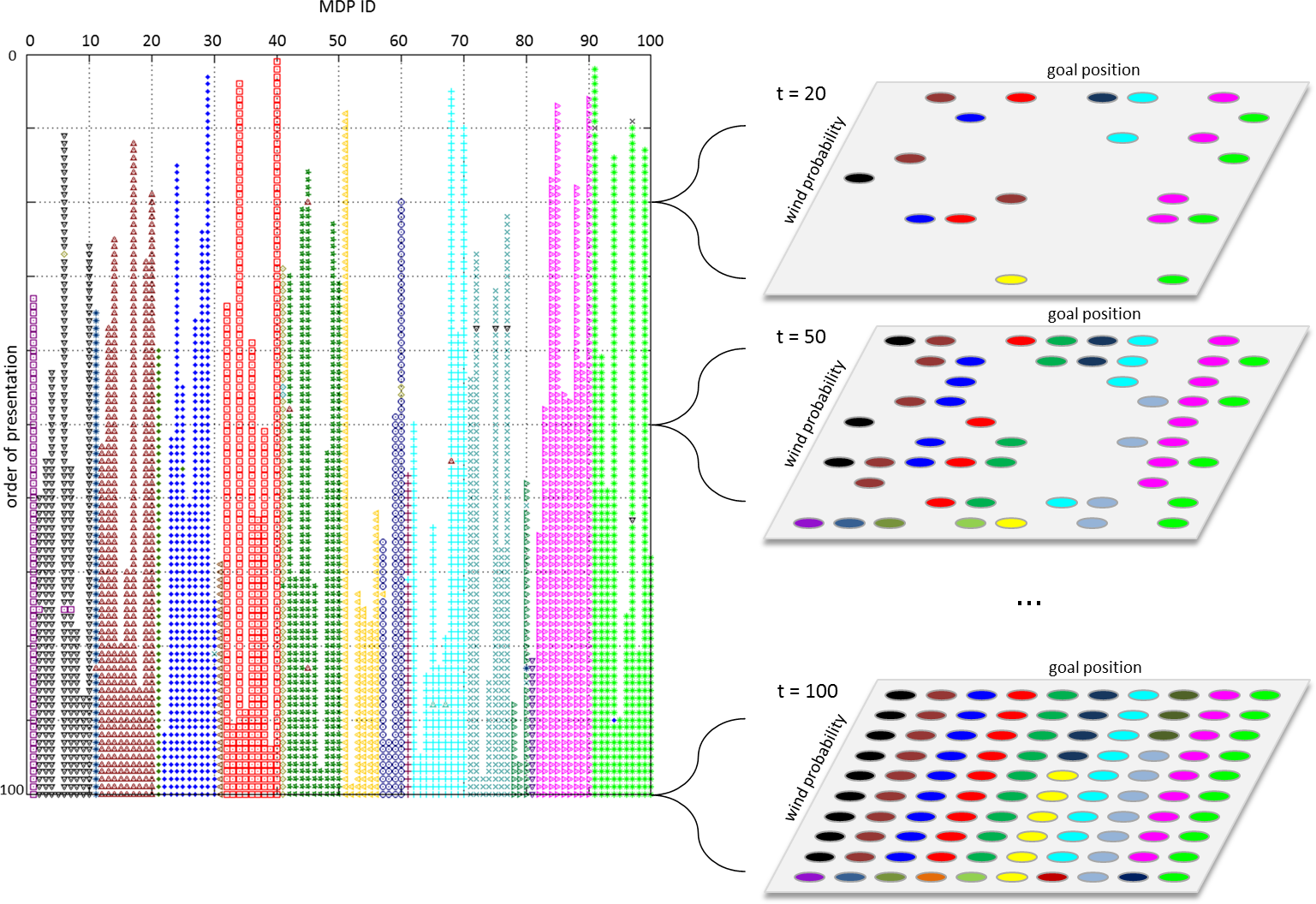}%{inc_clus_markers.png}
\end{center}
\caption{{\sc Incremental clustering for the windy corridor domain}. The incremental clustering on the left was obtained by running Search-Clustering on an increasing number of MDPs. The $100$ MDPs (represented by a value on the $x$ axis) in this domain are presented in a random order to the algorithm, with the $y$ axis showing time of first presentation to the algorithm. All points of the same colour have been assigned to the same cluster. On the right are three zoomed-in time slices from the incremental clustering (at $t=20$, $t=50$ and $t=100$ respectively). Each time slice shows the intuitive interpretation of MDPs in the representation of Figure \ref{fig_windClust}.} \label{fig_incrementalclustering}
\end{figure}

\subsection{Surveillance Domain}\label{sec_survDom}

The surveillance domain is illustrated in Figure \ref{fig_survDom}. In this domain, the goal of the agent is to catch infiltrators who wish to break into a target region. There are $L$ different vulnerable locations (abbreviated {\em v-locations}) in the domain, and the infiltrators only choose a subset of those v-locations to infiltrate -- we call these {\em target v-locations}. The type of the infiltrators is defined by the sequence in which they visit the target v-locations and the goal of the agent is to find out where the target locations are and surveil them in the right sequence to find the infiltrators. 

\begin{figure}[!h]
\begin{center}
\includegraphics[width=3in]{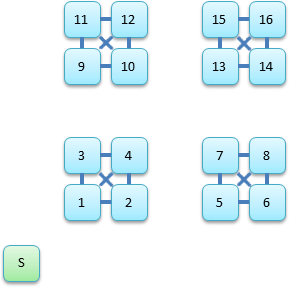}
\end{center}
\caption{{\sc The surveillance domain.} This caricature of the surveillance domain shows $16$ surveillance locations (v-locations) marked in blue. The start location is marked $S$. The full domain used in our experiments is a $48\times 48$ gridworld with $64$ v-locations. Each MDP in the domain requires the agent to surveil $i$ different locations, $i \in \{1,2,3,4\}$ in a particular sequence to receive a positive reward of $200$ for each location surveilled. Surveilling at a wrong location results a negative reward of $-10$ (the infiltrators have escaped). Each action taken gives a reward of $-1$. The target v-locations are clustered spatially into groups of $4$ (as shown by the connections in the figure), such that surveilling one location in the cluster instead of the other results in a reward of $190$ (a penalty of $200 - 190$ = 10) but does not end the episode (see Figure \ref{fig_survDom_traj} for further details).}\label{fig_survDom}
\end{figure}

The actions available to the agent are the motion actions in the cardinal directions as well as a surveillance action, each of which is deterministic in its outcome. Every action taken results in a reward of $-1$, an unsuccessful surveillance action (i.e. inspecting the target v-location in the wrong order) results in a reward of $-10$,  while a successful surveillance action (i.e. inspecting the correct v-location at the right time) results in a reward of $200$. If instead of surveilling the correct v-location, the agent surveils a location {\em adjacent} to it, then it receives a reward of $190$. Each v-location has 3 other v-locations that are adjacent (see Figure \ref{fig_survDom}). Hence two MDPs $\mdp$, $\mdp'$, each corresponding to different v-location sequences $(v_1, v_2, v_3)$ and $(v'_1, v'_2, v'_3)$, are similar in terms of their optimal policy/$d_V$ distance, and belong to the same cluster, if each pair of v-locations $v_i$ and $v'_i$ are adjacent -- because in this case, the optimal policy of $\mdp$ will yield near-optimal sequence of rewards of $190$s when applied in $\mdp'$ and vice versa.

\begin{figure}[!h]
\begin{center}
\includegraphics[scale=0.7]{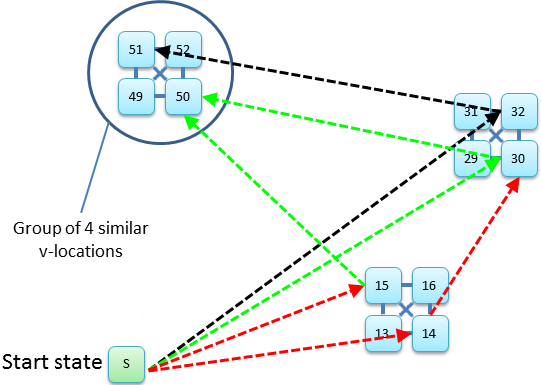}
\end{center}
\caption{{\sc Trajectory Examples.} Example of four different types of trajectories for the surveillance domain, showing three groups of four v-locations. The task MDP requires visiting two v-locations, identified as numbers $30$ and $50$ in sequence. A green line represents moving to the correct v-location in sequence, yielding a reward of $200$. A black line represents moving to an incorrect v-location, which is in the same group as the correct one, for a reward of $190$. A red line depicts movement to the incorrect group, for a reward of $-10$ when the surveillance action is taken. %The optimal trajectory for this MDP is shown in black/solid line. A trajectory that obtains a reward of $190$ per target v-location is shown in green/dotted line (it visits two v-locations in the same block in the right sequence). The red/dashed trajectory results in rewards of $-10,-10$ when the surveil action is taken at the  two v-location.
}\label{fig_survDom_traj}
\end{figure}

We present the following results for experiments run with a combination of different numbers of previous MDPs (referring to the surveillance tasks which have been encountered before, and possibly subsequently clustered) and numbers of target locations. The results show that the more complex the transfer task is, the better EXP-3-Transfer with clustering performs compared to Probabilistic Policy Reuse, where complexity is measured in terms of the number of previous MDPs and the difficulty of the target task.

\begin{itemize}

\item We compare the performance of EXP-3-Transfer, Probabilistic Policy Reuse and Q-learning as the complexity of the transfer problem increases. Here, the complexity of the transfer problem is both the number of previous tasks, and the complexity of the MDP itself (i.e. the number of target v-locations). These results are referred to as clustering gains.

\item We compare the effect of different types of clusterings (in Table \ref{table_expStruct}) for EXP-3-Transfer with $T = 10,000$. These results are called clustering comparisons.

\item We compare the effect of having different $T \in \{ 1000,5000,10000\}$ for EXP-3-Transfer with clustering for various number of previous tasks. These results are the time comparisons.

\end{itemize}

\subsubsection{Clustering Gains}\label{sec_expClustGain}

We first study the effects of clustering, by comparing the performance of both EXP-3-Transfer and Probabilistic Policy Reuse with and without clustering. The results presented in this section summarise those of the complete experiment set, the results of which are provided in Appendix \ref{app_ful_alg_res}.

Figure \ref{fig_complexity_all} demonstrates the performance of the two algorithms on two different task variants in the surveillance domain: having either a sequence of two or three v-locations. The figure measures the \emph{clustering gain}, being the difference in performance with and without clustering. This figure does not show a comparison with Q-learning. In the full results given in Appendix \ref{app_ful_alg_res}, we show that EXP-3 consistently outperforms Q-learning by a large margin, indicating that our algorithm escapes negative transfer.

\begin{figure}[!h]
      \subfloat[{\sc Cumulative clustering gains.}] { \label{fig_complexity_cum}
    \includegraphics[width=0.5\textwidth ]{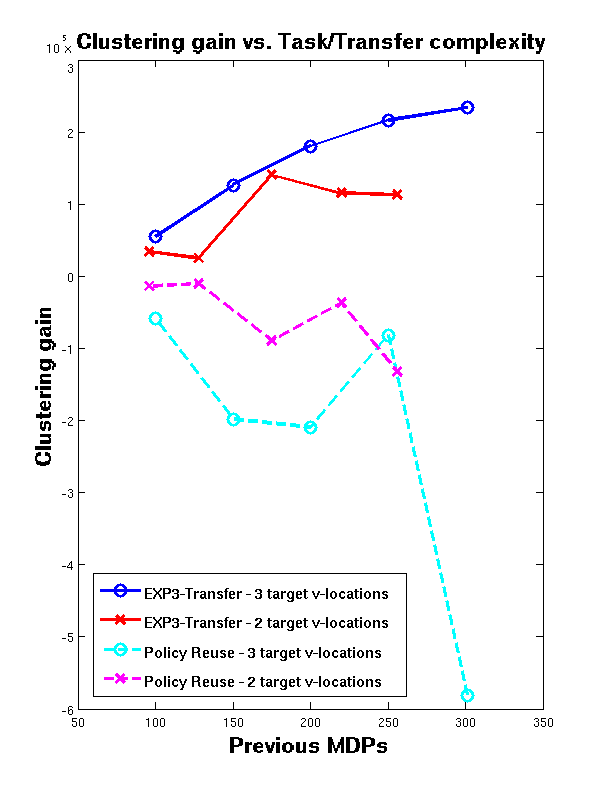} }
          \subfloat[{\sc Final clustering gains.} ] { \label{fig_complexity}
    \includegraphics[width=0.5\textwidth]{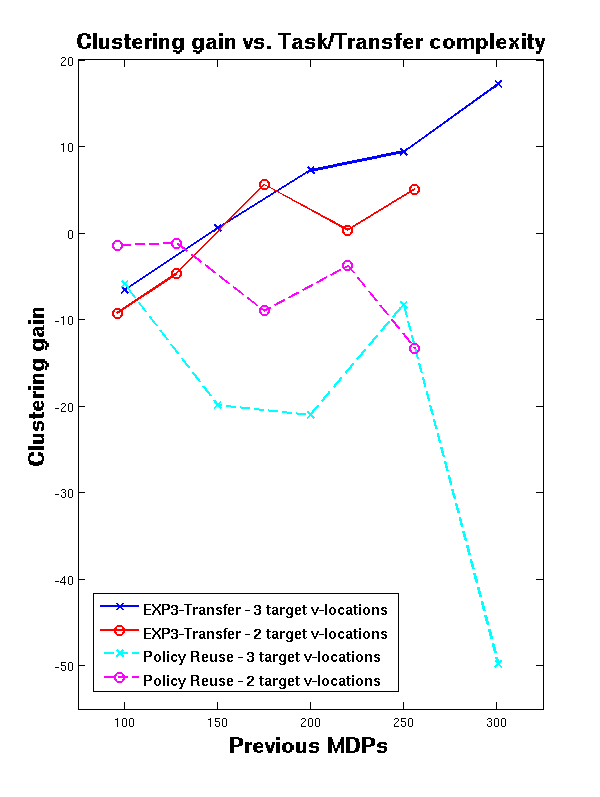}}
\caption{{\sc Clustering Gains.} The above figures show the {\bf clustering gain} for EXP-3-Transfer and Probabilistic Policy Reuse. For each $(x,y)$ data-point in each curve, the $y$-value is the difference in performance with and without clustering when there are $x$ previous MDPs. The performance measured in cumulative clustering gains (Figure \ref{fig_complexity_cum}) is the total cumulative discounted reward over $10,000$ episodes. The performance measure in final clustering gains (Figure \ref{fig_complexity}) is the discounted reward in the final episode. We again note that each $(x,y)$ point is averaged over $10$ different target tasks with $10$ trials per task.}
 \label{fig_complexity_all}
\end{figure}

As can be seen, EXP-3-Transfer always benefits from using clustering. Furthermore, the more complex the task, the better the performance. This is observed in the general upwards trend in the curves with an increasing number of previous MDPs, and the fact that the curve for the 3 v-locations lies above the curve for the 2 v-locations. This result is in complete agreement with our expectations, that in a bandit-like algorithm lowering the number of arms will result in lower regret. In addition, this result indicates that our clustering algorithm retains the correct arms so that with the removal of arms, the performance of EXP-3-Transfer is not affected adversely. 

%of our experimental results with two summary graphs which show the benefit of clustering vs. not clustering for our algorithm EXP-3-Transfer and Policy-Reuse. These results are presented in Figures \ref{fig_complexity_all} and \ref{fig_complexity_raw} respectively. The full results that these graphs summarize are given in Appendix \ref{app_ful_alg_res}.  Figures \ref{fig_complexity_cum} and \ref{fig_complexity} both show that EXP-3-Transfer always benefits from using clustering and in fact the more complex the task is, the better the performance is. This is observed in the general upwards trend in the curves with increasing number of previous MDPs and the fact that the curve for the 3 v-locations lies above the curve for the 2 v-locations. This result is in complete agreement with our expectations, that in a bandit like algorithm lowering the number of arms will result in lower regret. In addition, it also shows that our clustering algorithm retains the correct arms so that with the removal of arms, the performance of EXP-3-Transfer is not affected adversely. The above figure does not give comparison with Q-learning. In the full results given in the Appendix \ref{app_ful_alg_res}, we also show that EXP-3 consistently outperforms Q-learning by a large margin which shows that our algorithm escapes negative transfer.

Interestingly, for Probabilistic Policy Reuse the trend is reversed. It appears that clustering does not help this algorithm, and the clustering additionally becomes more detrimental as the task complexity increases. Our conjecture regarding the reason for this is that Probabilistic Policy Reuse uses the source policies not as potential optimal policies, but rather as exploration devices. By clustering the MDPs, we remove arms and hence reduce the number of exploration policies, which lowers the scope for exploration. This in turn results in negative performance gain for Probabilistic Policy Reuse.

It is also interesting to note that in Figure \ref{fig_complexity}, which shows the gain in terms of the final reward obtained, the initial gain for EXP-3-Transfer and PPR are both negative, and the gain for PPR is higher. However, as the transfer complexity increases (both in terms of previous MDPs and task complexity) the cumulative reward gain becomes positive for EXP-3-Transfer, while for PPR it continues to decrease. 

Given that the above figures show that PPR does not benefit from clustering, we compare the cumulative reward obtained by EXP-3-Transfer with clustering and PPR without clustering for the complex $3$-target-v-locations problem in Figure \ref{fig_complexity_raw}. This result shows that EXP-3-Transfer completely dominates PPR, with the difference becoming particularly stark when the number of previous tasks increases to $300$. 

\begin{figure}[!h]
\begin{center}
\includegraphics[width=0.9\textwidth]{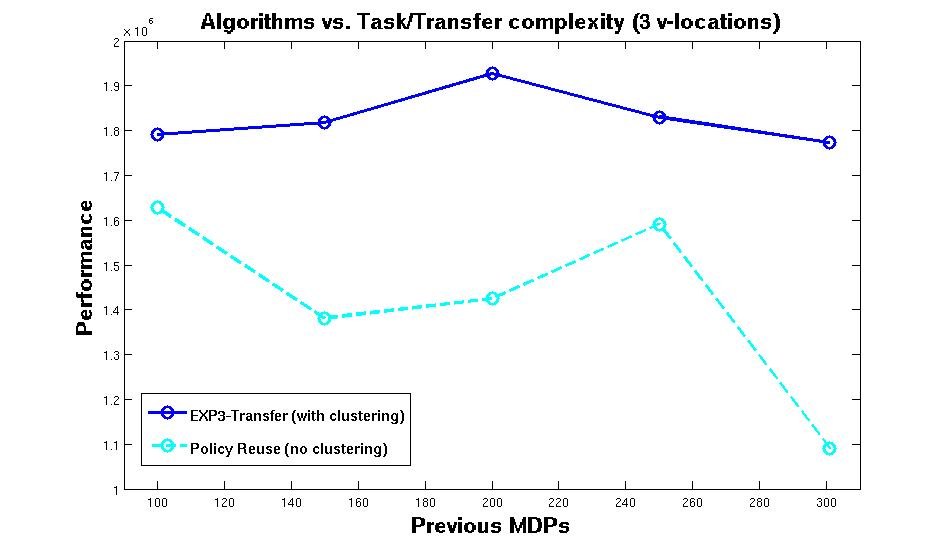}
\end{center}
\caption{{\sc Cumulative reward summary.} This figure shows the final cumulative rewards after $10,000$ episodes for EXP-3-Transfer with clustering and Probabilistic Policy Reuse without clustering for the surveillance domain with $3$-target-v-locations. The $x$-axis shows the number of previous MDPs.}\label{fig_complexity_raw}
\end{figure}

\subsubsection{Clustering Comparison}\label{sec_expClustComp}

We now compare the performance of EXP-3-Transfer when using the different types of clustering methods reported in Table \ref{table_expStruct}. As in the previous section, we examine the change in performance with increasing complexity of the transfer tasks. The summary of these results is given in Figure \ref{fig_clustComp_nostar}, and again, the full results are provided in Appendix \ref{app_ful_alg_res}.

\begin{figure}[!h]
\begin{center}
\includegraphics[width=0.93\textwidth]{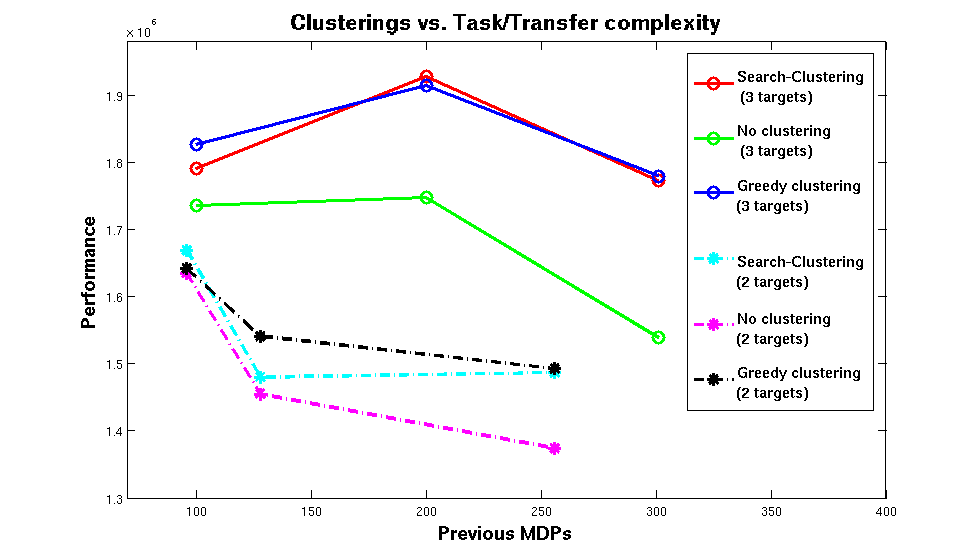}
\end{center}
\caption{{\sc Effect of clustering methods.} This figure compares the performance of EXP-3-Transfer run with no clustering, the greedy clustering procedure and Search-Clustering. The performance is measured in terms of the total cumulative discounted rewards over $10,000$ episodes. }\label{fig_clustComp_nostar}
\end{figure}

As can be seen, EXP-3-Transfer using Search-Clustering to obtain the source policies outperforms the case when we do not cluster the previous tasks. This confirms the result reported in the previous section. However, in addition we also observe that transfer after using the greedy clustering scheme performs about as well as using Search-Clustering for the 100 and 260 previous MDPs, but for the 125 previous MDPs, Search-Clusterings is significantly better. Search-Clustering and Greedy Clustering are comparable for those two numbers of previous MDPs largely due to the structure of the domain, where every element of each group of tasks is similar to every other task in the same group. That is, when the agent surveils any v-location in the same group as the true v-location required by the task it receives near-optimal rewards (see Figure \ref{fig_survDom_traj}). 

However, greedy clustering would fail on a more complex variant of the surveillance domain. To understand this issue, recall from the beginning of Section \label{sec_survDom} that two MDPs $\mdp$, $\mdp'$, corresponding to v-location sequences $(v_1, v_2, v_3)$ and $(v'_1, v'_2, v'_3)$, are similar in terms of their optimal policy, and hence belong to the same cluster, if each pair of v-locations $v_i$ and $v'_i$ are adjacent. In this case surveiling $v_i$ instead of $v'_i$ yeilds a reward of $190$ in $\mdp'_i$ and {\em vice versa}. Hence within each correct cluster of MDPs we have a {\em symmetry}: the $d_V$ distance between each pair of MDPs is very similar. At the same time the $d_V$ distance to MDPs not in the same cluster is quite different. So if in greedy clustering if we start with any MDP in a given cluster, we will find the other MDPs in the cluster. Greedy clustering will no longer work if we break this symmetry by explicitly modifying the domain definition. 

One simple way to do this is as follows. In the complex domain, for each correct cluster we have a single MDP (the centroid $\mdp^i$) that has low $d_V$ distance to all other elements of the cluster, while every other MDP of the cluster has a high $d_V$ distance to the other non-centroid elements. To get this effect, we have 5 v-locations per group of adjacent locations. We then change the reward function so that in each group of adjacent locations $\{v^1, v^2, v^3, v^4, v^5\}$ (see figure \ref{fig_graphDom}), the reward is asymmetric - we designate a single v-location, say $v^1$ which yields a reward of $190$ if it is surveilled instead of $v^i$ and vice-versa, but for all the other v-locations, surveilling $v^j$ instead of $v^i$, $i, j \neq 1$, yields a reward of $-10$. Hence the $\mdp^i$ would be the MDP with v-location sequence $(v^1_1, v^1_2, v^1_3)$. In this case, greedy clustering will fail to learn every cluster for which it does not start with the centroid MDP $\mdp^i$ because the for the non-centroid MDPs, the $d_V$ distance to the other MDPs would be too large. The result for the complex domain is given in table  \ref{table_clustComp_star}. We see that, transfer using the clusters produced by Search-Clustering outperforms transfer using clusters discovered by greedy clustering approach by a factor of 2 as we increase the number of previous MDPs. Recall that each point is obtained from averaging over 10 different target tasks. 

\begin{table}[!h]
\begin{center}
\begin{tabular}{ | c | c | c |}
\hline
{\bf No. of Previous MDPs} & {\bf Search Clustering} & {\bf Greedy Clustering} \\
\hline
100 & $8.7102 \times 10^6$ & $8.3667 \times 10^6$ \\
\hline
200 & $8.3667 \times 10^6$ & $3.809 \times 10^6$ \\
\hline
\end{tabular}
\end{center}
\caption{This table gives the performance of EXP-3-Transfer in the complex surveillance domain when using greedy clustering vs Search-Clusterings for tasks with 2 target v-locations. The performance is measured in terms of the total cumulative discounted rewards over $10,000$ episodes and averaging over 10 different target tasks. The table shows that the performance of Search-Clusterings remains steady while the performance of greedy clustering falters when we have more previous tasks to draw from.}\label{table_clustComp_star}
\end{table}

%To show cases where greedy clustering fails, we performed additional experiments on a more sophisticated variant of this domain, where the greedy clustering method may fail, and therefore Search-Clustering is necessary. In the original surveillance domain, if the agent surveilled any v-location in the same group as the true v-location required by the task, then it received near-optimal rewards (see Figure \ref{fig_survDom_traj}). This notion defined the task similarity in the domain. We now consider a modified similarity measure defined by a graph over the v-locations. If there is an edge between the v-locations then the reward obtained is $190$, otherwise it is $200$. This is illustrated in Figure \ref{fig_graphDom}. The interpretation of the centre v-location being at the top of a hill and the $4$ other surrounding points being at the foot of that hill. Hence, if we surveil the hill-top location we can also surveil the lower locations automatically. Surveilling one lower location means we may surveil the 
%hill-top location, but not necessarily the other locations. 

 \begin{figure}[!h]
 \begin{center}
 \includegraphics[scale=0.7]{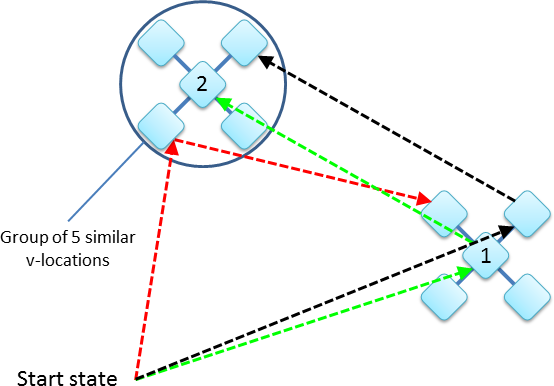}
 \end{center}
 \caption{{\sc The complex graph domain.} The complex graph domain consists of groups of 5 adjacent v-locations. In our experiments we used $16$ groups but for simplicity we only show two groups here. Surveilling the central v-location $v^1$ instead of any of the others has reward $190$ and vice-versa, while survelling $v^i$ instead of $v^j$, $i, j \neq 1$, has reward $-10$. As before, each MDP consists of surveilling the correct sequence of target v-locations, depicted here as `1' followed by `2'. As in Figure \ref{fig_survDom_traj}, a green line shows correct movement (reward $200$), a black line shows movement to the wrong v-location in the right group (reward $190$), and a red line corresponds to moving to the wrong group before surveilling (reward $-10$).
 %the black/solid line shows the optimal trajectory for a particular MDP. The green/dotted line shows a trajectory that incurs a penalty but is acceptable. A red/dashed line indicates a trajectory where the surveil action at the end incurs  a reward of $-10$.
 }\label{fig_graphDom}
 \end{figure}

\subsubsection{Time Comparisons}\label{sec_expTimeComp}

Finally we examine the effect of the $T$ parameter on performance on the original surveillance domain. Recall that the $T$ parameter affects both the clustering algorithm Search-Clustering and EXP-3-Transfer, and is the time duration in terms of the number of episodes over which the transfer procedure is run. We performed experiments with $7$ different combinations of numbers of previous MDPs and MDP complexity. The results of these are all qualitatively similar, and so we present only two graphs in Figure \ref{fig_time_all} for the most complex and the least complex transfer problem we have considered. We relegate the remaining graphs for the rest of the experiments to Appendix \ref{app_timeComp}. 

\begin{figure}[!h]
      \subfloat[{\sc Low complexity transfer problem.}] { \label{fig_time_low}
      \includegraphics[width=0.5\textwidth]{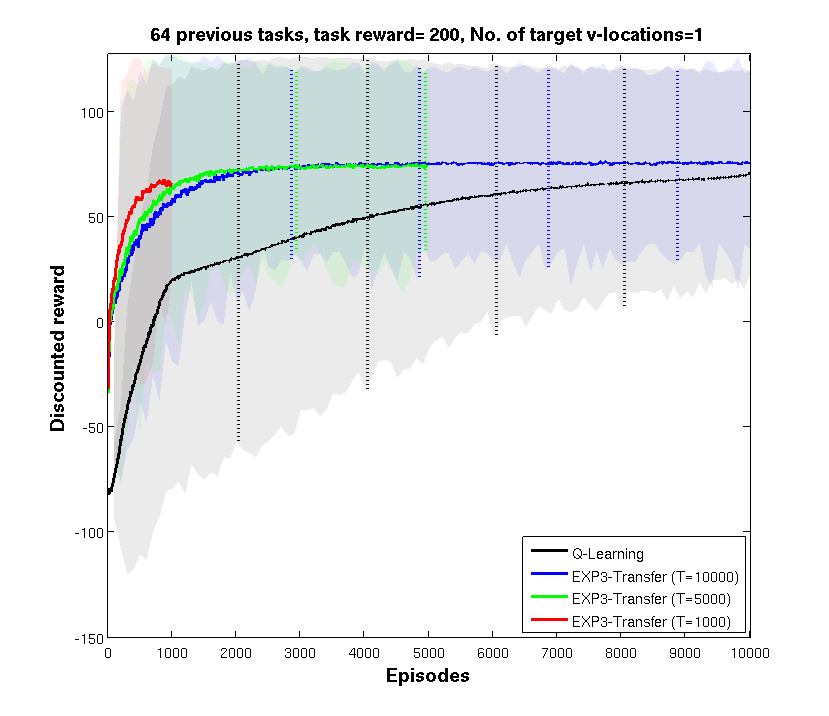} }
          \subfloat[{\sc High complexity transfer problem.} ] { \label{fig_time_high}
        \includegraphics[width=0.5\textwidth]{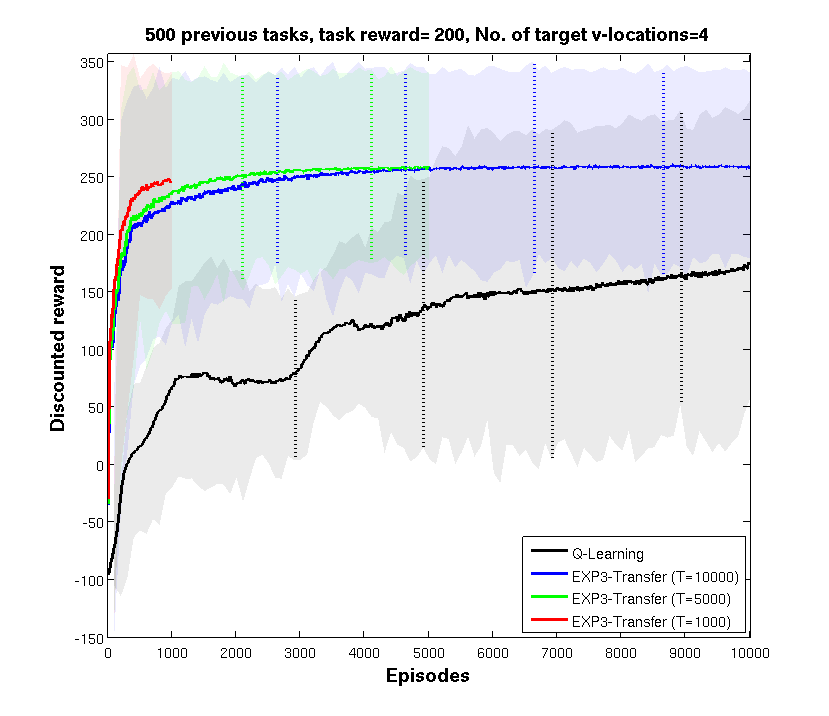}}
\caption{{\sc Effect of the $T$ parameter}. These figures show the learning curve of EXP-3-Transfer when run for different numbers of time steps (parameter $T$). This affects both the clustering and the arms chosen by EXP-3-Transfer. The parameters for each experiment are given in the title of the respective figure. As shown, for shorter $T$, the EXP-3-Transfer run with the lowest $T = 1,000$ is optimal. For the intermediate duration, $T=5,000$ is optimal, and for the remaining time $T=10,000$ is optimal.  Figures \ref{fig_time_low} and \ref{fig_time_high} respectively give the curves for the lowest and highest complexity task that was run. The shaded areas indicate one standard deviation.}
 \label{fig_time_all}
\end{figure}

As these results show, the performance curve for EXP-3-Transfer with parameter $T$ lies above the curves of EXP-3-Transfer with parameter $T' > T$, for any $t \leq T$. This illustrates the effect of optimising transfer performance for a set time duration.

\section{Conclusion}\label{sec_conc}

In this paper we developed a framework to concisely represent a large number of previous MDPs by a smaller number of source MDPs for transfer learning. We presented a principled online transfer learning algorithm, a principled way to evaluate source sets for use in this algorithm and way to find the source set. The key idea was to cluster the previous MDPs and then use the representative element of each cluster as the source tasks. We also presented extensive experiments to show the efficacy of our method. We now discuss several interesting  directions for future work.

In this paper we only considered discrete domains. However, it is possible to translate the overall approach to the continuous setting. In particular, to apply our approach to continuous space problem, all we will need is a pure RL algorithm (as an arm in EXP-3-Transfer) and a way to evaluate policies (to compute the $d_V$ distances). All our definitions, algorithms and results will then hold true in this setting. This is because our algorithms EXP-3-Transfer, MHAV and Search-Clustering and distance function $d_V$ treats the underlying MDPs and policies as black boxes with certain properties. The discreteness of the MDP is never exploited or required in either the algorithms or their analysis. 

%We also only looked at only one possible VPL metric. As we pointed out in the introductory material of Section \ref{sec_exp}, this metric is not particularly interesting, and additionally, after extensive search we were unable to discover another one. So another possible interesting line of future enquiry is to fully develop the theory of VPL metrics as it results in a theory that holds under weaker assumptions. This work may involve the development of measures of similarity used in related work on bisimulation \citep{castro_precup:2010}.

Finally, we end by pointing out that the idea of clustering a set of tasks to obtain a representative set is much more general. For instance, any other cost function derived under different assumptions can be applied with the clustering approach. As another example, the clustering approach may also be used in multi-agent systems to group together opponents according to whether the same policy of ours is equally effective against opponents in the same group. It will also be interesting to implement these methods and algorithms on scaled up, real version of the types of problems considered in this paper. We plan to pursue these and other extensions to the above in future work.

\clearpage

\clearpage 
\appendix

\section{Proofs}\label{app_proofs}

\begin{proof}[Proof of Theorem \ref{thm_EXPBound}]
A direct application of Corollary 3.2 in \citep{auer_cesa-bianchi_freund_schapire:2002} is not possible because our algorithm diverges from EXP-3 because the number of arms possibly decreases across time steps. The proof of the first part, while structurally similar tp the proof of Theorem 3.1 of \cite{auer_cesa-bianchi_freund_schapire:2002}, is different in some crucial detail due to the removal of arms/policies.  The second part, where we deal with arms that were removed, is novel.

%Let $C_t \eqdef \{1,2,\cdot,c+1\} - rem_t$ be the arms that are in play at line 4 of EXP-3-Transfer, and let $c_t = |C_t|$. Let $W_t \eqdef \sum_{i \in C_t} w_i(t)$, and $\tldW_{t-1} \eqdef \sum_{i \in C_t} w_i(t-1)$ (note the $C_t$, rather than $C_{t-1}$ in the summation in $\tldW$).  Then for all sequences of policies $i_1,i_2,\cdots, i_T$,  drawn by EXP-3-Transfer, 

Recall that $C_t$ is the set of policies not yet eliminated at the beginning of step $t$ EXP-3-Transfer, and let $c_t = |C_t|$. Let $W_t \eqdef \sum_{i \in C_t} w_i(t)$, and $\tldW_{t+1} \eqdef \sum_{i \in C_t} w_i(t+1)$ (note the $C_t$, rather than $C_{t+1}$ in the summation in $\tldW_{t+1}$). Then for all sequences of policies $i_1,i_2,\cdots, i_T$,  drawn by EXP-3-Transfer, 
\begin{align}\label{eq_Wt_Wtm1_rat}
\nonumber
& \frac{W_{t+1}}{W_t} \ovst{(1)}{\leq} \frac{\tldW_{t+1}}{W_t} 
\ovst{(2)}{=} \sum_{i \in C_t} \frac{w_{i}(t+1)}{W_t} 
\ovst{(3)}{=} \sum_{i \in C_t} \frac{w_{i}(t)}{\sum_{i\in C_{t}} w_i(t)}\exp[\beta \hat{x}_i(t)/(c+1)]\\
\nonumber 
& \ovst{(4)}{=} \sum_{i \in C_t} \frac{p_i(t) - \beta/c_t}{1-\beta}   
 \exp[\beta\hat{x}_i(t)/(c+1)]
\ovst{(5)}{=} \sum_{i \in C_t} \frac{p_i(t) - \beta/c_t}{1-\beta}   
 \exp[\beta\hat{x}_i(t)/c_t]
\\
&\ovst{(6)}{\leq} 1 + \frac{\beta/c_t}{1-\beta}x_{i_t}(t) + \frac{(e-2)(\beta/c_t)^2}{1-\beta}\sum_{i \in C_t} \hat{x}_i(t)
\end{align}
In the above, $^{(1)}$ follows because $\tldW_{t+1} \geq W_{t+1}$ as $|C_t| \geq |C_{t+1}|$, and because all the weights $w_{t+1}$ are positive. $^{(2)}$ follows by the update equation for $w_i(t+1)$ in line \ref{stp_updatewt} of EXP-3-Transfer. $^{(3)}$ follows by definition of $W_t$. $^{(4)}$ follows by definition of $p_i(t)$ in line \ref{stp_defpit} of EXP-3-Transfer. $^{(5)}$ holds because $c_t \leq c+1$ and because the term in the exponential is positive. Finally, $^{(6)}$ follows by the identical reasoning used to derive (8) in the proof of Theorem 3.1 \citep{auer_cesa-bianchi_fischer:2002}. 
%In essence, we introduced $\tldW_{t+1}$ to get (\ref{eq_Wt_Wtm1_rat}) as the reasoning in the aforementioned Theorem 3.1 no longer hold true for $W_{t+1}/W_t$ as the number of arms may change from step to step.  However, since $W_t$ and $\tldW_{t+1}$ are summed over the same number of arms, the reasoning holds true.

Now, we can proceed along similar lines as in Theorem 3.1. Using the fact that $1+x \leq \exp(x)$ gives us 
\begin{equation}\nonumber
\ln\frac{W_{t+1}}{W_t} \leq \frac{\beta/c_t}{1-\beta}x_{i_t}(t) + \frac{(e-2)(\beta/c_t)^2}{1-\beta}\sum_{i \in C_t} \hat{x}_i(t)
\end{equation}
Now since $c_t$ is non-increasing, $\beta/c_t \leq \beta/c_{T}$, and so
\begin{equation}\nonumber
\ln\frac{W_{t+1}}{W_t} \leq \frac{\beta/c_T}{1-\beta}x_{i_t}(t) + \frac{(e-2)(\beta/c_T)^2}{1-\beta}\sum_{i \in C_t} \hat{x}_i(t)
\end{equation}
Summing over $t$ telescopes and gives us 
\begin{equation}\nonumber
\ln\frac{W_{T}}{W_1} \leq \frac{\beta/c_T}{1-\beta}\sum_{t=1}^T x_{i_t}(t) + \frac{(e-2)(\beta/c_T)^2}{1-\beta}\sum_{t=1}^T\sum_{i \in C_t} \hat{x}_i(t)
\end{equation}
Now, going in the opposite direction, for each $\rho_j \in C_T$ we have
\begin{equation}\nonumber
\ln\frac{W_{T}}{W_1} \geq \ln \frac{w_j(T)}{c+1} = \frac{\beta}{c_T}\sum_{t = 1}^T\hat{x}_j(t) - \ln(c+1)
\end{equation}
Putting these together, we have for each $\rho_j \in C_T$,
\begin{equation}
\sum_{t=1}^T x_{i_t}(t) \geq (1-\beta)\sum_{t = 1}^T\hat{x}_j(t) - \frac{c_T\ln(c+1)}{\beta} - \frac{(e-2)\beta}{c_T}\sum_{t=1}^T  \sum_{i \in C_t} \hat{x}_i(t)
\end{equation}
Rearranging, we get 
\begin{equation}
\sum_{t = 1}^T\hat{x}_j(t) - \sum_{t=1}^T x_{i_t}(t) \leq \beta\sum_{t = 1}^T\hat{x}_j(t) + \frac{c_T\ln(c+1)}{\beta} + \frac{(e-2)\beta}{c_T}\sum_{t=1}^T  \sum_{i \in C_t} \hat{x}_i(t)
\end{equation}
Taking expectation in terms of the randomization in the algorithm in both sides above, conditional of $\rho_j \in C_T$,  we have  
\begin{equation}
\sum_{t = 1}^T x_j(t) - \expc_{E3T}[\sum_{t=1}^T x_{i_t}(t)| \rho_j \in C_T ] \leq \beta\sum_{t = 1}^T x_j(t) + \frac{c_T\ln(c+1)}{\beta} + \frac{(e-2)\beta}{c_T}\sum_{t=1}^T  \sum_{i \in C_t} x_i(t)
\end{equation}

For the final term on the right hand side, we used the fact  $\expc[\hat{x}_i(t)|i_1,i_2,\cdots,i_{t-1}, \rho_j \in C_T] \leq x_i(t)$ for any $i$. To see this, note that there are two possibilities -- either $\rho_i \in C_t$, or $\rho_i \not\in C_t$. If $\rho_i \in C_t$ then $\expc[\hat{x}_i(t)|i_1,i_2,\cdots,i_{t-1}, \rho_j \in C_T] = x_i(t) \geq 0$, and otherwise, $\expc[\hat{x}_i(t)|i_1,i_2,\cdots,i_{t-1}, \rho_j \in C_T] = 0$ as $p_i(t) = 0$. 

Now, using the fact that $\sum_{t=1}^T \sum_{i \in C_T} x_i(t) \leq c_T T$ (as $x_i(t) \in [0,1]$) and then rearranging, we get for each $\rho_j\in C_T$, 
\begin{multline}\nonumber
\sum_{t = 1}^T x_j(t) - \expc_{E3T}[\sum_{t=1}^T x_{i_t}(t)| \rho_j \in C_T ] \leq \beta\sum_{t = 1}^T x_j(t) + c_T\ln(c+1) + (e-2)\beta T \\= \frac{c_T\ln(c+1)}{\beta} + (e-1)\beta T
\end{multline}
Plugging in the value of $\beta$ in the theorem statement we get
\begin{equation}\label{eq_unrem}
\sum_{t = 1}^T x_j(t) - \expc_{E3T}[\sum_{t=1}^T x_{i_t}(t)| \rho_j \in C_T ] \leq  
\left( \frac{|C_T|}{\sqrt{c+1}} + \sqrt{c+1} \right ) 
\sqrt{(e-1)\ln(c+1)/T}
\end{equation}

Now note that $\Delta R(1-\gamma)^{-1} \expc_{P,R}[\sum_{t = 1}^T x_j(t)] = TV^{\rho_j}$, so taking expectation with respect to $P,R$, dividing by $T$ and multiplying by $\Delta R(1-\gamma)^{-1}$ we get from (\ref{eq_unrem})
\begin{equation}\nonumber
V^{\rho_j}  - \frac{1}{T}\expc_{E3T}[\sum_{t=1}^T \bar{x}_{i_t}(t)| \rho_j \in C_T ] \leq 
\frac{\Delta R}{1-\gamma} \left( \frac{|C_T|}{\sqrt{c+1}} + \sqrt{c+1} \right ) 
\sqrt{(e-1)\ln(c+1)/T}
\end{equation}
This completes the proof of the first part.

To prove the second part, we need the Hoeffding bound (see, for instance, \citep{dubhashi_panconesi:2009} for an exposition) which states that if $y_{1:n} \eqdef y_1,y_2,\cdots,y_n$ are i.i.d. draws of a random variable $Y$, with $Y_i \in [a,b]$, and $\bar{y}_n$ is the empirical mean of the $y_i$, then 
\begin{equation}\label{eq_HoeffBnd}
Pr[ |\bar{y}_n- \expc(Y)| > \eps] \leq \exp[-2n\eps^2/(b-a)^2]
\end{equation}
In the sequel, we will assume that $b = 1, a= 0$, and so the denominator in the exponent on the left hand side of the equation is just $1$. This bound then has the following  simple and well known consequence. Assume we have two i.i.d. samples $y_{1:n}$ and $y'_{1:m}$, drawn from two random variables $Y$ and $Y'$. Assume that $\bar{y}_n - \bar{y'}_m > \eps$, and $n$ and $m$ both satisfy $\exp[-2n\eps^2/4] \leq \delta'/2$ and $\exp[-2m\eps^2/4] \leq \delta'/2$. Then, by (\ref{eq_HoeffBnd}) 
\begin{equation}
Pr[ |\bar{y}_n- \expc(Y)| > \eps/2] \leq \delta'/2, \qquad Pr[ |\bar{y'}_n- \expc(Y')| > \eps/2] \leq \delta'/2 
\end{equation}
Then, by the triangle inequality and the union bound, with probability at least $1-\delta'$, $\expc(Y) > \expc(Y')$. 

Now in line \ref{stp_deltaRemove} of EXP-3-Transfer, we remove a source policy arm if  $\eps = z_j/n_j - z_k/n_k$, we have  $\eps/2 > \sqrt{-\ln (\delta/2c) (2n_j)^{-1}}$ and $\eps/2  > \sqrt{-\ln (\delta/2c)(2n_k)^{-1}}$. This implies, that with probability $> 1-\delta/c$,  $V^{\rho_j} > V^{\rho_k}$. Since there are $c$ arms, this implies that if there is an arm that was removed, by the union bound with probability at least $> 1-\delta$, $V^{\rho_j} > V^{\rho_k}$ for some arm $j$ for every arm $k$ that is eventually removed. 

%Now note that the expectation of the first term $\sum_{t = 1}^T x_j(t)$ in (\ref{eq_unrem}) is $T V^{\rho_j}$ when $\rho_j$ is a stationary source policy. Hence, coupling the results of this paragraph with (\ref{eq_unrem}) we get that 
%\begin{equation}\nonumber
%\expc [\sum_{t = 1}^T x_k] - \expc[ G_{E3T}] \leq 2.63 \sqrt{c\ln(c) T}
%\end{equation}
%which is what we are required to prove.
\end{proof}

%\begin{proof}[Proof of Lemma \ref{lemma_dvnometric}]
%Let there be three MDPs $\mdp_1,\mdp_2,\mdp_3$ defined on a state space with a single state and three actions $a_1,a_2,a_3$. Assume that $R_1(a_1)=100, R_1(a_2)=90,R_1(a_3)=-100 $ and for $i \in \{2,3\}$, $R_i(a_i) = 100$, and $R_i(a_j) = 90$ when $i \neq j$. So the optimal action for $\mdp_i$ is $a_i$. But now, $d_V(\mdp_1,\mdp_2) = 100-90 = 10$, $d_V(\mdp_2,\mdp_3) = 100 - 90$, but $d_V(\mdp_1,\mdp_3) = 100- (-100) =200$, showing that $d_V$ does not satisfy the triangle inequality and hence is not a metric.
%\end{proof}
%
%For the next proof, we need to restate Lemma 1 in \citep{strehl_littman:2008}. 
%\begin{lemma}\label{lemma_strehl}[Strehl, Li and Littman]
%Let $\mdp_1$ and $\mdp_2$ be two MDPs defined on the same state-action space, and $\eps_1 = |R_1(s,a) - R_2(s,a)|$, and $\eps_2 = \lonorm{P_1(\cdot|s,a) - P_2(\cdot|s,a)}$, then for any policy $\pi$ and state $s$,
%\begin{equation}\label{eq_strehl}
%|V_1^\pi(s) - V_2^\pi(s)| \leq (\eps_1 + \gamma R_{\max} \eps_2)(1-\gamma)^{-2}
%\end{equation}
%$\square$
%\end{lemma}

For the next two proofs, we need to restate Lemma 1 in \citep{strehl_littman:2008}. 
\begin{lemma}\label{lemma_strehl}[Strehl, Li and Littman]
Let the new $\mdp_{N+1}$ have transition and reward functions $R_{N+1}$ and $T_{N+1}$. Let $|R_i(s,a) - R_{N+1}(s,a)| \leq K_i$ and $|T_i(.|s,a) - T_{N+1}(.|s,a)|  \leq K'_i$ where $R_i$ and $T_i$ are the reward and transition functions for MDP $\mdp_i$. Then, for any policy $\pi$ and state $s$,
\begin{equation}\label{eq_strehl}
|V_{N+1}^\pi(s) - V_i^\pi(s)| \leq K(i)
\end{equation}
where $K(i) \eqdef \frac{K_i + \gamma K'_i R_{max}}{(1-\gamma)^2}$ (first defined in (\ref{eq_Kidef})).
$\square$
\end{lemma}
Now we can state our proof.

\begin{proof}[Proof of Lemma \ref{lemma_avgQuant}]
Fix any previous $\mdp_k$ and let $\mdp^{i}$ be the centroid of the cluster $A_{i}$ of $\bfA$ such that $\mdp_k \in A_{i}$. By definition, the optimal policy of $\mdp^{i}$, used as an arm in EXP-3-Transfer, is $\rho_{i}$ and $V^{\pi^*_k}_k - V^{\rho_{i}}_k \leq \eps_{i}$. By Lemma \ref{lemma_strehl}, $|V^{\pi^*_k}_{N+1} - V^{\pi^*_k}_k| \leq K(k)$ and $|V^{\rho_{i}}_{N+1} - V^{\rho_{i}}_k| \leq K(k)$. Putting the three inequalities together, we have $|V^{\pi^*_k}_{N+1} - V^{\rho_{i}}_{N+1}| \leq \eps_{i} + 2K(k)$. By Theorem \ref{thm_EXPBound}, if $\rho_i$ was not eliminated by EXP-3-Transfer, $V^{\rho_{i}}_{N+1} - \expc[G_{E3T}]/T \leq g(c)$, and therefore $V^{\pi^*_k}_{N+1} - \expc[G_{E3T}]/T \leq g(c) + \eps_{i} + 2K(k)$.

Now consider the case where the policy $\rho_i$ was eliminated at some point. Then, by the second part of Theorem \ref{thm_EXPBound}, there exists a policy $\rho_{i'}$ for which with probability at least $1-\delta$, $V^{\rho_i} \leq V^{\rho_{i'}}_{N+1}$. From the relationship between $\rho_{i}$ and $\pi^*_k$ established above, we get that with probability $1-\delta$, $V_{N+1}^{\pi^*_k} \leq V_{N+1}^{\rho_{i'}} + \eps_i + 2K(k)$.

%\begin{equation}\nonumber
%\expc[V^{\pi^*_k}_{N+1}] - \expc[G_{E3T}]/T \leq g(c) + \bar{\eps} + 2\bar{K}
%\end{equation}
%which completes the proof.
\qed
\end{proof}

\begin{proof}[Proof of Theorem \ref{thm_npcomplete}]
First, let $|V| = M$, and given any ordering of the elements of $V$, identify each vertex $v \in V$ with its position in the ordering -- so we can take $V = \{1,2,\cdots,M\}$. Let $\mdp_1,\mdp_2,\cdots,\mdp_{M}$ be a set of MDPs defined on a state space  $\stts = \{s\}$, and action space $\acts = \{1,2,\cdots,M\}$. The transition function for the MDPs in this case is trivial (all actions transition with probability $1$ from $s$ to $s$). The reward function for MDP $i$ defined as follows. $R_i(s,a_i) = 0$; if $(i,j) \in E$ then $R_i(s,a_{j}) = 0$, otherwise  $R_i(s,a_{j}) =- h M g(M)$ where $h > 1$ and $g$ is the function used in Definition \ref{def_cost2} to define the cost function for clusters. In the following we will identify MDP $\mdp_i$ with vertex $i\in V$ and this way show that the optimal clustering for this corresponds to a maximal clique under the mapping $i \rta v_i$ and $A \rta V_i$.

By construction, $\pi^*_i(s) = a_{i}$, $V^*_i = 0$, $V^{\pi^*_j}_i = V^{a_j}_i = 0$ iff $(i,j) \in E$, and $V^{a_j}_i = - h M g(M)$ otherwise. Hence, by the definition in (\ref{eq_mdpDDef}), 
\begin{equation}\label{eq_npDVCond}
d_V(\mdp_i,\mdp_j) = 
\begin{cases}
0 &\mbox{ iff } (i,j) \in E\\
h M g(M) &\mbox{ otherwise } \\
\end{cases}
\end{equation}
Now recall that $cost_{m} \eqdef g(c) + \bar{\eps}_m$. Let an optimal clustering be $\bfA^*$ and let $A(\mdp_i)$ denote the cluster in $\bfA^*$ that $\mdp_i$ belongs to. We now show that if $\mdp_{i_1},\mdp_{i_1},\cdots,\mdp_{i_l} \in A \in \bfA^*$, then $i_1,i_2,\cdots,i_l$ form a clique in $G$. In other words, we show that, if $A(\mdp_i) = A(\mdp_j)$ then $(i,j) \in E$, or equivalently $(i,j)\not\in E$ then $A(\mdp_i) \neq A(\mdp_j)$. By way of contradiction, assume that $(i,j) \not\in E$ but $A(\mdp_i) = A(\mdp_j)$. Since $i,j$ do not have an edge between them, by (\ref{eq_npDVCond}) the diameter of $\bfA^*$ is at least $d_V(\mdp_i,\mdp_j) = h M g(M) / M = h g(M)$. Which in turn implies that $cost_{m}(\bfA^*) = g(|\bfA^*|) + hg(M)$. Now consider the clustering $\bfA'$ obtained by putting each MDP $\mdp_i$ in its own cluster. This clustering has cost $g(M) + 0 < g(|\bfA^*|) + hg(M)$ -- contradicting the optimality of $\bfA^*$. Hence, the clusters of $\bfA^*$ has cost $g(|\bfA^*|)$  and corresponds to a collection of cliques that partition $V$ -- denote this collection of cliques by $J^*$.

Now note that each collection of cliques $V_1,V_2,\cdots V_j$ that partition $V$ correspond to a clustering $\bfA$ such that $\mdp_i,\mdp_j \in A$ iff $i,j \in V_l$ for some $l$; in this case $j = |\bfA|$.  Now assume that there is a collection of cliques $I$ such that $|I| < |J^*|$ and let the corresponding clustering be $\bfA_I$. Then we show that $cost_{m} (\bfA_I) < cost_{m}(\bfA^*)$, resulting in a contradiction.  To see this note that each $\mdp_i,\mdp_j \in A \in \bfA_I$ then  $d_V(\mdp_i,\mdp_j) = 0$ by (\ref{eq_npDVCond}). Hence the diameter of $\bfA_I$ = $0$. So the cost of $\bfA_I$ is $g(|\bfA_I|) + 0 < g(|\bfA^*|)$ since by definition of $\bfA_I$, $|\bfA_I| < |\bfA^*|$. 

Because of the contradiction, $J^*$ is indeed a minimum clique cover, showing that the problem of minimum clique cover can be reduced to the problem of finding the optimal clustering. To complete the proof, we need to show that this reduction takes polynomial time. The only cost in computing a $\mdp_i$ is setting the reward function, which takes time $C|V|$ for some constant $C$. 
\end{proof}

%$\frac{1}{N} [\sum_{k=1}^N V^{\pi^*_k}_{N+1} - N\expc[G_{E3T}]/T \leq g(c) + \frac{1}{N} \sum_{k=1}^N \eps_{j_k}$

%\subsection{Proofs From Section \ref{sec_clustAlg}}

\begin{proof}[Proof of Lemma \ref{lemma_lpIrr}]
To show irreducibility we have to show that for any $(\lambda,y)$ and $(\lambda',y')$ there exists a $n$ such that $\psmh^n[\lambda',y'| \lambda,y ]> 0$. To see this, first note that $\phi_Y$ was assumed to be irreducible. So, there exists a $n_1$ such that with $\phi_Y^n( y' | y) > 0$. Now consider a particular path $\bfy \eqdef yy_1y_2\cdots y_{n-1}y'$ with probability $>0$ under $\phi_Y$. From the definition in (\ref{eq_PMHDef}), the probability under $\psmh$ of each transition $y_{i} \rta y_{i+1}$ is 
\begin{equation}\nonumber
\beta \bar{\phi}[\lambda,y_{i+1}| \lambda,y_i ]\bar{\accpt}_{ \lambda,y_i}[\lambda,y_{i+1}] 
> \beta b \phi_Y[\lambda,y_{i+1}|\lambda,y_i]
\end{equation}
where the inequality follows as $\bar{\accpt}_\cdot[\cdot] > b$ by assumption in the Lemma statement. Hence, the total probability of the path $yy_1y_2\cdots y_{n-1} y'$ under $\psmh$ is lower bounded by $b^n\beta^n \phi_Y(\bfy)$ (where $\phi(\bfy) = \prod_{i=0}^n\phi(y_{i+1}|y_i)$). Summing over all possible paths of length $n$ going from $y$ to $y'$ gives that the probability of each $(\lambda,y')$ from $(\lambda,y)$ is lower bounded by $b^n\beta^n \phi_Y^n( y' | y) > 0$.

Now assume that $\lambda = \lambda_k$ while $\lambda' = \lambda_{k'}$. If $k < k'$, we can bound the probability under $\psmh$ of going from $(\lambda_i,y')$ to $(\lambda_{i+1},y')$,  where $k \leq i < k'$, by $z_i \eqdef \alpha\alpha'\frac{(1-\alpha)}{\alpha} (\lambda_{i+1}/\lambda_i)^{-f(y)}$ (this follows from definition of $\psmh$ and $\bar{\accpt}$). Hence we reach $(\lambda',y')$ from $(\lambda,y')$ with probability $z \eqdef \prod_{i=k}^{k'-1}z_i$. By a symmetric argument, if $k' < k$, we reach $\lambda'$ from $\lambda$ with probability at least $z' \eqdef \prod_{i=k}^{k'-1}{z'}_i$, where $z'_i \eqdef \alpha(1-\alpha')\frac{\alpha}{(1-\alpha)}(\lambda_{i}/\lambda_{i+1})^{-f(y)}$. Both $z,z'$ are positive by the finiteness of $f(y)$ and $\lambda_i$s. Putting all the above together, we have that the probability of transitioning from $(\lambda,y)$ to $(\lambda',y')$ is lower bounded by
\begin{equation}\nonumber
b^n\beta^n \phi_Y^n(y' |  y) \min\{z,z'\} > 0
\end{equation}
which shows that $\psmh$ is irreducible.

To show that $\psmh$ is aperiodic, it is sufficient to note that $\alpha + \beta < 1$. Then, with probability $1-\alpha-\beta$, $\psmh$ returns to the same state in $1$ step, which ensures that the g.c.d. of the set of time steps where $\psmh$ returns to the same state is $1$. \qed
\end{proof}

\begin{proof}[Proof of Theorem \ref{thm_drawOpt}]
$\bar{\phi}$ is irreducible by Lemma \ref{lemma_lpIrr} and by construction of an MH chain, $\psmh$ has $\bar{\Pi}$ stationary distribution. Hence, by the first part of Theorem \ref{thm_stationConv} $\psmh$ converges to $\bar{\Pi}$ in total variation. By the second part of the same theorem,  if $\ltvnorm{\psmh^t(\cdot|x_0) - \bar{\Pi}(\cdot)}\leq k$, then for all $t' >t$, $\ltvnorm{\psmh^{t'}( \cdot | x_0) - \bar{\Pi}(\cdot)}\leq k$. \qed
\end{proof}

\begin{proof}[Proof of Theorem \ref{thm_SCconv}]
As we mentioned above, this proof follows very closely the proof of Theorem 4.9  in \citep{levin_peres_wilmer:2009}. To begin with, first we note that by irreducibility of $\psmh$, the diameter $D$ (defined in (\ref{eq_diameter})) is finite. Hence, by definition of $\delta$ in (\ref{eq_delratio}), for each $x,x'$ we have that $\psmh(x'|x) \geq \delta \bar{\Pi}(x')$. 

Let $\msmh$ denote the transition matrix for the kernel $\psmh$ so that $\msmh(x,x') = \psmh(x'|x)$ -- i.e. row $i$ contains the distribution $\psmh(\cdot|x_i)$. Let $\bbfPi$ denote the transition matrix where each row is $\bar{\Pi}$. Then, setting $\eta \eqdef (1-\delta)$, we can write
\begin{equation}\nonumber
\msmh = (1-\eta)\bbfPi + \eta Q
\end{equation}
where $Q$ is another transition matrix. To see that $Q$ is a valid transition matrix, note that row $i$ of $Q$ is given by $\eta^{-1} [\psmh(\cdot|x_i) - (1-\eta) \bar{\Pi}(\cdot)]$. Summing the elements of this row, we get $\sum_{x'} \psmh(x'|x_i) - (1-\eta) \bar{\Pi}(x')$ = $\eta$, whence each row of $Q$ sums to $1$. Furthermore, by the definition that  $(1-\eta) = \delta$, each entry is also positive, showing that $Q$ is indeed a valid transition matrix.

\newcommand{\tmtx}{\mathbf{M}}

Now note that for any transition matrix $\tmtx$, $\tmtx \bbfPi = \bbfPi$. Additionally, since $\bar{\Pi}$ is stationary for $\psmh$,  $\bbfPi \msmh = \bbfPi$. We will now use the above facts to show by induction on $k$ that 
\begin{equation}\label{eq_convHyp}
\msmh^{Dk} = (1-\eta^k) \bbfPi + \eta^k Q^k
\end{equation}
which will imply the convergence we seek.

Clearly (\ref{eq_convHyp}) is true for $k = 0$. Assume, as the inductive hypothesis, that it is true for $k\leq n$. Then, we have
\begin{align}
\nonumber
&\msmh^{D(n+1)} = \msmh^{Dn}\psmh^{D} \\
\nonumber
&= (1-\eta^n) \bbfPi + \eta^n Q^n \msmh^{D}   \\
\nonumber
&= (1-\eta^n) \bbfPi + \eta^n Q^n [ (1-\eta) \bbfPi + \eta Q] \\
\nonumber
&= (1-\eta^n) \bbfPi  -\eta^{n+1} \bbfPi + \eta^n\bbfPi +   \eta^{n+1}Q^{n+1}\\ 
\nonumber
&= (1- \eta^{n+1} )\bbfPi + \eta^{n+1}Q^{n+1} 
\end{align}
The first equality is just the definition of $k$-step transitions. The second equality is obtained by applying the inductive hypothesis and because $\bbfPi \msmh = \bbfPi$. The third and fourth equality follows from applying the inductive hypothesis on $\msmh^D$ and the two facts about $\bbfPi$ established above. The final equality is obtained by cancelling out the terms. 

Now $\eta^k \rta 0$ as $k \rta \infty$ , and so each row of $\msmh$ converges to $\bar{\Pi}$. In other words for each $x$, $\lim_{t\rta \infty} \sum_{x'}\psmh^t (x'|x) - \bar{\Pi}(x') = 0$. This implies  $\lim_{t\rta \infty} \lonorm{\psmh^t (\cdot|x) - \bar{\Pi}(\cdot)} = 0$. Now since $\ltvnorm{\psmh^t (\cdot|x) - \bar{\Pi}(\cdot)}$ $=$ $\frac{1}{2} \lonorm{\psmh^t (\cdot|x) - \bar{\Pi}(\cdot)}$, this completes the proof.
\qed
\end{proof}

\begin{proof}[Proof of Lemma \ref{lemma_Dindep}]
Fix any two $(\lambda,y)$ and $(\lambda',y')$ and let $\bfx \eqdef x_0x_1\cdots x_n$ be a path with $x_0 = (\lambda,y)$ and $x_n = (\lambda',y')$. Assume that this path has positive probability under $\psmh$ for certain value $a,b,c$, respectively of $\alpha'$, $\alpha,\beta$. Then, by definition (\ref{eq_psmh}) of $\psmh$, the probability of this path has the form $C a^k (1-a) ^{k_2} b^{k_2} c^{k_3} (1-b-c)^{k_4}$ where the $k_i$ are integers and $C$ is a constant. Then, under a difference set of values $a',b',c'$, the probability of this path has the form $C {a'}^k (1-a') ^{k_2} {b'}^{k_2} {c'}^{k_3} (1-b'-c')^{k_4}$. Since $\alpha',\alpha,\beta \in (0,1)$, this probability must also be non-zero. Hence the set of paths of positive probability are invariant with respect to the values of $\alpha',\alpha$ and $\beta$. Since $D$ is the length of the shortest path of positive probability, this proves the lemma.
\end{proof}

\begin{proof}[Proof of Lemma \ref{lemma_aperirr}]
We just need to show that, for any two clusterings $\bfA$ and $\bfA'$, only a finite number of re-arrangement steps is sufficient to obtain $\bfA'$ from $\bfA$. Let the clusters of $\bfA'$, in some order, be $A'_1,A'_2,\cdots, A'_n$. Assume that the points of $A'_i$ are spread across $A_{i_1},\cdots,A_{i_k}$ with $n_1,n_2,\cdots,n_k$ points respectively. Then, with non-zero probability $A'_i$ will be created with $n_1$ points from $A_{i_1}$ (see Appendix \ref{sec_kernelComp} for the explicit computation). And from then on, with non-zero probability (again, see the computations given) the points of $A'_i$ in $A_{i_j}$ will be added to $A'_i$. Hence with non-zero probability $A'_i$ will be created. This holds for each $A'_i$, and hence we have a non-zero probability of constructing $\bfA'$ from $\bfA$. 
\end{proof}

\section{Worst Case Quantification of the Cost Function}\label{app_worstCase}

This appendix continues from Section \ref{sec_mdpN1prev} where we derived the case cost function for a clustering by considering an average case scenario. We now derive a cost function in the worst case setting. To begin, we define the following worst case measure of homogenity of the clustering.
\begin{equation}
\eps = \max_i \eps_i
\end{equation}
The main result is as follows.
\begin{lemma}\label{lemma_worstQuant}
If EXP-3-Transfer is run with source policies derived from $\bfA$ using definition \ref{def_srcPol} with $\beta$ set as in Theorem \ref{thm_EXPBound}, we have for each $1 \leq i \leq N$,
\begin{equation}\nonumber
\expc[V^{\pi^*_i}_{N+1}] - \expc[G_{E3T}]/T \leq g(c) + \eps + \max_i 2K(i) 
\end{equation}
Here the expectation is taken over the randomization of the task drawing process, randomization in EXP-3-Transfer and $P_{N+1}$ and $R_{N+1}$ (same as in Theorem \ref{thm_EXPBound}).
\end{lemma}
\begin{proof}
By the exact same steps as in the proof of Lemma \ref{lemma_avgQuant}, we have $V^{\pi^*_k}_{N+1} - \expc[G_{E3T}]/T \leq g(c) + \eps_{k_i} + 2K(k)$. Taking the max over $k$ yields
\begin{equation}\nonumber
\expc[V^{\pi^*_k}_{N+1}] - \expc[G_{E3T}]/T \leq g(c) +  \eps + \max_i 2K(i)
\end{equation}
which completes the proof.
\qed
\end{proof}

We now discuss which of the two quantifications, worst or average case, is more appropriate. If we have reason to strongly believe that the next MDP $\mdp_{N+1}$ is chosen by nature adversarially with respect to our choice of cluster $\bfA$ -- that is, nature chooses $\mdp_{N+1}$ to maximize $\max_{A_j \in \bfA}  \max_{\mdp_i \in A_j} \eps_j+ K(i)$ --   then clearly, the worst case quantification is the correct way to evaluate a clustering. On the other hand, if we do not have any reason to believe this, then the average case might be more appropriate. For instance a consider a clustering of $1000$ MDPs that contains all the MDPs in a $5$ clusters,  such that $4$ of the clusters have diameter $<10$ while the $5^{th}$ one is sparse but wide (say $10$ elements with diameter $100$). For many domains, we would consider this a good clustering and for this situation, the average case quantification would be $\leq 0.999 \times 10 + 0.001 \times 100 \max_{i} K(i)$ $\leq 10 + \max_iK(i)$, whereas the worst case quantification would be $100 + \max_{i} K(i)$. Intuitively this seems a little too pessimistic and indeed, we also observed similar results for our experiments, in the sense that the worst case quantification failed to uncover clusters that we would intuitively consider good. Hence, for the rest of the paper we use the average case quantification to define our distance function.

\section{Computations}\label{sec_kernelComp}
Here we present the computation of the ratio $\frac{\phi[\lambda,\bfA|\lambda',\bfA']}{\phi[\lambda',\bfA'|\lambda,\bfA]}$ defined using (\ref{eq_barphi}) and constructed using $\bar{\phi}^M_Y$ defined in Section \ref{sec_srchClust}. For this section, we set $|\bfA| = N$. We have four cases to consider.\\

\noindent {\bf Case 1:}
With probability $\alpha \alpha'$, $\lambda'$ increased and $\bfA'=\bfA$. In this case, we have 
$\phi[\lambda',\bfA|\lambda,\bfA] = \alpha \alpha'$, $\phi[\lambda,\bfA|\lambda',\bfA] = \alpha (1-\alpha')$ and 
\begin{equation}\nonumber
\frac{\phi[\lambda,\bfA|\lambda',\bfA]}{\phi[\lambda',\bfA|\lambda,\bfA]} = \frac{1-\alpha'}{\alpha'}.
\end{equation}

\noindent {\bf Case 2:}
With probability $\alpha (1-\alpha')$, $\lambda'$ decreased and $\bfA'=\bfA$. In this case we have $\phi[\lambda',\bfA|\lambda,\bfA]=\alpha (1-\alpha')$, $\phi[\lambda,\bfA)|\lambda',\bfA]=\alpha \alpha'$, 
\begin{equation}\nonumber
\frac{\phi[\lambda,\bfA|\lambda',\bfA]}{\phi[\lambda',\bfA|\lambda,\bfA]} = \frac{\alpha'}{1-\alpha'}
\end{equation}

\noindent {\bf Case 3:}
With probability $1-\alpha-\beta$, $\lambda'=\lambda$ and $\bfA=\bfA'$. $\phi[\lambda,\bfA,\lambda',\bfA]=\phi[\lambda',\bfA,\lambda,\bfA]=1-\alpha-\beta$ 
\begin{equation}\nonumber
\frac{\phi[\lambda,\bfA|\lambda,\bfA]}{\phi[\lambda,\bfA|\lambda,\bfA]} = 1
\end{equation}

\noindent {\bf Case 4:} 
With probability $\beta \beta'$, $\lambda'=\lambda$ and $\bf$ is rearranged. Now the probability of moving $k_i$ points from $A_i$ to $A_j$ is, 
\begin{equation}\nonumber
P(A_i,A_j;k_i)= N^{-2} PE(k_i;|A_i|,\theta_1) { |A_i| \choose k_i }^{-1} %\binom{|A_i|}{k_i}^{-1}
\end{equation}
The reverse probability now depends on what actually has been moved. We have 4 subcases:\\

\noindent {\it Case 4.1:}
If $k_i$ points are moved between clusters $A_i$ and $A_j$ from clustering $\bf$, with $0<k_i<|A_i|$:
\begin{equation}\nonumber
 P(A_j,A_i; k_i)= N^{-2} PE(k_i;|A_j|+k_i,\theta_1) \binom{|A_j|+k_i}{k_i}^{-1}
\end{equation}
Now, we have that $\phi[\lambda,\bfA'|\lambda,\bfA]=\beta \beta' P(A_i,A_j;k_i)$ and $\phi[\lambda,\bfA|\lambda,\bfA']=\beta \beta' P(A_j, A_i;k_i)$, so that, we have:
\begin{equation}\nonumber
\frac{\phi[\lambda,\bfA|\lambda,\bfA']}{\phi[\lambda,\bfA'|\lambda,\bfA]}
 =\frac{PE(k_i;|A_j|+k_i,\theta_1){\binom{|A_i|}{k_i}}}{PE(k_i;|A_i|,\theta_1) {\binom{|A_j|+k_i}{k_i}}}
\end{equation}

\noindent {\it Case 4.2:}
If $k_i$ points are moved from cluster $A_i$ to a new cluster $A_{|A|+1}$, with $0<k_i<|A_i|$:

\begin{equation*}
 P(A_{N+1}, A_i; k_i)= (N+1)^{-2} PE(k_i;k_i,\theta_1)
\end{equation*}
note that $|A_{N+1}|=k_i$. Now, we have that $\phi[\lambda,\bfA'|\lambda,\bfA]=\beta \beta' P(A_i, A_{N+1}; k_i)$ and $\phi[\lambda,\bfA|\lambda,\bfA']=\beta \beta' P(A_{N+1},A_i; k_i)$. So the desired ratio is:
\begin{equation}\nonumber
\frac{\phi[\lambda,\bfA|\lambda,\bfA']}{\phi[\lambda,\bfA'|\lambda,\bfA]} = \frac{N^2 PE(k_i;k_i,\theta_1) {\binom{|A_i|}{k_i}}}{({N+1})^2 PE(k_i;|A_i|,\theta_1) }
\end{equation}

\noindent {\it Case 4.3:}
If $|A_i|$ points are moved from cluster $A_i$ to existing cluster $A_j$, now we have one less cluster so that, 
\begin{equation*}
P(A_j, A_i: |A_i|)= (N-1)^{-2} PE(|A_i|;|A_j|+|A_i|,\theta_1)  \binom{|A_i|+|A_j|}{|A_i|}^{-1}
\end{equation*}
The $\phi$ values are: $\phi[\lambda,\bfA'|\lambda,\bfA]=\beta \beta' P(A_i \xrightarrow{|A_i|} A_j)$ and $\phi[\lambda,\bfA|\lambda,\bfA']=\beta \beta' P(A_j \xrightarrow{|A_i|} A_i)$. Together, this gives us the ratio:
\begin{equation}\nonumber
\frac{\phi[\lambda,\bfA|\lambda,\bfA']}{\phi[\lambda,\bfA'|\lambda,\bfA]} = \frac
{N^2 PE(|A_i|;|A_j|+|A_i|,\theta_1)  }
{(N-1)^2 PE(|A_i|;|A_i|,\theta_1) {\binom{|A_i|+|A_j|}{|A_i|}} }
\end{equation}

\noindent {\it Case 4.4:}
If $|A_i|$ points are moved from cluster $A_i$ to a new cluster $A_{|A|+1}$:
\begin{equation}\nonumber
P(A_{N+1}, A_i: |A_i|)= N^{-2} PE(|A_i|;|A_i|,\theta_1) 
\end{equation}
The clustering $\bfA$ does not change in this case and the $\phi$ values are:
$\phi[\lambda,\bfA'|\lambda,\bfA]=\beta \beta' P(A_i, A_{N+1};|A_i|)$, 
$\phi[\lambda,\bfA|\lambda,\bfA']=\beta \beta' P(A_{N+1},A_i; |A_i|)$, which gives us
\begin{equation}
\frac{\phi[\lambda,\bfA|\lambda,\bfA']}{\phi[\lambda,\bfA'|\lambda,\bfA]} = 1
\end{equation}

%\noindent {\bf Case 5:}
%With probability $\beta (1-\beta')$, $\lambda'=\lambda$ and $A$ is %reconstructed.

\section{Surveillance Domain Experiments: Algorithm Comparisons}\label{app_ful_alg_res}

In this section we give detailed cumulative reward curves for the $4$ algorithms: E3T with clustering, Policy-Reuse with clustering and Policy-Reuse with clustering. The results are given in Figures \ref{fig_alg_res_1} to \ref{fig_alg_res_3}. The results more or less show what the summary graphs showed. In particular, when the number of previous tasks and the complexity of task is low, Policy-Reuse is better than our algorithm. However, as the complexity keeps increasing, our algorithm begins to dominate both versions of Policy-Reuse, showing that clustering is beneficial.

%\begin{figure}[h]
%    \centering
%\begin{tabular}{cc}
%      \subfloat[My Sub-figure 1.]{\label{fig:sf1}
%     \includegraphics[scale=0.6]{figures/results/64_200_1/Averaged_results_alg.png}} \\
% \subfloat[My Sub-figure 2.]{\label{fig:sf2}
%          \includegraphics[scale=0.6]{figures/results/96_200_2/Averaged_results_alg.png}} 
%     \end{tabular}
%\caption{My Figures.}
%  \label{fig:myfigs}
%\end{figure}

\begin{figure}[h]
    \centering
      \subfloat[]{\label{fig:sf1}
     \includegraphics[width=0.47\textwidth]{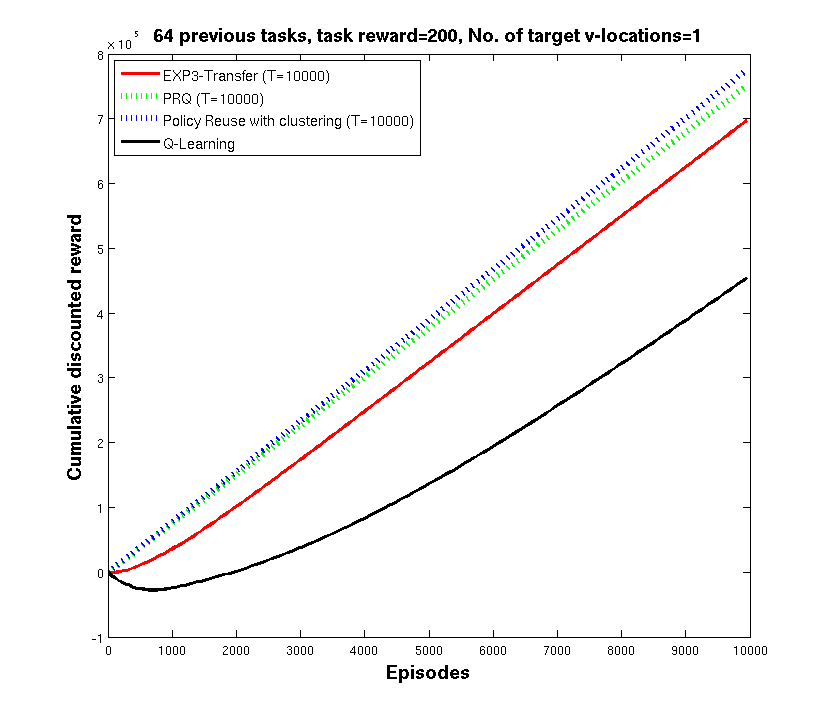}} \quad
 \subfloat[]{\label{fig:sf2}
       \includegraphics[width=0.47\textwidth]{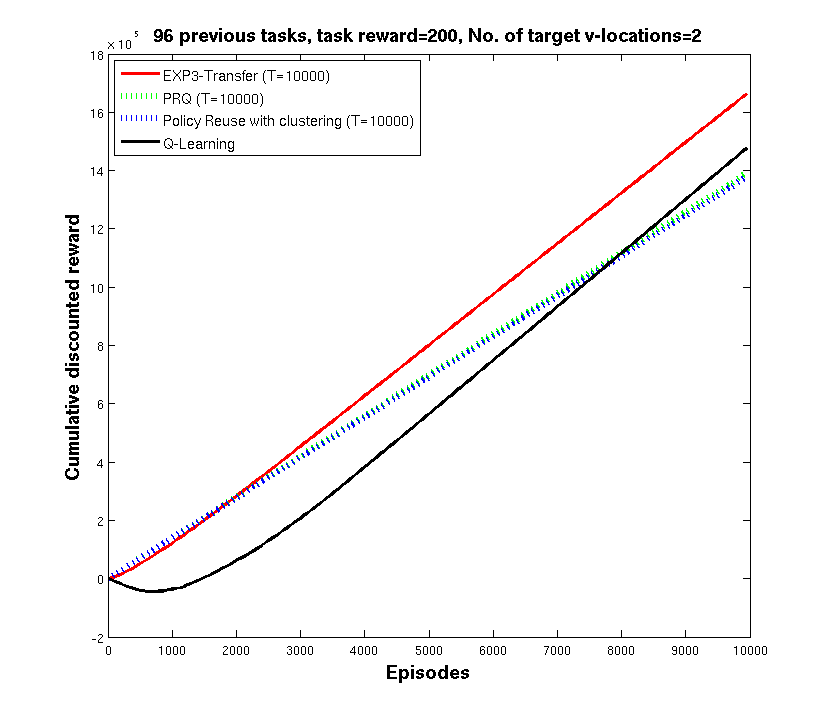}} 
\caption{{\sc Algorithm Comparisons.} These figures compares the performance of EXP-3-Transfer with clustering, Policy-Reuse with and without clustering, and Q-learning for various settings of the task (see the figure title). These are the detailed plots of the summary results presented in Section \ref{sec_expClustGain}. }
  \label{fig_alg_res_1}
\end{figure}

\begin{figure}[h]
    \centering
      \subfloat[ ]{\label{fig:sf1}
        \includegraphics[width=0.47\textwidth]{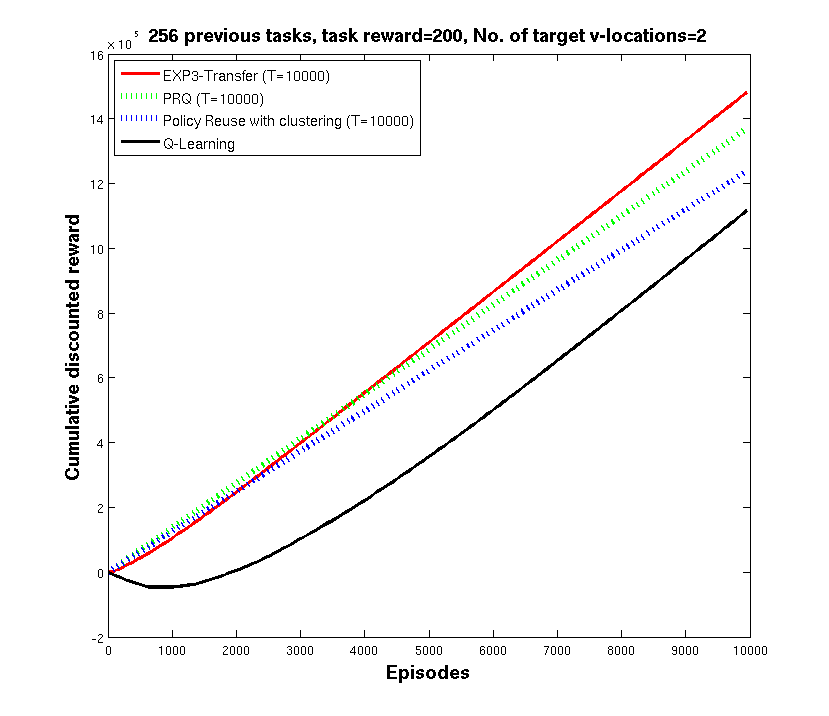}} \quad
 \subfloat[ ]{\label{fig:sf2}
       \includegraphics[width=0.47\textwidth]{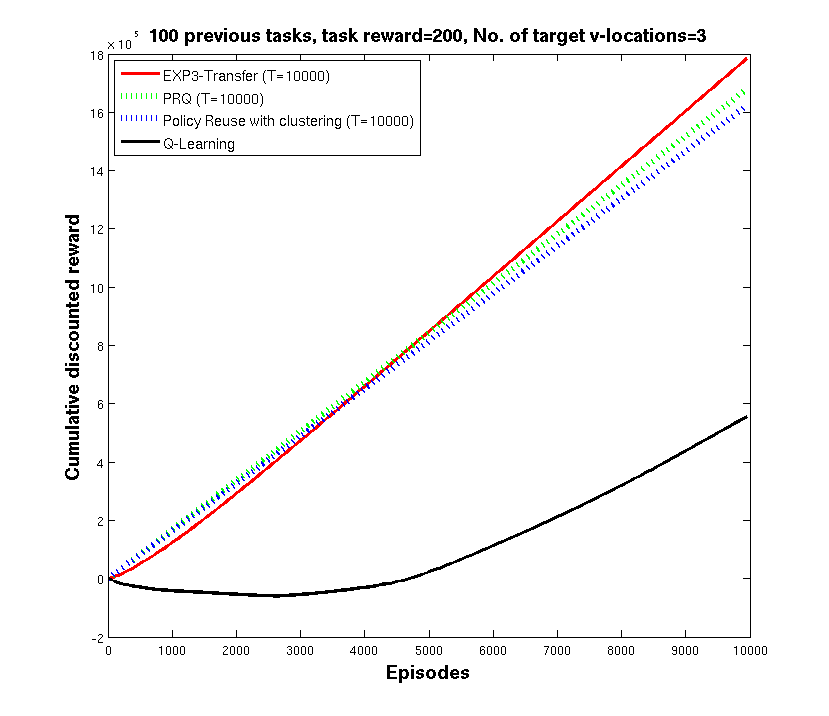}} 
\caption{{\sc Algorithm Comparisons continued.} These figures compares the performance of EXP-3-Transfer with clustering, Policy-Reuse with and without clustering, and Q-learning for various settings of the task (see the figure title). These are the detailed plots of the summary results presented in Section \ref{sec_expClustGain}. }
  \label{fig_alg_res_2}
\end{figure}

\begin{figure}[h]
    \centering
      \subfloat[ ]{\label{fig:sf1}
      \includegraphics[width=0.47\textwidth]{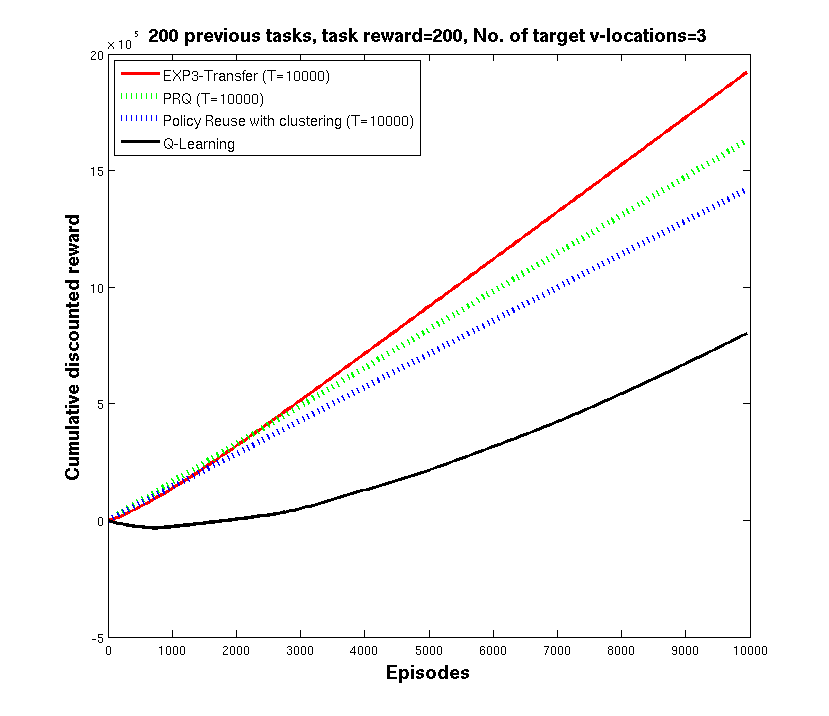}} \quad
 \subfloat[ ]{\label{fig:sf2}
         \includegraphics[width=0.47\textwidth]{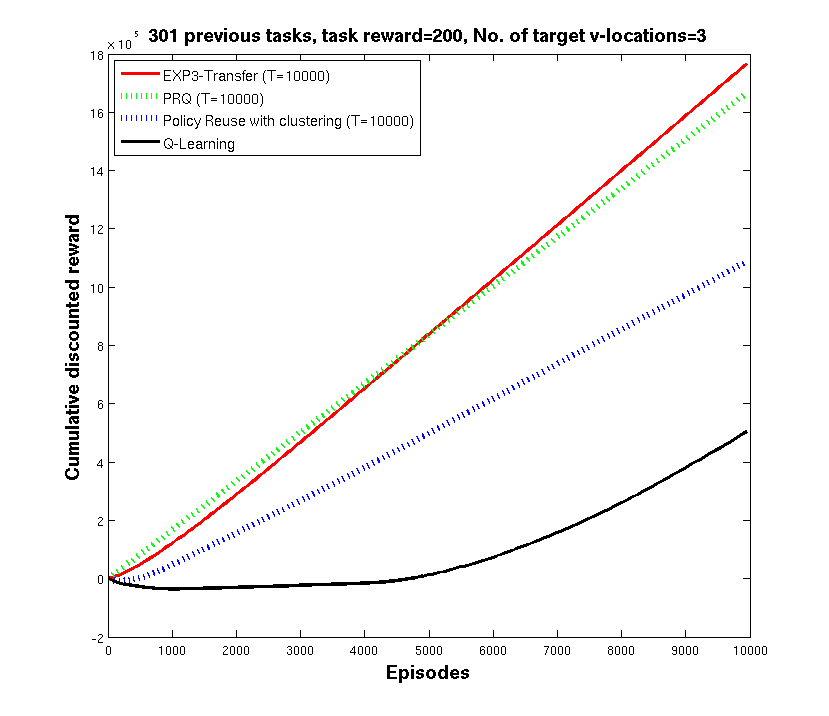}} 
\caption{{\sc Algorithm Comparisons continued.} These figures compares the performance of EXP-3-Transfer with clustering, Policy-Reuse with and without clustering, and Q-learning for various settings of the task (see the figure title). These are the detailed plots of the summary results presented in Section \ref{sec_expClustGain}. }
  \label{fig_alg_res_3}
\end{figure}

%\begin{figure}[h]
%\begin{center}
%\includegraphics[height=3.5in,width=\textwidth]{figures/results/200_200_3/Averaged_results_alg.png}
%\end{center}
%\caption{Comparison of algorithms when clustering is used to encode previous MDPs into a smaller set. The experiment parameters are given in the title. This shows that EXP-3-Transfer and Policy Reuse (Policy-Reuse) finds a good policy from the previous set and quickly learns to use it. EXP-3-Transfer outperforms Policy Reuse, and they both significantly outperform Q-learning.}\label{fig_alg_5}
%\end{figure}
%
%
%\begin{figure}[h]
%\begin{center}
%\includegraphics[height=3.5in,width=\textwidth]{figures/results/301_200_3/Averaged_results_alg.png}
%\end{center}
%\caption{Comparison of algorithms when clustering is used to encode previous MDPs into a smaller set. The experiment parameters are given in the title. This shows that EXP-3-Transfer and Policy Reuse (Policy-Reuse) finds a good policy from the previous set and quickly learns to use it. EXP-3-Transfer outperforms Policy Reuse, and they both significantly outperform Q-learning.}\label{fig_alg_6}
%\end{figure}

\clearpage

\subsection{Surveillance Domain: Clustering Comparisons}

In this section in figures \ref{fig_clustComp_1} to \ref{fig_clustComp_3} we present the learning curves summarized in figure \ref{fig_clustComp_nostar}. The general trend follows what was observed in Section \ref{sec_expClustComp}.

\begin{figure}[h]
\centering
 \subfloat[ ]{\label{fig:sf1}
    \includegraphics[width=0.47\textwidth]{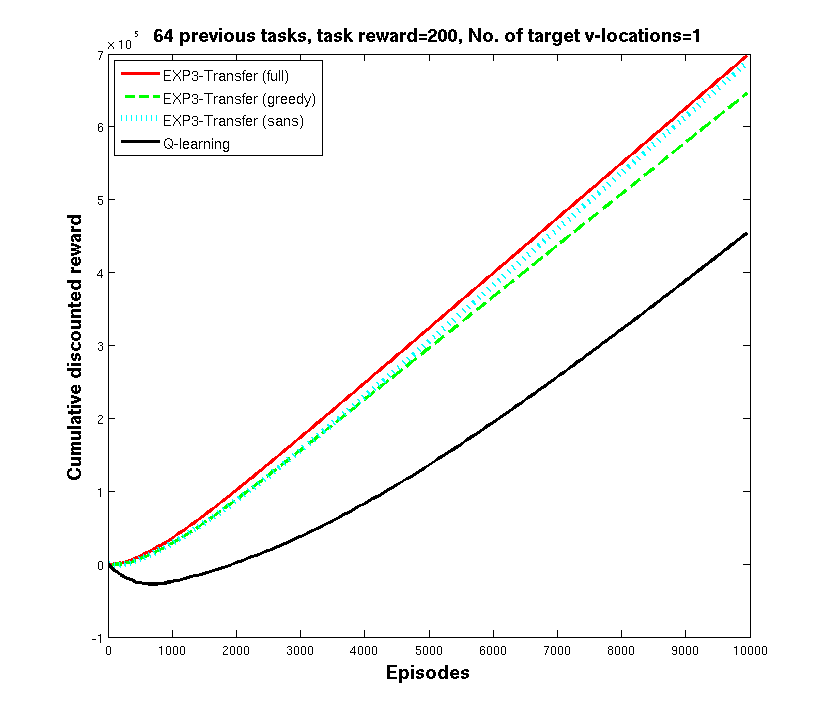}}
 \subfloat[ ]{\label{fig:sf1}
\includegraphics[width=0.47\textwidth]{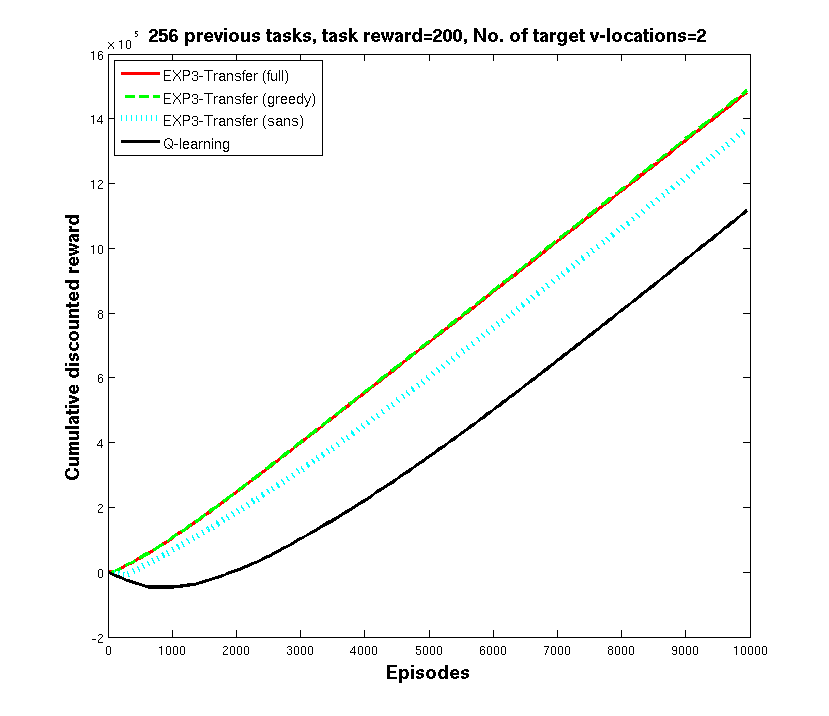}}
\caption{{\sc Clustering comparisions extended Results}. These figures show the results that are summarized in Figure \ref{fig_clustComp_nostar}. The title of the graphs describe the experiment setup.}\label{fig_clustComp_1}
\end{figure}

%\begin{figure}[h]
%\begin{center}
%\includegraphics[height=3.5in,width=\textwidth]{figures/results/256_200_2/Averaged_results_cls.png}
%\end{center}
%\caption{The Surveillance Domain. The domain is $48\times 48$ gridworld with $64$ possible surveillance locations (v-locations) marked in green. Each MDP in the  domain requires the agent to surveil $i$ different locations, $i \in \{1,2,3,4\}$ in  a particular sequence to receive positive reward of $200$. Each step gives a reward of $-1$.}\label{fig_survDom}
%\end{figure}
%

%\begin{figure}[h]
%\begin{center}
%\includegraphics[height=3.5in,width=\textwidth]{figures/results/100_200_3/Averaged_results_cls.png}
%\end{center}
%\caption{The Surveillance Domain. The domain is $48\times 48$ gridworld with $64$ possible surveillance locations (v-locations) marked in green. Each MDP in the  domain requires the agent to surveil $i$ different locations, $i \in \{1,2,3,4\}$ in  a particular sequence to receive positive reward of $200$. Each step gives a reward of $-1$.}\label{fig_survDom}
%\end{figure}
%\clearpage 
%
%
%\begin{figure}[h]
%\begin{center}
%\includegraphics[height=3.5in,width=\textwidth]{figures/results/200_200_3/Averaged_results_cls.png}
%\caption{The Surveillance Domain. The domain is $48\times 48$ gridworld with $64$ possible surveillance locations (v-locations) marked in green. Each MDP in the  domain requires the agent to surveil $i$ different locations, $i \in \{1,2,3,4\}$ in  a particular sequence to receive positive reward of $200$. Each step gives a reward of $-1$.}\label{fig_survDom}
%\end{center}
%\end{figure}

\begin{figure}[h]
\centering
 \subfloat[ ]{\label{fig:sf1}
      \includegraphics[width=0.47\textwidth]{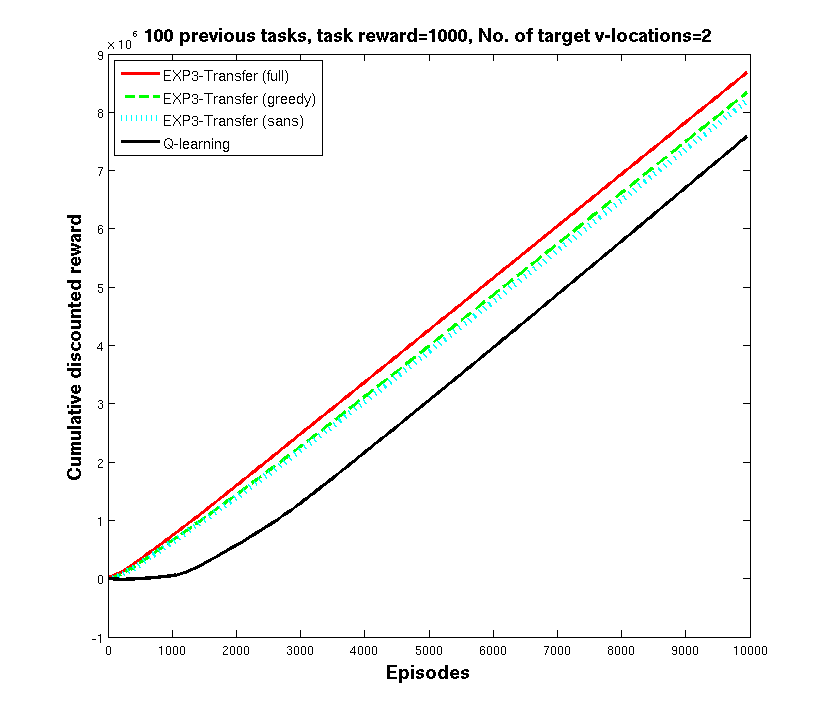}}
 \subfloat[ ]{\label{fig:sf1}
     \includegraphics[width=0.47\textwidth]{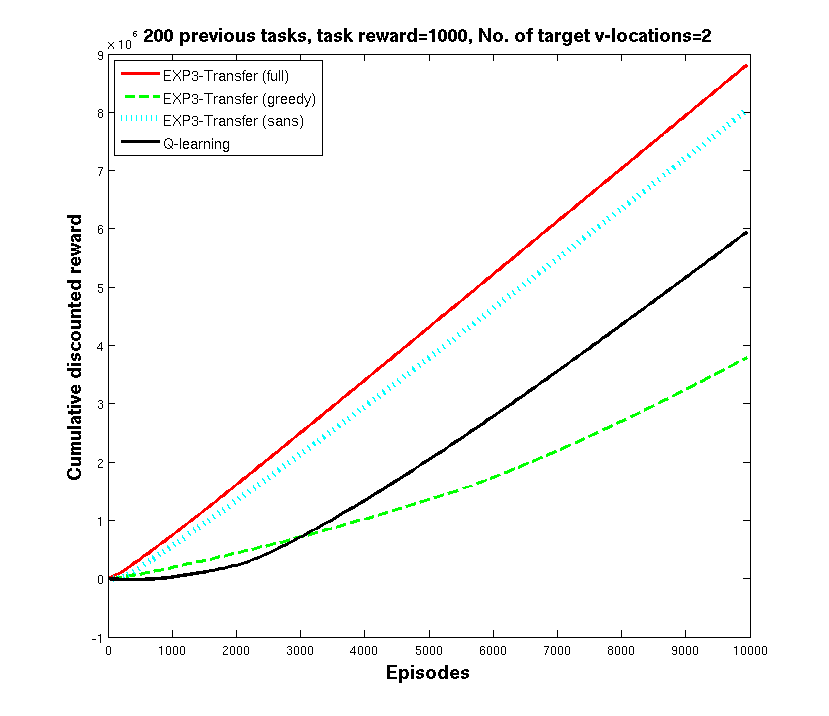}}
\caption{{\sc Clustering comparisions extended results continued.} These figures show the results that are summarized in Figure \ref{fig_clustComp_nostar}. The title of the graphs describe the experiment setup.}\label{fig_clustComp_2}
\end{figure}

\begin{figure}[h]
\begin{center}
\includegraphics[width=0.47\textwidth]{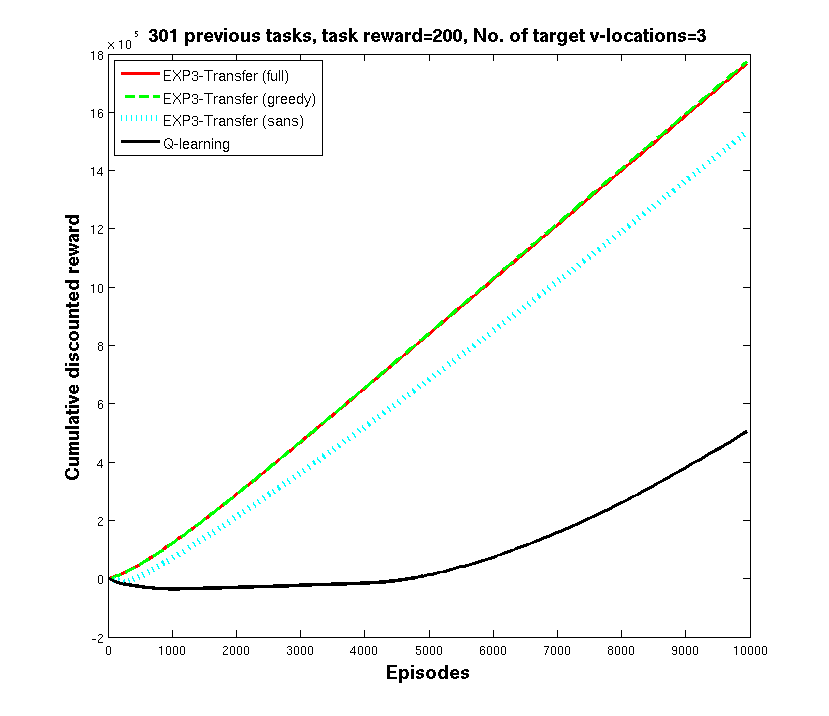}
\end{center}
\caption{{\sc Clustering comparisions extended results continued.} These figures show the results that are summarized in Figure \ref{fig_clustComp_nostar}. The title of the graphs describe the experiment setup.}\label{fig_clustComp_3}

\end{figure}

\subsection{Time Comparisons}\label{app_timeComp}

Figures \ref{fig_timeComp1} to \ref{fig_timeComp3} gives the time comparison results for transfer problems not described in Figure \ref{fig_time_all}.

%\begin{figure}[h]
%\begin{center}
%\includegraphics[height=3.5in,width=\textwidth]{figures/results/96_200_2/Averaged_results_time.png}
%\end{center}
%\caption{Learning curve of EXP-3-Transfer  when run for different number of time steps (parameter $T$). This affects both the clustering and the arms chosen by EXP-3-Transfer. The parmaters for the experiments are given in the title fo the figure. As the figure shows, for shorter $T$, the EXP-3-Transfer run with the lowest $T = 1000$ is optimal. For the remaining part, $T=5000$ is optimal, and for the remaining time $T=10,000$ is optimal.}\label{fig_survDom}
%\end{figure}
%
%
%\begin{figure}[h]
%\begin{center}
%\includegraphics[height=3.5in,width=\textwidth]{figures/results/256_200_2/Averaged_results_time.png}
%\end{center}
%\caption{Learning curve of EXP-3-Transfer  when run for different number of time steps (parameter $T$). This affects both the clustering and the arms chosen by EXP-3-Transfer. The parmaters for the experiments are given in the title fo the figure. As the figure shows, for shorter $T$, the EXP-3-Transfer run with the lowest $T = 1000$ is optimal. For the remaining part, $T=5000$ is optimal, and for the remaining time $T=10,000$ is optimal.}\label{fig_survDom}
%\end{figure}

\begin{figure}[h]
\centering
 \subfloat[ ]{\label{fig:sf1}
     \includegraphics[width=0.47\textwidth]{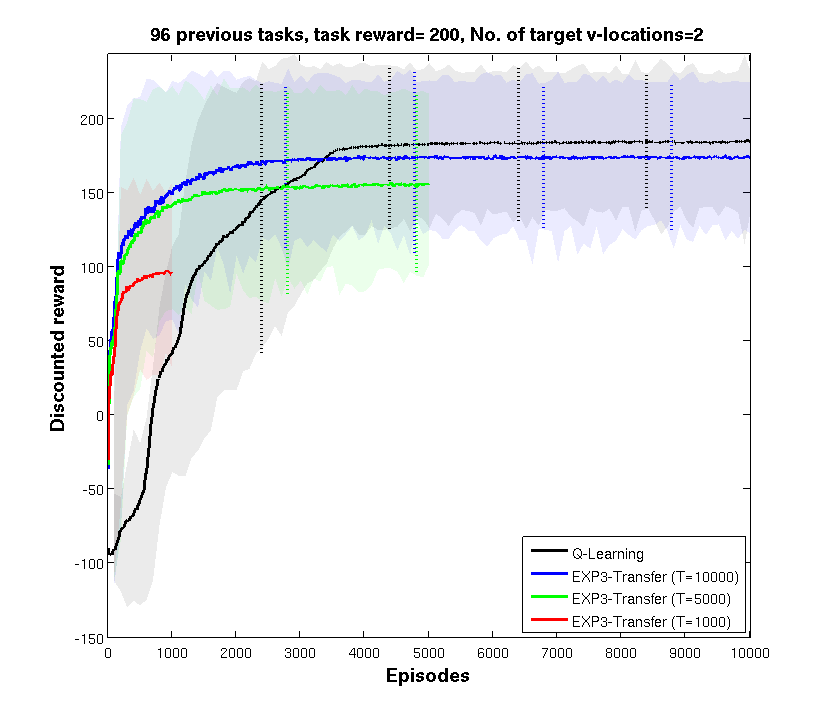}}
 \subfloat[ ]{\label{fig:sf1}
      \includegraphics[width=0.47\textwidth]{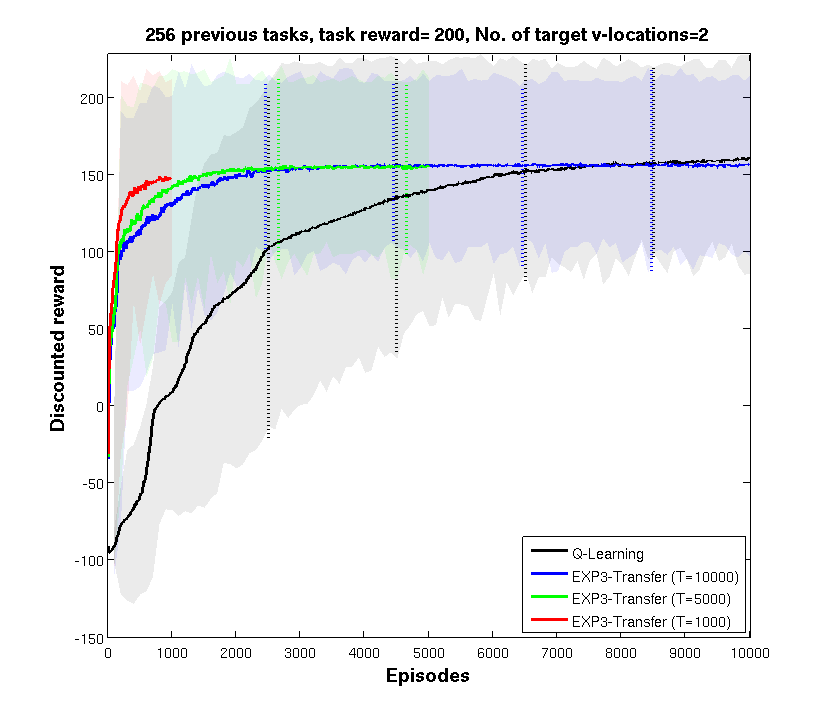}}
\caption{{\sc Time comparisions extended results.} These figures show time comparsion results for transfer tasks in addition to Figure \ref{fig_time_all}. The title of the graphs show the experiment setup. The shaded areas give the standard deviation for the learning curves.}\label{fig_timeComp1}
\end{figure}

%\begin{figure}[h]
%\begin{center}
%\includegraphics[height=3.5in,width=\textwidth]{figures/results/100_200_3/Averaged_results_time.png}
%\end{center}
%\caption{Learning curve of EXP-3-Transfer  when run for different number of time steps (parameter $T$). This affects both the clustering and the arms chosen by EXP-3-Transfer. The parmaters for the experiments are given in the title fo the figure. As the figure shows, for shorter $T$, the EXP-3-Transfer run with the lowest $T = 1000$ is optimal. For the remaining part, $T=5000$ is optimal, and for the remaining time $T=10,000$ is optimal.}\label{fig_survDom}
%\end{figure}
%
%
%\begin{figure}[h]
%\begin{center}
%\includegraphics[height=3.5in,width=\textwidth]{figures/results/200_200_3/Averaged_results_time.png}
%\end{center}
%\caption{Learning curve of EXP-3-Transfer  when run for different number of time steps (parameter $T$). This affects both the clustering and the arms chosen by EXP-3-Transfer. The parmaters for the experiments are given in the title fo the figure. As the figure shows, for shorter $T$, the EXP-3-Transfer run with the lowest $T = 1000$ is optimal. For the remaining part, $T=5000$ is optimal, and for the remaining time $T=10,000$ is optimal.}\label{fig_survDom}
%\end{figure}

\begin{figure}[h]
\centering
 \subfloat[ ]{\label{fig:sf1}
     \includegraphics[width=0.47\textwidth]{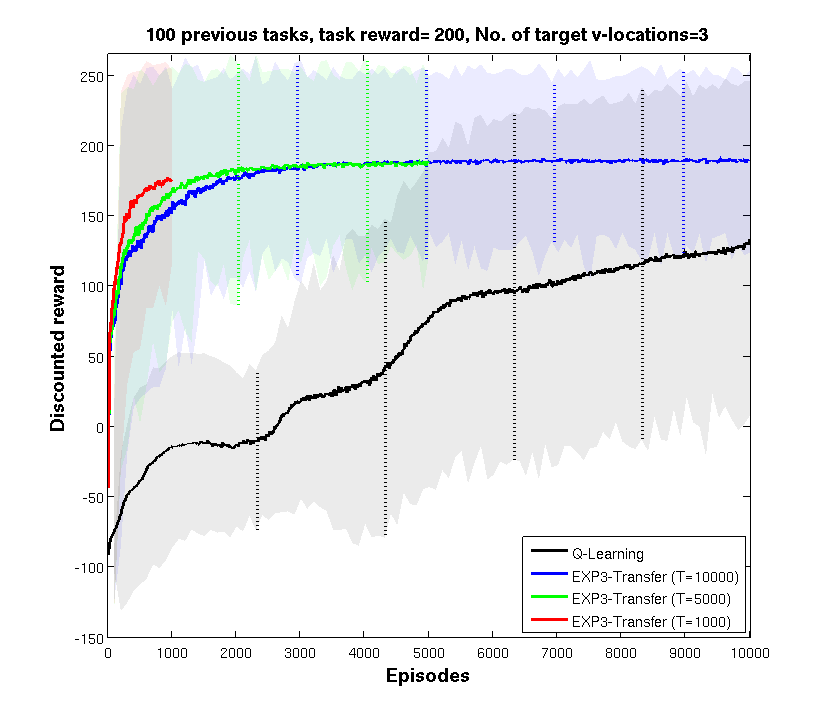}}
 \subfloat[ ]{\label{fig:sf1}
    \includegraphics[width=0.47\textwidth]{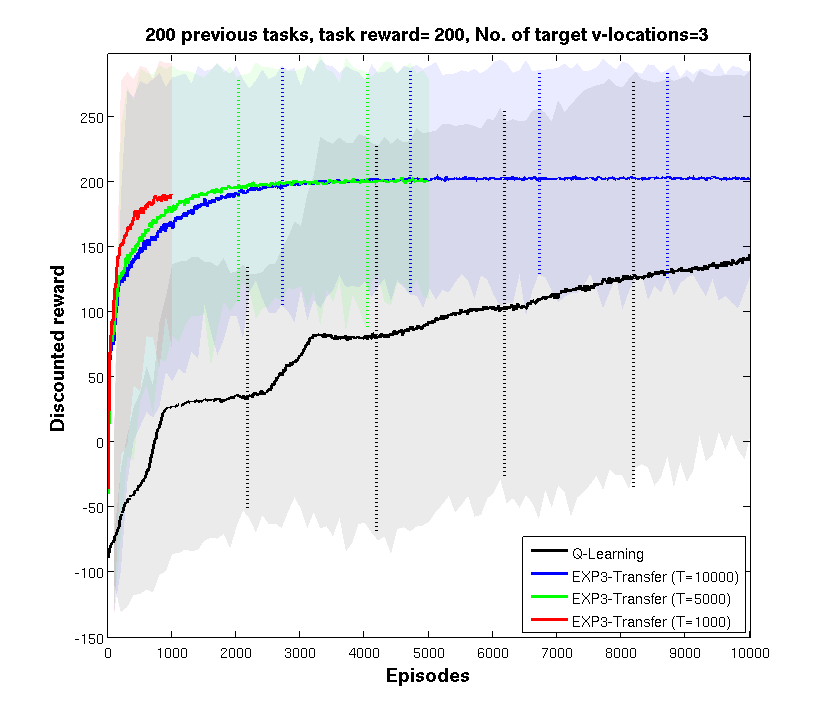}}
\caption{{\sc Time comparisions extended continued.} These figures show time comparsion results for transfer tasks in addition to Figure \ref{fig_time_all}. The title of the graphs show the experiment setup. The shaded areas give the standard deviation for the learning curves.}\label{fig_timeComp2}
\end{figure}

\begin{figure}[h]
\begin{center}
\includegraphics[width=0.47\textwidth]{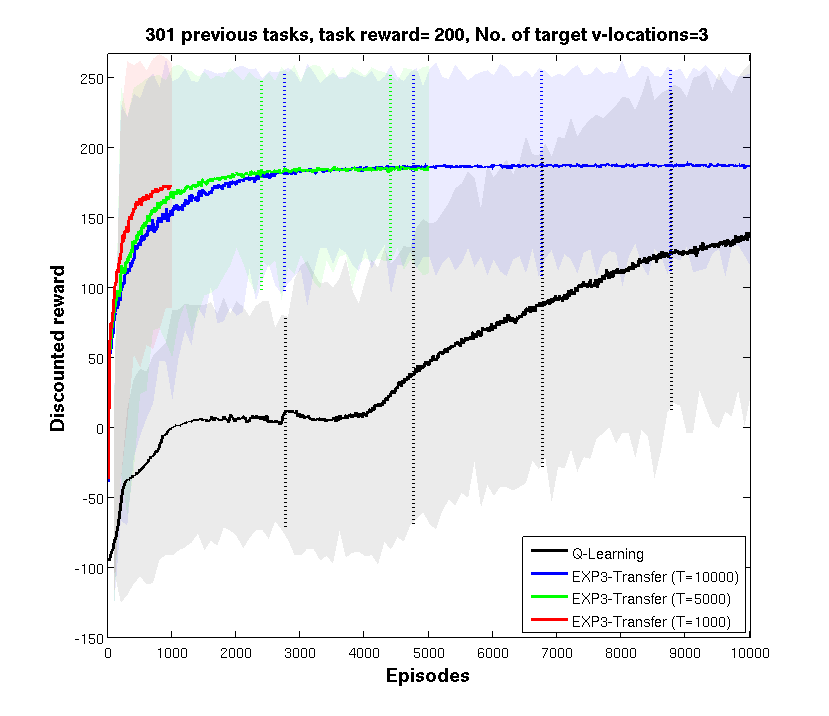}
\end{center}
\caption{{\sc Time comparisions extended continued.} These figures show time comparsion results for transfer tasks in addition to Figure \ref{fig_time_all}. The title of the graphs show the experiment setup. The shaded areas give the standard deviation for the learning curves.}\label{fig_timeComp3}
\end{figure}

\clearpage

\bibliography{hassan}

\begin{thebibliography}{45}
\providecommand{\natexlab}[1]{#1}
\providecommand{\url}[1]{\texttt{#1}}
\expandafter\ifx\csname urlstyle\endcsname\relax
  \providecommand{\doi}[1]{doi: #1}\else
  \providecommand{\doi}{doi: \begingroup \urlstyle{rm}\Url}\fi

\bibitem[Ammar et~al.(2012)Ammar, Tuyls, Taylor, Driessen, and
  Weiss]{ammar_tuyls_taylor_driessen_weiss:2012}
Haitham~Bou Ammar, Karl Tuyls, Matthew~E. Taylor, Kurt Driessen, and Gerhard
  Weiss.
\newblock Reinforcement learning transfer via sparse coding.
\newblock In \emph{Proceedings of International Conference on Autonomous Agents
  and Multiagent Systems}, 2012.

\bibitem[An et~al.(2012)An, Kempe, Kiekintveld, Shieh, Singh, Tambe, and
  Vorobeychik]{an_kempe_kiekintveld_shieh_singh_tambe_vorobeychik:2012}
Bo~An, David Kempe, Christopher Kiekintveld, Eric Shieh, Satinder Singh, Milind
  Tambe, and Yevgeniy Vorobeychik.
\newblock Security games with limited surveillance.
\newblock In \emph{Proceedings of the 26th Conference on Artificial
  Intelligence}, 2012.

\bibitem[Auer et~al.(2002{\natexlab{a}})Auer, Cesa-Bianchi, and
  Fischer]{auer_cesa-bianchi_fischer:2002}
Peter Auer, Nicolo Cesa-Bianchi, and Paul Fischer.
\newblock Finite-time analysis of the multiarmed bandit problem.
\newblock \emph{Machine Learning}, 47:\penalty0 235--256, 2002{\natexlab{a}}.

\bibitem[Auer et~al.(2002{\natexlab{b}})Auer, Cesa-Bianchi, Freund, and
  Schapire]{auer_cesa-bianchi_freund_schapire:2002}
Peter Auer, Nicolo Cesa-Bianchi, Yoav Freund, and Robert~E. Schapire.
\newblock The nonstochastic multiarmed bandit problem.
\newblock \emph{{SIAM} Journal on Computing}, 32:\penalty0 48--77,
  2002{\natexlab{b}}.

\bibitem[Azar et~al.(2013)Azar, Lazaric, and
  Brunskill]{azer_lazaric_brunskill:2013}
Mohammad~Gheshlaghi Azar, Alessandro Lazaric, and Emma Brunskill.
\newblock Regret bounds for reinforcement learning with policy advice.
\newblock In \emph{Proceedings of $23^{rd}$ European Conference on Machine
  learning ({ECML})}, 2013.

\bibitem[Baxter(2000)]{baxter:2000}
Jonathan Baxter.
\newblock A model of inductive bias learning.
\newblock \emph{Journal of Artificial Intelligence Research}, 12:\penalty0
  149--198, March 2000.

\bibitem[Carroll and Seppi(2005)]{carroll_seppi:2005}
James~L. Carroll and Kevin Seppi.
\newblock Task similarity measures for transfer in reinforcement learning task
  libraries.
\newblock In \emph{Proceedings of 2005 IEEE International Joint Conference on
  Neural Networks}, 2005.

\bibitem[Castro and Precup(2010)]{castro_precup:2010}
Pablo~Samuel Castro and Doina Precup.
\newblock Using bisimulation for policy transfer in mdps.
\newblock In \emph{Proceedings of the the $24^{th}$ {AAAI} Conference on
  Artificial Intelligence}, 2010.

\bibitem[Cesa-Bianchi and Lugosi(2006)]{cesa-bianchi_lugosi:2006}
Nicolo Cesa-Bianchi and Gabor Lugosi.
\newblock \emph{Prediction Learning and Games}.
\newblock Cambridge University Press, 2006.

\bibitem[Chickering(2003)]{chickering:2003}
David Chickering.
\newblock Optimal structure identification with greedy search.
\newblock \emph{Journal of Machine Learning Research}, 3:\penalty0 507--554,
  2003.

\bibitem[Chickering and Boutilier(2002)]{chickering_boutlier:2002}
David~Maxwell Chickering and Craig Boutilier.
\newblock Learning equivalence classes of bayesian-network structures.
\newblock \emph{Journal of Machine Learning Research}, 2:\penalty0 150--157,
  2002.

\bibitem[Dubhashi and Panconesi(2009)]{dubhashi_panconesi:2009}
Devdatt~P. Dubhashi and Alessandro Panconesi.
\newblock \emph{Concentration of Measure Inequalities for the Analysis of
  Randomized Algorithms}.
\newblock Cambridge University Press, 2009.

\bibitem[Fernandez et~al.(2006)Fernandez, Garcia, and
  Veloso]{fernandez_veloso:2006}
Fernando Fernandez, Javier Garcia, and Manuela Veloso.
\newblock Probabilistic policy reuse in a reinforcement learning agent.
\newblock In \emph{Proceedings of the 5th International Conference on
  Autonomous Agents and Multiagent Systems}, 2006.

\bibitem[Fernandez et~al.(2010)Fernandez, Garcia, and
  Veloso]{fernandez_garcia_veloso:2010}
Fernando Fernandez, Javier Garcia, and Manuela Veloso.
\newblock Probabilistic policy reuse for inter-task transfer learning.
\newblock \emph{Robotics and Autonomous Systems}, 58:\penalty0 866--871, 2010.

\bibitem[Ferns et~al.(2004)Ferns, Panangaden, and
  Precup]{ferns_panangaden_precup:2004}
Norm Ferns, Prakash Panangaden, and Doina Precup.
\newblock Metrics for finite markov decision processes.
\newblock In \emph{Proceedings of the Conference on Uncertainty in Artificial
  Intelligence}, 2004.

\bibitem[Ferrante et~al.(2008)Ferrante, Lazaric, and
  Restelli]{ferrante_lazaric_restelli:2008}
Eliseo Ferrante, Alessandro Lazaric, and Marcello Restelli.
\newblock Transfer of task representation in reinforcement learning using
  policy-based protovalue functions.
\newblock In \emph{Proceedings of the $7^th$ International Confernce on
  Autonomous Agent And Multiagent Systems}, 2008.

\bibitem[Hauser et~al.(2009)Hauser, Urban, Liberali, , and
  Braun]{hauser_urban_liberali_braun:2009}
John~R. Hauser, Glen~L. Urban, Guilherme Liberali, , and Michael Braun.
\newblock Website morphing.
\newblock \emph{Marketing Science}, 28\penalty0 (2):\penalty0 202--223, March
  2009.

\bibitem[Heckerman and Chickering(1995)]{heckerman_chickering:1995}
David Heckerman and David~M. Chickering.
\newblock Learning bayesian networks: The combination of knowledge and
  statistical data.
\newblock \emph{Machine Learning}, pages 20--197, 1995.

\bibitem[Karp(1972)]{karp:1972}
Richard Karp.
\newblock Reducibility among combinatorial problems.
\newblock In \emph{Proceedings of a Symposium on the Complexity of Computer
  Computations}, 1972.

\bibitem[Kirkpatrick et~al.(1983)Kirkpatrick, Gelatt, and
  Vecchi]{kirkpatrick_gelatt_vecchi:1983}
S.~Kirkpatrick, C.~D. Gelatt, and M.~P. Vecchi.
\newblock Optimization by simulated annealing.
\newblock \emph{Science}, 220:\penalty0 671--680, 1983.

\bibitem[Kolmogorov and Fomin(1970)]{kolmogorov_fomin:1970}
Andrey Kolmogorov and Sergei~V. Fomin.
\newblock \emph{Introductory Real Analysis}.
\newblock Dover Publications, 1970.

\bibitem[Konidaris and Barto(2007)]{konidaris_barto:2007}
George Konidaris and Andrew~G. Barto.
\newblock Building portable options: skill transfer in reinforcement learning.
\newblock In \emph{Proceedings of the 20th International Joint Conference on
  Artificial Intelligence}, 2007.

\bibitem[Lazaric and Restilli(2011)]{lazaric_restilli:2011}
Allesandro Lazaric and Marcello Restilli.
\newblock Transferring from multiple mdps.
\newblock In \emph{Proceedings of the Neural Information Processing Systems
  Conference}, 2011.

\bibitem[Lazaric et~al.(2008)Lazaric, Restelli, and
  Bonarini]{lazaric_restlli_bonarni:2008}
Allesandro Lazaric, Marcello Restelli, and Andrea Bonarini.
\newblock Transfer of samples in batch reinforcement learning.
\newblock In \emph{Proceedings of the $25^{th}$ International Conference on
  Machine Learning}, 2008.

\bibitem[Levin et~al.(2009)Levin, Peres, and Wilmer]{levin_peres_wilmer:2009}
David~A. Levin, Yuval Peres, and Elizabeth~L. Wilmer.
\newblock \emph{Markov Chains and Mixing Times}.
\newblock American Mathematical Society, 2009.

\bibitem[Littlestone and Warmuth(1994)]{littlestone_warmuth:1994}
Nick Littlestone and Manfred~K. Warmuth.
\newblock The weighted majority algorithm.
\newblock \emph{Information and Computation}, 108(2):\penalty0 212--261, 1994.

\bibitem[Locatelli(2000)]{locatelli:2000}
M.~Locatelli.
\newblock Simulated annealing, algorithms for continuous global optimization:
  convergence conditions.
\newblock \emph{Journal of Optimization Theory and Applications},
  104(1):\penalty0 121--133, 2000.

\bibitem[Mahadevan(2005)]{mahadevan:2005}
Sridhar Mahadevan.
\newblock Proto-value functions: Developmental reinforcement learning.
\newblock In \emph{Proceedings of International Conference on Machine
  Learning}, 2005.

\bibitem[Mitchell and Thrun(1993)]{mitchell_thrun:1993}
Tom~M. Mitchell and Sebastian Thrun.
\newblock Explanation-based neural network learning for robot control.
\newblock In \emph{Adavances in Neural Information Processing Systems}, pages
  287--294, San Mateo, CA, 1993. Morgan Kaufmann Press.

\bibitem[Nilim and Ghaoui(2005)]{nilim_el_ghaoui:2005}
Arnab Nilim and Laurent~El Ghaoui.
\newblock Robust control of markov decision processes with uncertain transition
  matrices.
\newblock \emph{Operations Research}, 53(5):\penalty0 780--798, 2005.

\bibitem[Puterman(1994)]{puterman:1994}
Martin~L. Puterman.
\newblock \emph{Markov Decision Processes: Discrete Stochastic Dynamic
  Programming}.
\newblock John Wiley and Sons, 1994.

\bibitem[Ravindran(2013)]{ravindran:2013}
Balaraman Ravindran.
\newblock Relativized hierarchical decomposition of markov decision processes.
\newblock \emph{Decision making: neural and behavioural approaches},
  42:\penalty0 465--488, 2013.

\bibitem[Ravindran and Barto(2003)]{ravindran_barto:2003}
Balaraman Ravindran and Andrew~G. Barto.
\newblock {SMDP} homomorphisms: An algebraic approach to astraction in
  semi-markov decision processes.
\newblock In \emph{Proceedings of the International Joint Conference on
  Artificial Intelligence}, 2003.

\bibitem[Robert and Casella(2005)]{robert_casella:2005}
Christian~P. Robert and George Casella.
\newblock \emph{Monte Carlo Statistical Methods}.
\newblock Springer, Berling, 2005.

\bibitem[Sorg and Singh(2009)]{sorg_singh:2009}
Jonathan Sorg and Satinder Singh.
\newblock Transfer via soft homomorphisms.
\newblock In \emph{Proceedings of the $8^th$ International Conference on
  Autonomous Agents and Multiagent Systems}, 2009.

\bibitem[Strehl and Littman(2008)]{strehl_littman:2008}
Alexander~L. Strehl and Michael~L. Littman.
\newblock An analysis of model-based interval estimation for markov decision
  processes.
\newblock \emph{Journal of Computer Systems Science}, 74(8):\penalty0
  1309--1331, 2008.

\bibitem[Sutton and Barto(1998)]{sutton_barto:1998}
Richard~S. Sutton and Andrew~G. Barto.
\newblock \emph{Reinforcement Learning: An Introduction}.
\newblock MIT Press, Cambridge, MA, 1998.

\bibitem[Talvitie and Singh(2007)]{talvitie_singh:2007}
Erik Talvitie and Satinder Singh.
\newblock An experts algorithm for transfer learning.
\newblock In \emph{Proceedings of the $20^{th}$ International Joint Conference
  on Artifical Intelligence}, 2007.

\bibitem[Taylor and Stone(2009)]{taylor_stone:2009}
Matthew Taylor and Peter Stone.
\newblock Transfer learning for reinforcement learning domains: A survey.
\newblock \emph{Journal of Machine Learning Research}, 10:\penalty0 1633--1685,
  2009.

\bibitem[Thrun(1996)]{thrun:1996a}
Sebastian Thrun.
\newblock \emph{Explanation-Based Neural Network Learning: A Lifelong Learning
  Approach}.
\newblock Kluwer Academic Publishers, Boston, MA, 1996.

\bibitem[Thrun and Mitchell(1995)]{thrun_mitchell:1995}
Sebastian Thrun and Tom Mitchell.
\newblock Lifelong robot learning.
\newblock \emph{Robotics and Autonomous Systems}, 15:\penalty0 25--46, 1995.

\bibitem[Vovk(1990)]{vovk:1990}
Vladimir Vovk.
\newblock Aggregating strategies.
\newblock In \emph{Proceedings of the $3^{rd}$ Internation Conference on
  Computational Learning Theory}, 1990.

\bibitem[Wilson et~al.(2007)Wilson, Fern, Ray, and
  Tadepalli]{wilson_fern_ray_tadepalli:2007}
Aaron Wilson, Alan Fern, Soumya Ray, and Prasad Tadepalli.
\newblock Multi-task reinforcement learning: A hierarchical bayesian approach.
\newblock In \emph{Proceedings of the $24^{th}$ International Conference on
  Learning Theory}, 2007.

\bibitem[Yu and Mannor(2009)]{yu_mannor:2009}
Jia~Yuan Yu and Shie Mannor.
\newblock Arbitrarily modulated markov decision processes.
\newblock In \emph{Proceedings of the {IEEE} Conference on Decision and
  Control}, 2009.

\bibitem[Yu et~al.(2009)Yu, Mannor, and Shumkin]{yu_mannor_shimkin:2009}
Jia~Yuan Yu, Shie Mannor, and Nahum Shumkin.
\newblock Markov decision processes with arbitrary reward processes.
\newblock \emph{Mathematics of Operations Research}, 34(3):\penalty0 737--757,
  2009.

\end{thebibliography}

\end{document}